%% file: main.tex
\documentclass[11pt]{article}

\usepackage[a4paper,
            left=25mm,
            right=25mm,
            top=25mm,
            bottom=28mm]{geometry}

\usepackage{amsmath,amssymb,amsthm}
\usepackage{latexsym}
\usepackage{bm}
\usepackage{graphicx}
\usepackage{algorithm}
\usepackage{algorithmic}
\usepackage{float}  
\usepackage{placeins}

\usepackage{titletoc}

\usepackage{color}

\providecommand{\bestval}[1]{#1}
\renewcommand{\bestval}[1]{\ensuremath{\bm{#1}}}
\providecommand{\secondval}[1]{#1}
\renewcommand{\secondval}[1]{\ensuremath{\underline{#1}}}

\def\x{{\bm x}}
\def\rmd{\mathrm{d}}

\def\Rbb{\mathbb{R}}
\def\U{{\bm U}}
\def\e{{\bm e}}

\newtheorem{theorem}{Theorem}
\newtheorem{proposition}{Proposition}
\newtheorem{corollary}{Corollary}
\newtheorem{assumption}{Assumption}
\newtheorem{definition}{Definition}
\newtheorem{lemma}{Lemma}
\newtheorem{example}{Example}
\newtheorem{remark}{Remark}

\usepackage[colorlinks=true,
            linkcolor=blue,
            citecolor=blue,
            urlcolor=blue]{hyperref}

\title{Identifiability and Estimation for Unlabeled Finite Mixtures under Marginal Independence}
\author{
Takafumi Kanamori$^{1,2}$, Yushi Hirose$^1$, Shohei Yamamoto$^1$\\[0.5ex]
{\small $^1$Department of Mathematical and Computing Science, Institute of Science Tokyo}\\
{\small $^2$RIKEN Center for Advanced Intelligence Project}
}
\date{}  

\begin{document}

\maketitle

\begin{abstract}
We study component recovery and mixing-matrix estimation from unlabeled finite mixtures whose observable distributions share the same latent components but have unknown mixing weights. The main identifying signal is marginal independence: each component is assumed to be independent on at least one coordinate pair, but no labels, clean component samples, or mixing weights are observed. We first prove a structural result for product components: under a subset-rank condition on the spans of the univariate marginals, any independent affine combination of the components must coincide with a single component. We then extend this principle to observable mixtures and show that, under the corresponding subset-rank, full-rank, and no-cancellation conditions, marginally independent affine combinations recover the corresponding latent components. When every component is independent on some coordinate pair, all components are identifiable, and the mixing matrix is recoverable under the stated completion conditions. Finally, we propose a Product-Marginal Maximum Mean Discrepancy (PM-MMD) estimator over affine combinations of the observable mixtures and prove uniform convergence and stability under approximate marginal independence. This framework also separates the empirical roles of the assumptions: irreducibility is, in general, not directly testable from the unlabeled mixtures alone, whereas marginal independence yields a candidate-level diagnostic through held-out PM-MMD. Controlled and flow-cytometry experiments show when marginal independence provides a useful recovery signal. In the reported multi-component comparisons, condition-aware representative selection stabilizes PM-MMD and improves recovery relative to clustering, factorization, and pairwise mixture-proportion baselines using the same unlabeled mixtures.
\end{abstract}

\section{Introduction}

Many weakly supervised learning problems provide only aggregate or unlabeled distributions. In label-shift adaptation, learning with label proportions, positive-unlabeled learning, and multi-source domain adaptation, one often observes mixtures of latent class-conditional distributions without clean component samples or exact mixing weights
\cite{quadrianto2009estimating,elkan2008learning,roberts2022unsupervised,deng2023mixture}.
A central question is whether the latent component distributions can be recovered from such observable mixtures alone.

This paper studies this question through marginal independence. We consider several observable mixtures that share the same latent components but have unknown mixing weights. Instead of assuming labeled samples, anchor sets, parametric component families, or known proportions, we assume that each component is marginally independent for at least one pair of coordinates. The goal is to recover the latent components and, when possible, the mixing matrix from the observable mixtures.

Our first result clarifies the geometric mechanism behind the recovery. For product-form components, under a subset-rank condition on the spans of the univariate marginals, any fully independent distribution in the affine span of the components must coincide with one of the components. Thus, the intersection between the affine mixture space and the independence model is discrete and consists only of the component distributions. This gives an explicit identifiability condition for mixtures of independent components.

We then extend this principle to marginal independence. For a fixed coordinate pair, we show that any affine combination of the observable mixtures whose bivariate marginal is independent must recover a component that is independent on that pair. Under a no-cancellation condition and the corresponding subset-rank condition on the relevant marginal spans, the coordinate-pair independence constraint identifies the corresponding components. If every component is independent for at least one coordinate pair, all components are identifiable. Under additional irreducibility or square-completion conditions, the mixing matrix is also recoverable.

Our next contribution is an estimation procedure that instantiates this population characterization. 
We introduce a Product-Marginal Maximum Mean Discrepancy (PM-MMD)
criterion: for a candidate affine combination of the observable mixtures, we compare its bivariate marginal with the product of its one-dimensional marginals using a characteristic kernel. This follows the spirit of kernel-based independence tests such as HSIC
\cite{gretton2007kernel}
and MMD-based two-sample testing
\cite{JMLR:v13:gretton12a}.
We then search over affine combinations of the observable mixtures and choose combinations that approximately minimize the PM-MMD criterion. In this way, we recover components that satisfy independence or marginal independence. 
In the square case, we complete the mixing matrix directly when all weight
vectors are recovered, and provide partial-recovery completion rules for the cases where one or two components remain unresolved.

Marginal independence and irreducibility differ in their empirical status. Marginal independence concerns a candidate distribution constructed from the observed mixtures. Once such a candidate is formed, its bivariate marginal can be compared with the product of its one-dimensional marginals on held-out samples. PM-MMD is used as a held-out check: small values suggest compatibility with marginal independence, whereas large values suggest that the candidate lacks the required product-marginal structure. Irreducibility is different. It concerns a residual component in a mixture decomposition. The residual component and the true mixture proportion are not observed. Thus irreducibility is not, in general, directly verifiable from the unlabeled mixtures alone. We use it only as an identification assumption for completing the mixing matrix.

Beyond mixture models, our work connects to broader research on weak or partial supervision, where latent proportions or latent components must be inferred from incomplete or aggregated information. Examples include learning with label proportions, class-prior estimation under distribution shift, and multi-source domain adaptation
\cite{quadrianto2009estimating,roberts2022unsupervised,deng2023mixture}.
These settings share the high-level theme of learning from aggregate or unlabeled distributions. In contrast to methods that require clean component samples, known proportions, or domain-invariance assumptions, our approach uses marginal-independence structure as the identifying signal.

The estimator is based on kernel mean embeddings of probability measures~\cite{muandet2017kernel}. We use known deviation bounds to control the error of empirical embeddings and of empirical PM-MMD criteria. This yields uniform convergence of the empirical criterion over the affine search domain. We also derive stability guarantees under approximate marginal independence. Thus, the paper provides both identifiability results and a practical estimator for unlabeled finite mixtures based on marginal-independence signals.

We complement the theoretical results with numerical experiments in 
Section~\ref{sec:numerical-experiments}. 
The main text reports a compact set of experiments: controlled sanity checks,
a three-component DLBCL flow-cytometry task from raw gated pools, and a binary
DLBCL calibration against recent mixture-proportion-estimation (MPE) baselines.
Additional semi-synthetic, oracle clean-component, and sample-size studies are
given in the appendix.  These experiments show when marginal independence
provides useful recovery information and how condition-aware representative
selection improves stability in multi-component recovery.

The rest of the paper is organized as follows.
Section~\ref{sec:related-work} reviews related work.
Section~\ref{sec:product-form-identifiability} studies the product-space setting and
identifiability from full independence.
Section~\ref{sec:mixing-vector-recovery} develops the marginal-independence
recovery theory for mixing vectors and components.
Section~\ref{sec:mmd-estimation} introduces the PM-MMD estimator and
its statistical guarantees.
Section~\ref{sec:near-marginal-independence} discusses stability under
near marginal independence.
Section~\ref{sec:numerical-experiments} presents the numerical experiments, and
Section~\ref{sec:conclusion} concludes the paper.
Appendix~\ref{app:notation} 
provides a notation list for the symbols used throughout the paper. 
The remaining appendices collect proofs, technical details, additional recovery cases, implementation details, and supplementary experiments. 

\section{Related Work}
\label{sec:related-work}

We briefly review related work and highlight the differences in assumptions and goals.

\paragraph{Mixture proportion and class prior estimation.}
There is a large literature on estimating mixture proportions and class priors.
Typical settings assume partial knowledge of some component distributions or access to labeled data for some classes
\cite{blanchard2010semi,%
elkan2008learning,%
jain2016estimating,
scott2015rate}.
Another variant of weak supervision is learning with label proportions (LLP), where training examples are grouped into bags and only the class proportion of each bag is observed
\cite{quadrianto2009estimating}.
Busa-Fekete et al.~\cite{busa-fekete2025nearly} derived nearly optimal sample-complexity guarantees for this setting.
Kernel-based, divergence-based, and black-box methods use separability, linear independence of component measures, approximate anchor sets, or classifier-based reductions
\cite{christoffel2016class,garg2021mixture,ramaswamy2016mixture,yu2018efficient,zhu2023mixture}.
These methods provide estimators of mixture proportions under structural assumptions on supports, kernels, or classifier scores.

The decontamination literature studies related proportion and residue recovery problems under irreducibility or joint irreducibility assumptions
\cite{blanchard14:_decon_mutual_contam_model,katzsamuels2019decontamination,scott13:_class_asymm_label_noise}.
These results are closest to our completion step, where irreducibility is used after some components have already been recovered.
However, their identifying signal is irreducibility or support structure rather than marginal independence.
Our first-stage recovery instead identifies components as affine combinations of several unlabeled mixtures that satisfy a product-marginal constraint.

Risk-based approaches follow a different route.
They construct unbiased or nearly unbiased empirical risks for learning the desired classifier, often assuming or separately estimating the class prior.
The non-negative risk estimator for positive-unlabeled learning is a prominent example~\cite{kiryo2017positive}.
However, these approaches do not directly identify component distributions as measures and typically focus on a single observable mixture.
In contrast, we study intersections between a finite mixture model and an independence model and use several observable mixtures to identify both components and the mixing matrix.

\paragraph{Class distribution estimation under distribution shift.}
Class distribution estimation under distribution shift is closely related.
Tasche \cite{tasche2023invariance} considered a labeled training distribution and an unlabeled test distribution.
Under covariate shift and certain joint shift models, invariance of posterior probabilities implies identifiability of the test-time class proportions.
More recent work studies several unlabeled domains that share class-conditional densities but differ in class proportions.
Roberts et al.~\cite{roberts2022unsupervised} analyzed this latent label shift setting.
Mielniczuk et al.\ \cite{mielniczuk2025class} addressed class-prior estimation in a positive-unlabeled setting under label shift and proposed a kernel-based estimator with consistency guarantees.
Deng et al.\ \cite{deng2023mixture} studied a multi-source multi-target domain adaptation setting, assuming each target distribution is a mixture of multiple source domains and providing algorithms with theoretical guarantees for estimating the mixture weights and predicting target-specific models.
These works do not impose independence or marginal independence on the components.
Our setting is different.
We assume several unlabeled mixtures that share the same components, and identifiability follows from independence-type constraints and marginal-span rank conditions, rather than domain invariance.

\paragraph{Noisy-label transition-matrix estimation.}
Class-conditional noisy-label learning estimates a transition matrix from a clean
label to an observed noisy label, often using loss correction, importance
reweighting, anchor-point assumptions, or posterior-simplex
geometry~\cite{liu2016importance,li2021volminnet}.  This literature is related
to our setting through a Bayes-dual reinterpretation: if the observed mixture
index is denoted by $Y'$ and treated as a noisy label, then the standard
transition matrix is $P(Y'\mid Y)$, whereas our target mixing matrix is
$P(Y\mid Y')$.  The two are connected by Bayes' rule after a prior over
mixture indices is specified.  However, the formulations are not identical.
In our primary model, $Y'$ is a mixture or group index and the data are
conditional samples from $P(X\mid Y'=j)$; the prior $P(Y')$ is a
sampling-design choice.  Noisy-label transition methods estimate
$P(Y'\mid Y)$ through discriminative assumptions on $P(Y'\mid X)$, while
PM-MMD uses marginal independence of the component distributions.

\paragraph{Multi-view latent variable models and tensor methods.}
Independence across several observed variables is a classical tool in latent variable models.
Hall and Zhou \cite{hall2003nonparametric} used coordinate-wise independence and linear independence of marginals to obtain identifiability for two-component mixtures on product spaces.
Allman et al.\ \cite{allman09:_ident} studied latent structure models with at least three observed variables and established identifiability results under rank conditions using tensor representations.

These ideas led to tensor-based algorithms for learning latent variable models.
Anandkumar et al.\ \cite{anandkumar2014tensor} developed tensor decompositions for a wide class of models, including finite mixtures.
Song et al.\ \cite{song2014nonparametric} and Sgouritsa et al.\ \cite{sgouritsa2013identifying} studied nonparametric multi-view and product mixture models using moments, kernel embeddings, and independence tests.
Kernel-based independence statistics such as HSIC play a central role in these works
\cite{gretton2007kernel}.

Our approach is also based on independence.
However, it focuses on the geometry of the intersection between a finite mixture model and an independence-type model.
Components are characterized as isolated points in this intersection.
We recover them as affine combinations of several observable mixtures by minimizing a kernel-based independence criterion, rather than by tensor or spectral decompositions.
Moreover, we use marginal independence on selected coordinate pairs, rather than requiring a full multi-view product representation for every component.

\paragraph{Kernel methods for independence and mixture models.}
Kernel mean embeddings provide powerful tools for comparing probability distributions~\cite{muandet2017kernel}.
They underlie maximum mean discrepancy statistics and kernel tests of homogeneity and independence~\cite{JMLR:v13:gretton12a,gretton2007kernel,tolstikhin2016minimax}.
Kernel, embedding-based, and related black-box approaches have also been used for mixture proportion estimation and weak supervision~\cite{ramaswamy2016mixture,%
yu2018efficient,%
garg2021mixture,%
zhu2023mixture}.
In our work we use kernel embeddings to define an independence loss over affine combinations of observable mixtures.
We then establish uniform convergence of the empirical loss and derive rates for the resulting estimators
under our identifiability conditions.

\section{Affine Identifiability from Product-Form Independence}
\label{sec:product-form-identifiability} 

We first define the signed affine density class and the probability density class. 

Let $(\mathcal{X}_k,\mathcal{B}_k,\mu_k)$, $k=1,\ldots,d$, be measurable spaces,
and let
$\mathcal{X} := \mathcal{X}_1 \times \cdots \times \mathcal{X}_d,\,
 \mathcal{B} := \mathcal{B}_1 \otimes \cdots \otimes \mathcal{B}_d$, and 
  $\mu := \mu_1 \otimes \cdots \otimes \mu_d$. 
Throughout, we identify probability distributions on $(\mathcal{X},\mathcal{B})$
with their densities with respect to the product measure $\mu$.
The sets of all functions with total mass one and 
all probability densities with the variable
$\x = (x_1,\ldots,x_d) \in \mathcal{X}_1 \times \cdots \times \mathcal{X}_d$ are denoted by
\begin{align*}
  \mathcal{S}
 &= \Bigl\{ 
 q \in L_1(\mathcal{X},\mu) \,\Bigm|\, \int_{\mathcal{X}} q(\x)\,\rmd\mu(\x) = 1
    \Bigl\}, \quad \text{and}\\
  \mathcal{P}
& = \Bigl\{ 
 q \in L_1(\mathcal{X},\mu) \,\Bigm|\, q(\x) \ge 0,\; \int_{\mathcal{X}} q(\x)\,\rmd\mu(\x) = 1
    \Bigr\},
\end{align*}
and the set of product-form functions is
\begin{align*}
  \mathcal{E}
  = \Bigl\{ q \in \mathcal{S} \,\Bigm|\,q(\x) = q_1(x_1)\cdots q_d(x_d),\ 
    \text{where}\ 
    q_j(x_j)=\int_{\mathcal{X}_{(-j)}}\!\!\!\!\! q(\x)\prod_{i\neq {j}}\rmd\mu_i(x_i),\, j\in[d]
    \Bigr\}
 \subset \mathcal{S}.
\end{align*}
We write $\mathrm{d}\mu(\x)$ simply as $\mathrm{d}\x$.
For $\sigma \subset [d]$ and $j \in [d]$, set
$\x_\sigma := (x_i)_{i\in\sigma}\in \mathcal{X}_\sigma := \prod_{i\in\sigma}\mathcal{X}_i$
and 
$\x_{(-j)} := (x_i)_{i\neq j}\in \mathcal{X}_{(-j)} := \prod_{i\neq j}\mathcal{X}_i$. 
For any integrable function $f$, we abbreviate
$\int_{\mathcal{X}_\sigma} f(\x)\,\prod_{i\in\sigma} \mathrm{d}\mu_i(x_i)$
as $\int_{\mathcal{X}_\sigma} f(\x)\,\mathrm{d}\x_\sigma$, and 
$\int_{\mathcal{X}_{(-j)}} f(\x)\,\prod_{i\neq j} \mathrm{d}\mu_i(x_i)$ 
as $\int_{\mathcal{X}_{(-j)}} f(\x)\,\mathrm{d}\x_{(-j)}$, respectively. 
Probability densities in $\mathcal{P}\cap\mathcal{E}$ correspond to mutually independent coordinates, 
whereas elements of $\mathcal{P}\setminus\mathcal{E}$ allow arbitrary
dependence among $(x_1,\ldots,x_d)$.

The affine hull generated by $\{p_i\}_{i=1}^m\subset\mathcal{P}$ with the weight
vector ${\bm\theta} = (\theta_1,\ldots,\theta_m)^T$ is denoted by 
\begin{align*}
 \mathcal{M}=\bigg\{ p_{\bm\theta}(\x):=\sum_{i=1}^{m}\theta_i p_i(\x)\in\mathcal{S} 
 \,\Bigm|\,
 \theta_i\in\Rbb,\ 
 \sum_{i=1}^{m}\theta_i=1 \bigg\}. 
\end{align*}
Note that some elements of ${\bm\theta}$ can be negative. 
Our goal is to find a sufficient condition under which 
\begin{align}
 \label{eq:affine-product-identifiability}
 \mathcal{M}\cap\mathcal{E} = \{p_i\,|\,i\in[m]\}
\end{align}
exactly holds. 

To formulate such a condition, let us consider a projection onto
$\mathcal{E}$. For $q \in \mathcal{S}$, let $q_k$ be the $k$-th marginal
function of $q$, i.e., 
\begin{align*}
q_k(x_k)=\int_{\mathcal{X}_{(-k)}} q(\x)\mathrm{d}\x_{(-k)}.
\end{align*}
For each component $p_i$, write $p_{ik}(x_k)$ for its $k$th marginal.  Define
\begin{align*}
  \Pi_{\mathcal{E}}[q]  := \prod_{k=1}^d q_k \in \mathcal{E}. 
\end{align*}
For ${\bm\theta} = (\theta_1,\ldots,\theta_m)^T$ with $\sum_{i=1}^m \theta_i = 1$,
the $m$-projection of $p_{\bm\theta}(x) := \sum_{i=1}^m \theta_i p_i(x) \in \mathcal{M}$ 
onto $\mathcal{E}$ is thus given by
\begin{align}
\label{eq:m-projection-mixture}
  \Pi_{\mathcal{E}}[p_{\bm\theta}]
  = \prod_{k=1}^d \sum_{i=1}^m \theta_i p_{ik}
  = \sum_{\sigma:[d]\to[m]}\Bigl( \prod_{k=1}^d \theta_{\sigma(k)} \Bigr) \prod_{k=1}^d p_{\sigma(k)k}, 
\end{align}
where the last equality follows from expanding the product of sums.
Clearly, $\Pi_{\mathcal{E}}[q] = q$ holds if and only if
$q \in \mathcal{E}$. Under the product-form assumption
$\{p_i\}_{i\in[m]}\subseteq\mathcal{E}$ imposed below,
\eqref{eq:affine-product-identifiability} is therefore equivalent to the statement
that every $p_{\bm\theta}\in\mathcal{M}\cap\mathcal{E}$ coincides with one of
the component densities.  The sufficient condition below yields the stronger
coefficient-level conclusion ${\bm\theta}={\bm e}_i$ for some $i\in[m]$,
where $\bm{e}_i$ denotes the $i$-th standard basis vector in $\mathbb{R}^m$.

For $p_i\in\mathcal{P}$ and distinct $s,t\in[d]$, write
$p_i(x_s,x_t)$ for its $(s,t)$-marginal.
For a nonempty subset
$I\subseteq[m]$ and $k\in[d]$, define
\begin{align*}
 \delta_k(I)
 :=\dim\operatorname{span}\{\,p_{ik}(x_k)\mid i\in I\,\}.
\end{align*}
The next theorem gives a sufficient condition for
\eqref{eq:affine-product-identifiability}.
\begin{theorem}
\label{thm:product-form-affine-identifiability}
Suppose that $d\ge2$ and
\begin{align*}
 p_i(\x)=\prod_{k=1}^{d}p_{ik}(x_k)\in\mathcal{P}\cap\mathcal{E},
 \qquad i\in[m].
\end{align*}
Assume that, for every $I\subseteq[m]$ with $|I|\ge2$, there exist distinct
coordinates $s,t\in[d]$, possibly depending on $I$, such that
\begin{align*}
 \delta_s(I)+\delta_t(I)>|I|+1.
\end{align*}
Then, for every ${\bm\theta}\in\mathbb{R}^m$ satisfying
$\sum_{i\in[m]}\theta_i=1$, the condition
$p_{\bm\theta}\in\mathcal{E}$ implies
${\bm\theta}={\bm e}_i$ for some $i\in[m]$; in particular,
\eqref{eq:affine-product-identifiability} holds.
\end{theorem}
The proof is given in Appendix~\ref{app:proof-product-form-identifiability}. 
\begin{remark}
For a nontrivial family with $m\ge2$, the subset-rank assumption itself
requires at least two coordinates.  If $d=1$, then the independence model
$\mathcal{E}$ coincides with $\mathcal{S}$, and
\eqref{eq:affine-product-identifiability} need not hold.
The subset-rank assumption in
Theorem~\ref{thm:product-form-affine-identifiability}
is expressed only through the dimensions of spans of univariate marginals,
and the coordinate pair may depend on the support set $I$.
As a sufficient special case, if there exist two fixed coordinates $s,t$ for
which both marginal families $\{p_{is}\}_{i\in[m]}$ and
$\{p_{it}\}_{i\in[m]}$ are linearly independent, then
$\delta_s(I)=\delta_t(I)=|I|$ for every nonempty $I\subseteq[m]$.
The bound $\delta_s(I)+\delta_t(I)\le |I|+1$ obtained in the proof for every
distinct $s,t$ is necessary for a nontrivial product-form affine combination,
but its converse does not hold in general.
Thus the subset-rank condition is sufficient, but not necessary, for uniqueness.
Under the stated subset-rank condition, any affine combination
$\sum_{i=1}^m \theta_i p_i$ that belongs to $\mathcal{E}$ must coincide with
one of the components $p_i$, even though the coefficients $\theta_i$ are
allowed to be negative.
Hence, Theorem~\ref{thm:product-form-affine-identifiability} gives a simple
sufficient condition for the identifiability in
\eqref{eq:affine-product-identifiability}.
\end{remark}

\begin{remark}
The scope of Theorem~\ref{thm:product-form-affine-identifiability} differs from that
of existing identifiability results for latent structure models.
Allman et al.~\cite{allman09:_ident} study global identifiability of the full
mixture model, including both the component distributions and the mixing
weights.
Their assumptions require linear independence of each coordinate block and at
least three conditionally independent views.
Sgouritsa et al.~\cite{sgouritsa2013identifying} use kernel embeddings and
tensor decompositions to show that the number of mixture components equals the
tensor rank, and use this fact to learn nonparametric product mixtures with a
finite set of confounders.
In contrast, Theorem~\ref{thm:product-form-affine-identifiability} gives a structural
criterion ensuring that any coordinate-wise independent distribution in the
affine span of $\mathcal{M}$ coincides with one of the original product
components.
\end{remark}

\section{Recovery of Mixing Weights}
\label{sec:mixing-vector-recovery}

We now pass from component-level identifiability to the observable multi-mixture setting.
The observed distributions are unlabeled mixtures of the same latent components.
Their mixing weights are unknown.
Our goal is to find affine combinations of the observed mixtures that recover the latent components.
This section first sets up the observed mixture model.
It then shows how marginal independence identifies such affine combinations.
Finally, it discusses the testability of the assumptions used in the recovery arguments.

\subsection{Observed Mixture Models}
For $\{p_i\}_{i=1}^{m}\subset\mathcal{P}$ and $L \ge m$, we define
\begin{align*}
  U_\ell(x) = \sum_{i\in[m]} \theta_{\ell i} p_i(x)\in\mathcal{S},   \qquad \ell=1,\ldots,L, 
\end{align*}
where $\Theta = (\theta_{\ell i}) \in [0,1]^{L\times m}$ satisfies
$\Theta \bm{1}_m = \bm{1}_L$.
Let $\bm{p} = (p_1,\ldots,p_m)^T$ and $\bm{U} = (U_1,\ldots,U_L)^T$.
Then
\begin{align*}
  \bm{U} = \Theta \bm{p}. 
\end{align*}
Suppose that the mixing matrix $\Theta$ has full column rank. 
If \(\Theta\) were known, one could recover $\bm{p}$ from $\bm{U}$ by
\begin{align*}
\bm{p}=(\Theta^T \Theta)^{-1}\Theta^T \bm{U}. 
\end{align*}
This oracle identity motivates the search for affine coefficients.
Note that
\begin{align*}
  (\Theta^T \Theta)^{-1}\Theta^T  \bm{1}_L
  = (\Theta^T \Theta)^{-1}\Theta^T \Theta \bm{1}_m
  = \bm{1}_m,
\end{align*}
so each component $p_i$ can be written as an affine combination of
$U_1,\ldots,U_L$.

Next we show that the independence property in
Theorem~\ref{thm:product-form-affine-identifiability} is crucial for recovery. 
Let $r_\ell \in \mathbb{R}$, $\ell\in[L]$, be coefficients such that
$\sum_{\ell\in[L]} r_\ell = 1$, and suppose that
\begin{align*}
  \sum_{\ell\in[L]} r_\ell U_\ell(x) \in \mathcal{E}. 
\end{align*}
Using the representation of $U_\ell$, we obtain
\begin{align*}
  \sum_{\ell\in[L]} r_\ell U_\ell(x)
    = \sum_{\ell\in[L]} r_\ell \sum_{i\in[m]} \theta_{\ell i} p_i(x)
    = \sum_{i\in[m]} \beta_i p_i(x),
\end{align*}
where $\beta_i:=\sum_{\ell\in[L]} r_\ell \theta_{\ell i},\,i\in[m]$. 
Since $\Theta \bm{1}_m = \bm{1}_L$ and $\sum_{\ell} r_\ell = 1$, we have
$\sum_{i\in[m]} \beta_i = 1$. 
If the assumptions of
Theorem~\ref{thm:product-form-affine-identifiability} hold for
$p_i(x) = \prod_{k\in[d]} p_{ik}(x_k)$, then
\begin{align*}
\sum_{i\in[m]} \beta_i p_i \in \mathcal{E}
\end{align*}
implies that
\begin{align*}
\bm{\beta} = (\beta_1,\ldots,\beta_m)^T = \bm{e}_i
\end{align*}
for some $i\in[m]$. 
In other words, any affine combination $\sum_{\ell\in[L]} r_\ell U_\ell$ 
that lies in $\mathcal{E}$ must coincide with one of the component
densities $p_i$.
This observation will be used in later sections, where we estimate such
weight vectors $\bm{r} = (r_1,\ldots,r_L)^T$ from data and recover the
independent components via $\sum_{\ell\in[L]} r_\ell U_\ell$.

\begin{remark}
Consider the correspondence to a classification problem. Let $p_i$ denote the class-conditional
density for label $i$, so that a mixing vector plays the role of the class-probability vector. If only
unlabeled mixtures are observed, the components $\{p_i\}$ can be recovered only up to a permutation of
labels. When some mixtures are known to involve specific labels, this permutation ambiguity is removed.
For example, in PU-learning \cite{elkan2008learning,jain2016estimating}, 
if $U_{1}$ equals $p_1$ and $U_{2}$ is a mixture of $p_1$ and
$p_2$, then the identity of the positive component $p_1$ fixes its label, and the
remaining recovered component must correspond to the other label. 
A similar argument applies when the problem has $K$ classes. 
If $U_{\ell}$ (for $\ell=1,\dots,L$) is known to be a mixture of 
the class-conditionals $p_i$ for $i=1,\dots,\min(\ell,K)$, 
then once the components are recovered, 
their correspondence to class labels is uniquely determined within this ordering.
\end{remark}

\begin{remark}
We briefly look at a simple setting with two components. 
Let $p \in \mathcal{P}\cap\mathcal{E}$ be an independent distribution and
$q \in \mathcal{P}\setminus\mathcal{E}$ be a dependent distribution, and form the affine combination
$\alpha p + (1-\alpha) q,\,\alpha \in \mathbb{R}$. 
for values of $\alpha$ such that this mixture belongs to $\mathcal{P}$.
 Fix the marginals $q_k(x_k), k\in[d]$ of $q$ and consider all joint distributions with these
marginals.
Within this class, one can show that the set of $q$ for which there exists
$\alpha \ne 1$ satisfying
$\alpha p + (1-\alpha) q \in \mathcal{P}\cap\mathcal{E}$
forms at most a one-parameter family, 
\begin{align*}
\frac{\prod_{k\in[d]}(\alpha p_k(x_k)+(1-\alpha)q_k(x_k)) - \alpha p(\x)}{1-\alpha}\in\mathcal{P}\setminus\mathcal{E},\quad \alpha\neq 1. 
\end{align*}
Hence such $q$ constitute only a one-dimensional exceptional subset.
For a typical choice of $q$ outside this one-dimensional subset, the above
condition forces $\alpha = 1$, and the mixture then coincides 
with the independent component $p$ itself. 
Thus, even in this simple two-component setting, independence of the mixture
already yields an identifiability property.
This point will be further examined in Theorem~\ref{thm:marginal-independence-identifiability}. 
\end{remark}

\subsection{Recovery from Marginal Independence Assumption}
\label{subsec:marginal-independence-recovery}

We now consider the recovery of probability densities under a more general
independence assumption, referred to as the \emph{marginal independence
assumption}. In particular, each component $p_i \in \mathcal{P}$,
$i \in [m]$, is no longer required to factorize into the product of its
univariate marginals.

For a subset $\sigma \subset [d]$, define 
\begin{align*}
  \mathcal{S}_\sigma
 &:= \Bigl\{
 p : \mathcal{X}_\sigma \to \mathbb{R}
 \,\Bigm|\,
 \int_{\mathcal{X}_\sigma} p(\x_\sigma)\,\rmd{\x}_\sigma = 1
 \Bigr\}, \\
  \mathcal{P}_\sigma
 &:= \Bigl\{
 p : \mathcal{X}_\sigma \to \mathbb{R}
 \,\Bigm|\,
 p \ge 0,\;
 \int_{\mathcal{X}_\sigma} p(\x_\sigma)\,\rmd{\x}_\sigma = 1
 \Bigr\}, \\
  \mathcal{E}_\sigma
  &:= \Bigl\{
        \prod_{s\in\sigma} p_s(x_s)
        \,\Bigm|\,
        p_s \in \mathcal{S}_s,\ s\in\sigma
      \Bigr\}
  \subset \mathcal{S}_\sigma .
\end{align*}
For a singleton $\sigma = \{s\} \subset [d]$ we simply write $\mathcal{S}_s$ or $\mathcal{P}_s$.
For $\sigma = \{s,t\} \subset [d]$ we write 
$\mathcal{S}_{st}, \mathcal{P}_{st}$ and
$\mathcal{E}_{st}$ instead of 
$\mathcal{S}_{\sigma}, \mathcal{P}_\sigma$ and $\mathcal{E}_\sigma$.
Note that $\mathcal{S}_{[d]}=\mathcal{S}, \mathcal{P}_{[d]} = \mathcal{P}$ and $\mathcal{E}_{[d]} = \mathcal{E}$.
For $p_i \in \mathcal{P}$, $i \in [m]$, and indices $s,t \in [d]$ with $s<t$,
we define
\begin{align}
 \label{eq:def-rst}
  R_{st}
  := \bigl\{
        i \in [m]
        \bigm|\,
        p_i(x_s,x_t) = p_i(x_s)\,p_i(x_t) \in \mathcal{E}_{st}
      \bigr\}.
\end{align}
Here the univariate marginals $p_i(x_s)$ and $p_i(x_t)$ depend on $i$ and the
coordinate, but we write them in this form whenever it causes no confusion.
Thus $R_{st}$ collects those components that are independent in the pair of
coordinates $(x_s,x_t)$.

Let $U_\ell(\x)$, $\ell \in [L]$, be the mixtures of $p_i$, $i \in [m]$,
defined by
\begin{align}
  U_\ell(\x)
  = \sum_{i\in[m]} \theta_{\ell i} p_i(\x)\in\mathcal{S},
  \qquad \ell \in [L],                            \label{eq:def-observed-mixtures}
\end{align}
for coefficients $\theta_{\ell i} \in [0,1]$ satisfying
$\sum_{i\in[m]} \theta_{\ell i} = 1$.
Equivalently, the mixing matrix $\Theta = (\theta_{\ell i}) \in [0,1]^{L\times m}$
satisfies $\Theta {\bm 1}_m = {\bm 1}_L$.
For coefficients $r_\ell \in \mathbb{R}, \ell\in[L]$ with $\sum_{\ell\in[L]} r_\ell = 1$,
a linear combination of $U_\ell$ can be written as
\begin{align*}
  \sum_{\ell\in[L]} r_\ell U_\ell(\x)
  &= \sum_{\ell\in[L]} r_\ell \sum_{i\in[m]} \theta_{\ell i} p_i(\x)
   = \sum_{i\in[m]} \beta_i p_i(\x),
\end{align*}
where
\begin{align*}
  \beta_i := \sum_{\ell\in[L]} r_\ell \theta_{\ell i},
  \qquad i \in [m].
\end{align*}
Since $\Theta \bm{1}_m = \bm{1}_L$ and $\sum_{\ell} r_\ell = 1$, we have
$\sum_{i\in[m]} \beta_i = 1$. 
Taking the $(s,t)$-marginal of the above expression yields
\begin{align}
  \sum_{\ell\in[L]} r_\ell U_\ell(x_s,x_t)
  &= \sum_{i\in R_{st}} \beta_i p_i(x_s)\,p_i(x_t)
   + \sum_{i\notin R_{st}} \beta_i p_i(x_s,x_t).
\end{align}
We are interested in affine combinations whose $(s,t)$-marginal is
independent; that is,
\begin{align*}
  \sum_{\ell\in[L]} r_\ell U_\ell(x_s,x_t) \in \mathcal{E}_{st}.
\end{align*}
For later use, we introduce the set
\begin{align*}
  \mathcal{U}_{st}
  := \Bigl\{
        \sum_{\ell\in[L]} r_\ell U_\ell(x_s,x_t) \in \mathcal{S}_{st}
        \,\Bigm|\,
        r_\ell \in \mathbb{R},\
        \sum_{\ell\in[L]} r_\ell = 1
      \Bigr\}.
\end{align*}

\begin{theorem}
\label{thm:marginal-independence-identifiability}
For $\{p_i\}_{i=1}^{m} \subset \mathcal{P}$, let $p_{ik}(x_k)$, $k\in[d]$,
denote the marginals of $p_i(\x)$.
Fix $s,t \in [d]$ with $s<t$ and let $R_{st}$ be the set defined in
\eqref{eq:def-rst}. 
Suppose that, for every $I\subseteq R_{st}$ with $|I|\ge2$,
\begin{align*}
 \delta_s(I)+\delta_t(I)>|I|+1.
\end{align*}
Suppose that the mixing matrix
$\Theta = (\theta_{\ell i}) \in [0,1]^{L\times m}$, satisfying
$\Theta \bm{1}_m = \bm{1}_L$, has full column rank. 
Assume that
\begin{align}
  \sum_{i\in[m]} \beta_i\, p_i(x_s,x_t) \notin \mathcal{E}_{st}
  \label{eq:no-cancellation}
\end{align}
for every $\bm{\beta} = (\beta_1,\ldots,\beta_m)^T \in \mathbb{R}^m$ such that
\begin{align*}
  \sum_{i\in[m]} \beta_i = 1
  \quad\text{and}\quad
  (\beta_i)_{i\in[m]\setminus R_{st}} \ne \bm{0}
  \in \mathbb{R}^{m-|R_{st}|}.
\end{align*}
Then
\begin{align*}
  \mathcal{U}_{st} \cap \mathcal{E}_{st}
  = \{\, p_i(x_s)\,p_i(x_t) \mid i \in R_{st} \,\}.
\end{align*}
For each $i \in R_{st}$, 
the coefficient vector $\bm{r} = \Theta(\Theta^T \Theta)^{-1} \bm{e}_i$ satisfies 
\begin{align*}
  \sum_{\ell\in[L]} r_\ell U_\ell(x_s,x_t) = p_i(x_s)\,p_i(x_t), 
\end{align*}
where $\bm{e}_i$ denotes the $i$th standard basis vector in $\mathbb{R}^m$. 
\end{theorem}
The proof is given in Appendix~\ref{app:proofs-mixing-vector-recovery}. 
When $\bm{r}^T\bm{U}(x_s, x_t)=p_i(x_s)p_i(x_t)$,
the coefficient-level conclusion of
Theorem~\ref{thm:product-form-affine-identifiability}, as applied in the proof of
Theorem~\ref{thm:marginal-independence-identifiability}, gives
$\bm{r}^T\Theta = \e_i^T$ under the subset-rank and no-cancellation conditions.
Hence, also $\bm{r}^T\bm{U}(\x)=p_i(\x)$ holds.

\begin{remark}
The assumption \eqref{eq:no-cancellation} in
Theorem~\ref{thm:marginal-independence-identifiability} is a generic no-cancellation condition. 
It rules out signed affine cancellations in which dependent non-target components combine with target components to produce a signed product-form marginal.
A sufficient condition for
$\sum_{i\in[m]} \beta_i p_i(x_s,x_t) \notin \mathcal{E}_{st}$ to hold for
every $(\beta_i)_{i\in[m]}$ satisfying the constraints in
Theorem~\ref{thm:marginal-independence-identifiability} is
\begin{align*}
  &\Bigl\{
      \sum_{i\in[m]\setminus R_{st}} \beta_i p_i(x_s,x_t)
      \,\Bigm|\,
      (\beta_i)_{i\in[m]\setminus R_{st}} \neq \bm{0}
    \Bigr\} 
 \cap
    \operatorname{span}\{\, p_i(x_s)p_j(x_t) \mid i,j \in [m] \,\}
    = \emptyset,
\end{align*}
since the condition
$\sum_{i\in[m]} \beta_i p_i(x_s,x_t) \notin \mathcal{E}_{st}$ means that
\begin{align*}
  \sum_{i\in[m]\setminus R_{st}} \beta_i p_i(x_s,x_t)
  \neq
  \sum_{i,j\in[m]} \beta_i \beta_j p_i(x_s)p_j(x_t)
  - \sum_{i\in R_{st}} \beta_i p_i(x_s)p_i(x_t)
\end{align*}
for all such $\bm{\beta}$, and the right-hand side always belongs to
$\operatorname{span}\{\, p_i(x_s)p_j(x_t) \mid i,j \in [m] \,\}$.
For continuous variables $x_s$ and $x_t$, the linear space, 
$\operatorname{span}\{\, p_i(x_s)p_j(x_t) \mid i,j \in [m] \,\}$, 
is finite dimensional, 
whereas the joint densities $p_i(x_s,x_t)$ with $i\notin R_{st}$ form an
infinite-dimensional family that can represent arbitrary dependence between
$x_s$ and $x_t$.
\end{remark}

\paragraph{Finite-table intuition for generic no-cancellation.}
The no-cancellation condition in
Theorem~\ref{thm:marginal-independence-identifiability} should be read as a generic
non-degeneracy condition rather than as a structural equality.  A simple
finite-table dimension count makes this intuition precise.  Suppose that the
$(s,t)$-marginal is represented on an $a\times b$ grid.  The affine space of
signed total-mass-one joint tables has dimension $ab-1$, whereas the
product-form model has dimension
\[
  (a-1)+(b-1)=a+b-2.
\]
Thus the product-form model has codimension
\[
  (ab-1)-(a+b-2)=(a-1)(b-1).
\]
If the coefficient family used for signed cancellation has dimension $d_B$,
then, under a standard transversality, or non-tangency, condition, the set of
component parameters for which some nontrivial signed affine cancellation
becomes product-form has dimension at most
\[
  \dim\mathcal H+d_B-(a-1)(b-1),
\]
where $\mathcal H$ denotes the parameter space of the non-target bivariate
marginals.  Hence, whenever $(a-1)(b-1)>d_B$, this exceptional parameter set is
lower dimensional in $\mathcal H$ and has Lebesgue measure zero.  A precise
finite-table statement and proof are given in
Appendix~\ref{app:generic-no-cancellation}.

\begin{remark}
If $R_{st} = [m]$, then $(\beta_i)_{i\in[m]\setminus R_{st}} \neq \bm{0}$
can never occur.
Hence the additional assumption
$\sum_{i\in[m]} \beta_i p_i(x_s,x_t) \notin \mathcal{E}_{st}$ in
Theorem~\ref{thm:marginal-independence-identifiability} is automatically
satisfied.
In this case,
Theorem~\ref{thm:marginal-independence-identifiability} reduces to
Theorem~\ref{thm:product-form-affine-identifiability} with $d=2$.
\end{remark}

Theorem~\ref{thm:marginal-independence-identifiability} implies that for each
$i\in R_{st}$ there exists $\bm{r}_i \in \mathbb{R}^L$ with
$\bm{r}_i^T \bm{1}_L = 1$ such that $\bm{r}_i^{T}\U(\x) = p_i(\x)$ under
its assumptions.
If
\begin{align*}
  \bigcup_{\substack{s,t\in[d]\\ s<t}} R_{st} = [m],
\end{align*}
then all components $p_i$, $i\in[m]$, can be recovered from affine
combinations of the mixtures $U_\ell$, $\ell\in[L]$, and the mixture model
is identifiable under the marginal independence assumption.
Furthermore, we do not need to specify $R_{st}$ explicitly when recovering
$p_i(\x)$ for $i\in R_{st}$: the set $R_{st}$ is \emph{automatically
detected} by solving the problem
$\bm{r}^T \U(x_s,x_t) \in \mathcal{E}_{st}$.

Summarizing the above discussion, we obtain the following corollary.
\begin{corollary}
\label{cor:full-recovery-marginal-independence}
Let $p_{ik}(x_k)$, $k\in [d]$, be the marginals of $p_i\in\mathcal{P}$,
$i\in[m]$.
For each $s,t\in[d]$ with $s<t$, suppose that, for every
$I\subseteq R_{st}$ with $|I|\ge2$,
\begin{align*}
 \delta_s(I)+\delta_t(I)>|I|+1.
\end{align*}
Suppose that the matrix $\Theta=(\theta_{\ell i})\in[0,1]^{L\times m}$ in
\eqref{eq:def-observed-mixtures}, satisfying $\Theta \bm{1}_m = \bm{1}_L$, has full
column rank.
For each $s,t\in[d]$ with $s<t$, suppose that
\begin{align*}
  \sum_{i\in[m]} \beta_i p_i(x_s,x_t) \notin \mathcal{E}_{st}
\end{align*}
holds for every $(\beta_1,\ldots,\beta_m)^T \in \mathbb{R}^m$ such that 
\begin{align*}
  \sum_{i\in[m]} \beta_i = 1
  \quad\text{and}\quad 
  (\beta_i)_{i\in[m]\setminus R_{st}} \neq \bm{0}
  \in \mathbb{R}^{m-|R_{st}|}. 
\end{align*} 
Define
\begin{align*}
  \mathcal{U}
  := \Bigl\{
        \sum_{\ell\in[L]} r_\ell U_\ell(\x) \in \mathcal{S}
        \,\Bigm|\,
        r_\ell\in\mathbb{R},\
        \sum_{\ell\in[L]} r_\ell = 1
      \Bigr\}.
\end{align*}
Then
\begin{align*}
  \bigcup_{\substack{s,t\in[d]\\ s<t}}
  \bigl\{
    p\in\mathcal{U}
    \,\bigm|\,
    p(x_s,x_t)\in \mathcal{E}_{st}
  \bigr\}
  =
 \Bigl\{\, p_i \Bigm| i\in \bigcup_{\substack{s,t\in[d]\\ s<t}} R_{st}  \,\Bigr\}. 
\end{align*}
 In particular, if $\bigcup_{\substack{s,t\in[d]\\ s<t}} R_{st} = [m]$, 
 then all component densities $p_i$, $i\in[m]$, are identifiable. 
\end{corollary}

\begin{remark}[Partial recovery and completion beyond marginal independence]
\label{rem:partial-recovery-completion}
Corollary~\ref{cor:full-recovery-marginal-independence} identifies precisely the
components whose indices belong to
$\mathcal I:=\cup_{s,t}R_{st}$. 
If
\(\mathcal J:=[m]\setminus\mathcal I\) is nonempty, 
the marginal-independence argument by itself does not recover the components in
\(\mathcal J\).  These unresolved components can be completed only under
additional assumptions such as the joint irreducibility on $\mathcal{J}$. 
The details are presented in Appendix~\ref{app:partial-completion_theorems}.
\end{remark}

\subsection{Testability of the Assumptions}
\label{subsec:testability-assumptions}

This subsection discusses which assumptions in the recovery problem have direct empirical implications.  Marginal independence gives an observable restriction on an affine candidate formed from the observed mixtures.  
Irreducibility and anchor-type conditions, on the other hand, are assumptions on latent components or residual distributions.  They are used as identification assumptions, 
rather than as directly testable conditions.

We first recall the standard irreducibility notion used in mixture proportion estimation~\cite{blanchard2010semi,scott2015rate}.  For probability distributions $F$ and $H$, define
\[
  \kappa(F\mid H)
  :=
  \sup\Bigl\{
    \alpha\in[0,1]
    \,\Bigm|\,
    F=\alpha H+(1-\alpha)G
    \text{ for some probability distribution }G
  \Bigr\}.
\]
We say that $F$ is irreducible with respect to $H$ if
\[
  \kappa(F\mid H)=0.
\]
Equivalently, no positive amount of $H$ can be removed from $F$ while leaving a probability distribution.  
Anchor conditions are support-level conditions of a similar kind.  For example, they may require a measurable set $A_i$ such that
\[
  p_i(A_i)>0,
  \quad
  p_j(A_i)=0,\ \ j\in [m]\setminus\{i\}. 
\]
Thus an anchor condition asserts that one component has positive mass on a region where the other components have zero mass.

These conditions are not directly testable from unlabeled mixtures in general.
They refer to latent component distributions, residual distributions, or exact support relations.  For instance, consider a binary mixture
\[
  U=\pi P+(1-\pi)N.
\]
At the population level, the observable pair $(P,U)$ determines at most the maximal removable proportion $\kappa(U\mid P)$.  It does not by itself certify the residual condition
\[
  \kappa(N\mid P)=0
\]
unless the true proportion $\pi$ or the residual component $N$ is known.  Exact zero-mass statements in anchor conditions are even harder to certify from finite samples, especially for continuous distributions.

Marginal independence is different.  Fix an affine candidate
\[
  q_{\bm r}:=\bm r^T\bm U,
  \qquad
  \bm r^T\bm 1_L=1.
\]
For a coordinate pair $(s,t)$, the condition
\[
  q_{\bm r}(x_s,x_t)=q_{\bm r}(x_s)q_{\bm r}(x_t)
\]
is a product-marginal relation for the candidate $q_{\bm r}$.  
Since $q_{\bm r}$ is constructed from the observed mixtures, this relation has a direct empirical implication.  One can check whether the candidate is compatible with marginal independence on the selected coordinate pair.  
This does not prove in advance that an unobserved true component is marginally independent, but it makes the identifying signal falsifiable at the candidate level.

In this paper, irreducibility and anchor-type conditions should therefore be read as latent-decomposition assumptions.  
As shown in Appendix~\ref{app:partial-completion_theorems}, 
they provide additional identification power after some components have been recovered, but they should not be presented as directly testable from the unlabeled mixtures alone.  
By contrast, marginal independence provides an observable product-marginal check, and Section~\ref{sec:mmd-estimation} introduces an MMD-based criterion for carrying out this check from samples.

\section{Estimation with Maximum Mean Discrepancy}
\label{sec:mmd-estimation}
In this section we develop an estimation method of the vector ${\bm{r}}$ such that 
${\bm r}^T\bm{U}= p_i$ for some $i\in[m]$. For that purpose, the kernel mean embedding (KMM) is employed.

\subsection{MMD criterion and empirical estimator}
\label{subsec:mmd-criterion}
Consider marginal distributions $U_\ell(x, x')$ of $U_\ell(\x)$, 
$\ell \in [L]$, for $(x,x')\in\mathcal{X}_1 \times \mathcal{X}_2$. 
The same construction applies to the marginal densities on 
$\mathcal{X}_s \times \mathcal{X}_t$ for any fixed $s,t\in[d]$ with $s<t$. 
For each pair $\ell,\ell' \in [L]$, we define
\begin{align*}
  U_{\ell \otimes \ell'}(x,x') := U_\ell(x)\, U_{\ell'}(x'),
\end{align*}
that is, $U_{\ell \otimes \ell'}$ is the product of the one-dimensional
marginals of $U_\ell$ and $U_{\ell'}$.
Theorem~\ref{thm:marginal-independence-identifiability} ensures that, under its
assumptions,
\begin{align}
  \label{eq:product-marginal-identity}
  \sum_{\ell\in[L]} r_\ell U_\ell(x,x')
 = \Bigl(\sum_{\ell\in[L]} r_\ell U_\ell(x)\Bigr)
 \Bigl(\sum_{\ell'\in[L]} r_{\ell'} U_{\ell'}(x')\Bigr) 
 = \sum_{\ell,\ell'\in[L]} r_\ell r_{\ell'} U_{\ell \otimes \ell'}(x,x')
\end{align}
holds if and only if $\sum_{\ell} r_\ell U_\ell = p_i$ for some
$i \in R_{12}$.
Our goal in this section is to develop a learning algorithm that estimates
all such weight vectors $\bm{r}$ from samples drawn from $U_\ell$,
$\ell \in [L]$.

Suppose that i.i.d.\ samples $\{\bm{x}_{\ell,i}\}_{i=1}^{n_\ell} \sim U_\ell(x,x')$, 
$\ell \in [L]$, are observed, 
where $\bm{x}_{\ell,i} = (x_{\ell,i}, x'_{\ell,i}) \in \mathcal{X}_1 \times \mathcal{X}_2$. 
The goal is to estimate $\bm{r} = (r_1, \ldots, r_L)^T$ that satisfies \eqref{eq:product-marginal-identity}. 
Instead of the probability densities themselves, we use the maximum mean discrepancy (MMD) 
and its empirical estimate with a characteristic kernel~\cite{sriperumbudur2010hilbert}. 
Let $\mathcal{H}_k$ be the RKHS endowed with the kernel $k$ on 
$\mathcal{X}_1 \times \mathcal{X}_2$. 
The kernel embedding $\phi:\mathcal{S}\rightarrow\mathcal{H}_k$ is defined by 
\begin{align*}
 \phi(q) = \int_{\mathcal{X}} k(\bm{x},\cdot) q(\bm{x})\,\rmd{\x},\quad q\in\mathcal{S}
\end{align*}
The kernel mean embedding for empirical distributions is similarly defined, 
and the same notation $\phi$ is used. 

The maximum mean discrepancy (MMD) from $\mathcal{S}$ to its projection onto $\mathcal{E}$ 
is defined by 
\begin{align*}
 \mathrm{MMD}(q) 
 &:=
 \big\| \phi(q)-\phi(\Pi_{\mathcal{E}}[q]) \big\|_{\mathcal{H}_k},\quad q\in\mathcal{S}, 
\end{align*}
where $\|\cdot\|_{\mathcal{H}_k}$ is the RKHS norm of $\mathcal{H}_k$. 
Since \(k\) is characteristic on probability measures, its mean embedding is injective on finite signed measures with common total mass~\cite{sriperumbudur2010hilbert,simon2018kernel}.
In our problem, our target is to find the parameter $\bm{r}$ that minimizes
\begin{align*}
\mathrm{MMD}(\bm{r}^T\U)
 &=
 \bigg\|
  \sum_{\ell} r_\ell \mathbb{E}_{U_\ell}[k(\bm{x},\cdot)]
 -\sum_{\ell,\ell'} r_\ell r_{\ell'} \mathbb{E}_{U_{\ell\otimes\ell'}}[k((x_{\ell},x_{\ell'}'),\cdot)]
 \bigg\|_{\mathcal{H}_k}. 
\end{align*}
In practice, we compute $ \mathrm{MMD}(\bm{r}^T\widehat{\U})$, that is, 
\begin{align*}
 \mathrm{MMD}(\bm{r}^T\widehat{\U})
 &=
 \bigg\|
 \sum_{\ell}\frac{r_\ell }{n_\ell}\sum_{i=1}^{n_\ell}k(\bm{x}_{\ell,i},\cdot)
 -
 \sum_{\ell,\ell'}\frac{r_\ell}{n_\ell}\frac{r_{\ell'}}{n_{\ell'}}\!\!\!\!
 \sum_{i\in[n_\ell],j\in[n_{\ell'}]}\!\!\!\!\!\!\!\! k((x_{\ell,i},x_{\ell',j}'),\cdot)
 \bigg\|_{\mathcal{H}_k}, 
\end{align*}
where $\widehat{\U}=(\widehat{U}_\ell)_{\ell\in[L]}$ is the collection of the empirical distributions. 

Let us evaluate the gap between $ \mathrm{MMD}(\bm{r}^T\U)$ and $\mathrm{MMD}(\bm{r}^T\widehat{\U})$
over the $(L-1)$-dimensional extended simplex 
\begin{align*}
 \Delta_{\bar{r}} := \bigl\{ \bm{r} \in \mathbb{R}^L \,\big|\, 
 \bm{r}^T \bm{1} = 1,\ \|\bm{r}\|_1 \leq \bar{r}\big\}, 
\end{align*}
where $\bar{r}$ is a large constant greater than or equal to $1$. 
The non-negativity constraint $\bm{r} \ge \bm{0}$ is not imposed in general.

\begin{theorem}
\label{thm:pmmd-uniform-bound}
Let us define $N_{1/2}:=\sum_{\ell\in[L]} n_\ell^{-1/2}$. 
Suppose that the kernel function $k$ on $\mathcal{X}_1\times\mathcal{X}_2$ satisfies
$\sup_{\bm{x}} k(\bm{x},\bm{x})\le K^2$. Then, with probability at least
$1-\delta$,
\begin{align*}
 \sup_{\bm r\in\Delta_{\bar r}}
 \left|\mathrm{MMD}(\bm r^T\U)-\mathrm{MMD}(\bm r^T\widehat\U)\right|
 \le
 10\bar{r}^2K\Bigg(1+\sqrt{\log\frac{2L^2}{\delta}}\Bigg)N_{1/2}.
\end{align*}
\end{theorem}
The proof is deferred to Appendix~\ref{app:mmd-estimation-proofs}. 

As shown in \cite{tolstikhin2016minimax}, the minimax convergence rate of 
the standard MMD estimator with radial kernels is of order $1/\sqrt{n}$, 
where $n$ denotes the sample size.
Theorem~\ref{thm:pmmd-uniform-bound} ensures that, in our setting, 
$\mathrm{MMD}(\bm{r}^{\top}\widehat{\U})$ achieves the same convergence order. 

This result shows that the empirical MMD criterion is uniformly consistent on $\Delta_{\bar r}$. 
When all sample sizes satisfy $n_\ell \asymp n$ for 
some $n$, 
the deviation is of order $O_p(\bar r^2 K L\sqrt{\log{L}}  n^{-1/2})$, so the
minimizer of $\mathrm{MMD}(\bm{r}^T \widehat{\bm{U}})$ approximates the
population minimizer of $\mathrm{MMD}(\bm{r}^T \bm{U})$ with the usual
parametric rate up to the approximate linear factor in $L$.

\subsection[Estimation of Mixing Weights: The Case L=m]{Estimation of Mixing Weights: The Case $L=m$}
\label{subsec:weight-estimation-square}

In this section, 
we assume $L = m$ and that the mixing matrix $\Theta \in [0,1]^{m \times m}$ is
invertible. 
In this case, 
${\bm{U}}=\Theta{\bm p}$ and $R{\bm{U}}={\bm p}$ leads to $R=\Theta^{-1}$. 
The condition \eqref{eq:product-marginal-identity} holds if and only if 
\begin{align*}
  \bm{r} \in \{\, (\Theta^T)^{-1} \bm{e}_i \mid i \in R_{12} \,\}  \subset \Delta_{\bar{r}} 
\end{align*}
for any $\bar{r}>\max_{i\in[m]}\|(\Theta^T)^{-1}\bm e_i\|_1$. 

We define the quartic functions $d(\bm{r})$ and $\widehat{d}(\bm{r})$ as follows: 
\begin{align*}
 d(\bm{r}) = \bigl(\mathrm{MMD}(\bm{r}^T\bm{U})\bigr)^2,
 \qquad  
 \widehat{d}(\bm{r}) = \bigl(\mathrm{MMD}(\bm{r}^T\widehat{\bm{U}})\bigr)^2,
 \qquad 
 \bm{r}\in \Delta_{\bar{r}}. 
\end{align*}
Since $k$ is characteristic, $d(\bm r)=0$ is equivalent to the product-marginal identity
\eqref{eq:product-marginal-identity}. 
Hence, by Theorem~\ref{thm:marginal-independence-identifiability} and the invertibility of $\Theta$, 
the zeros of $d$ on $\Delta_{\bar r}$ are exactly
\[
  \bm r_i := (\Theta^T)^{-1}\bm e_i,\qquad i\in R_{12}.
\]
The following theorems describe the level sets of $d(\bm{r})$ and 
the relationship between $d$ and $\widehat{d}$. 
Their proofs are given in Appendix~\ref{app:mmd-estimation-proofs}. 


\begin{theorem}
 \label{thm:population-level-set}
 Suppose that, for each $i\in R_{12}$, the Hessian matrix of $d(\bm{r})$ at
 $\bm{r}_i$ is positive definite on the tangent space of $\Delta_{\bar{r}}$. 
 Then, for sufficiently small $c>0$, the level set
 \[
   \{\, \bm{r}\in \Delta_{\bar{r}} \mid d(\bm{r}) \le c \,\}
 \]
 consists of $|R_{12}|$ disjoint subsets. 
\end{theorem}
\begin{theorem}
\label{thm:empirical-level-set}
 Suppose that the Hessian matrix of $d(\bm{r})$ at $\bm{r}_i$ is positive definite 
 on the tangent space of $\Delta_{\bar{r}}$ for all $i \in R_{12}$. 
 For some sufficiently small $\epsilon_0 > 0$ and $c_0 > 0$, assume that 
 \begin{align*}
  \|d - \widehat{d}\|_\infty \le \epsilon \le \epsilon_0, 
  \quad \text{and} \quad 
  2\epsilon \le c \le c_0. 
 \end{align*}
 Then $\{\bm{r}\in\Delta_{\bar{r}} \mid \widehat{d}(\bm{r}) \le c\}$ consists of $|R_{12}|$ disjoint subsets, 
 and each subset contains exactly one $\bm{r}_i$, $i \in R_{12}$.
\end{theorem}

Based on Theorem~\ref{thm:pmmd-uniform-bound} and
Theorem~\ref{thm:empirical-level-set}, 
we can identify the number of independent components $|R_{12}|$ 
and the weight vectors $\bm{r}_i$, $i \in R_{12}$, 
using the empirical loss $\mathrm{MMD}(\bm{r}^T\widehat{\bm{U}})$. 
By Theorem~\ref{thm:pmmd-uniform-bound}, 
for sufficiently small $N_{1/2}$, the inequality
\begin{align*}
\big\|\sqrt{d} - \sqrt{\widehat{d}}\big\|_\infty 
 \leq 10(1+\sqrt{t+\log 2L^2}) \,\bar{r}^2 K N_{1/2}
\end{align*}
holds with probability at least $1 - e^{-t}$.
An upper bound for $\mathrm{MMD}(\bm{r}^T\bm{U})$ and $\mathrm{MMD}(\bm{r}^T\widehat{\bm{U}})$ 
is given by $2K\bar{r}^2$ for $\bm{r}\in\Delta_{\bar{r}}$. 
Hence,
\begin{align*}
 \|d - \widehat{d}\|_\infty 
 &\leq 
 40 (1+\sqrt{t+\log 2L^2}) \,\bar{r}^4 K^2 N_{1/2}
\end{align*}
Therefore, with probability at least $1 - e^{-t}$ and for sufficiently small $N_{1/2}$, 
the level set 
\begin{align*}
\bigl\{\bm{r}\in\Delta_{\bar{r}} \,\big|\, (\mathrm{MMD}(\bm{r}^T\widehat{\bm{U}}))^2 
 \leq
 80 (1+\sqrt{t+\log 2L^2}) \,\bar{r}^4 K^2 N_{1/2}
\bigr\}
\end{align*}
consists of $|R_{12}|$ disjoint subsets, and each subset contains one of the points 
$\bm{r}_i$, $i \in R_{12}$.

\subsection[Estimation and Completion of the Mixing Matrix: The Case L=m]{Estimation and Completion of the Mixing Matrix: The Case $L=m$}
\label{subsec:pmmd-algorithm}

We now describe the estimation-and-completion procedure based on the MMD
criterion in Section~\ref{subsec:mmd-criterion} and the weight-vector characterization in
Section~\ref{subsec:weight-estimation-square}.
Throughout this subsection, we assume $L=m$ and that the mixing matrix
$\Theta\in[0,1]^{m\times m}$ is invertible.
In this case, $\bm U=\Theta\bm p$, and hence the relation $R\bm U=\bm p$
implies $R=\Theta^{-1}$.
Accordingly, each weight vector is uniquely determined by
\[
  \bm r_i=(\Theta^T)^{-1}\bm e_i,
  \qquad i\in[m].
\]

The procedure consists of two stages.
The first stage estimates those weight vectors that are detectable from
marginal independence on a prescribed collection of coordinate pairs.
The second stage, when needed, completes unresolved components and the
corresponding columns of the mixing matrix by mixture-proportion and
residual-deflation steps.  The one- and two-unresolved-component cases are
given in Appendix~\ref{app:partial-completion}; cases with more
than two unresolved components require a higher-dimensional residual
vertex-estimation step and are not pursued here.

Fix a set of coordinate pairs
\[
  \mathcal A\subset\{(s,t)\in[d]\times[d]\mid s<t\}.
\]
For each $e=(s,t)\in\mathcal A$, let
\[
  \widehat{\bm U}_e
  :=
  (\widehat U_{e,1},\ldots,\widehat U_{e,m})^T
\]
denote the vector of empirical bivariate marginals on the coordinates
$(x_s,x_t)$, and define
\[
  \widehat d_e(\bm r)
  :=
  \bigl(\mathrm{MMD}(\bm r^T\widehat{\bm U}_e)\bigr)^2,
  \qquad
  \bm r\in\Delta_{\bar r}.
\]
Applying
Theorems~\ref{thm:marginal-independence-identifiability},
\ref{thm:pmmd-uniform-bound}, \ref{thm:population-level-set}, and
\ref{thm:empirical-level-set} to each selected pair $e\in\mathcal A$, we
expect the small-loss local minima of $\widehat d_e(\bm r)$ to concentrate
around the target vectors $\bm r_i$, $i\in R_e$.

The first stage has two modular components: candidate generation and
representative selection.  For each coordinate pair
$e=(s,t)\in\mathcal A$, candidate weight vectors are generated by searching
for small values of the empirical MMD loss $\widehat d_e(\bm r)$ over
$\Delta_{\bar r}$.  Since this empirical objective is generally nonconvex, we
treat the numerical optimizer as a modular component of the procedure.  The
candidate set may be generated by any routine that returns finitely many
low-loss candidates.

The resulting candidates are then reduced to a finite representative set.  In
finite samples, it is often useful to separate the data used for candidate
generation from the data used for candidate selection.  We therefore allow an
optional training/selection split: the training part generates candidates, and
the selection part, typically a validation split, scores and selects stable
representatives.  If no split is used, the same empirical distributions are
used in both steps.  The selection rule is required to retain candidates with
small selection loss while avoiding nearly duplicate representatives.  This
diversity requirement is useful when $|R_e|>1$, because several components may
be marginally independent for the same coordinate pair.

The estimation procedure is summarized in Algorithm~\ref{alg:pmmd}.

\begin{algorithm}[!t]
 \caption{PM-MMD: Weight estimation by pairwise MMD and representative selection}
\label{alg:pmmd}
\begin{algorithmic}[1]
\REQUIRE Samples $\{\x_{\ell,i}\}_{i=1}^{n_\ell}\sim U_\ell$ for $\ell\in[m]$;
a constant $\bar r\ge 1$;
a set of coordinate pairs $\mathcal A$;
a candidate-generation routine on $\Delta_{\bar r}$;
a representative-selection rule;
upper bounds $q_{\max}$ and $K_{\max}\le m$.
\ENSURE A representative set
$\mathcal C\subset\Delta_{\bar r}$ and plug-in estimates
$\{\widehat p_{\bm r}:=\bm r^T\widehat{\bm U}\mid \bm r\in\mathcal C\}$.

\STATE Optionally split the samples into a training part and a selection part.
If no split is used, use the same samples in both steps.
Construct the empirical mixtures $\widehat{\bm U}$ from all available samples.

\STATE Initialize the aggregated pairwise representative set
$\mathcal V\gets\emptyset$.

\FOR{each coordinate pair $e=(s,t)\in\mathcal A$}
    \STATE Construct empirical bivariate marginals
    $\widehat{\bm U}_e^{\mathrm{tr}}$ and
    $\widehat{\bm U}_e^{\mathrm{sel}}$ from the training and selection parts.

    \STATE Define the training loss
    \[
      \widehat d_e^{\mathrm{tr}}(\bm r)
      :=
      \bigl(\mathrm{MMD}(\bm r^T\widehat{\bm U}_e^{\mathrm{tr}})\bigr)^2,
      \qquad
      \bm r\in\Delta_{\bar r}.
    \]

    \STATE Apply the candidate-generation routine to
    $\widehat d_e^{\mathrm{tr}}$ over $\Delta_{\bar r}$, and denote the
    resulting finite candidate set by $\mathcal V_e^0$.

    \STATE Score each $\bm r\in\mathcal V_e^0$ by
    \[
      \widehat s_e(\bm r)
      :=
      \bigl(\mathrm{MMD}(\bm r^T\widehat{\bm U}_e^{\mathrm{sel}})\bigr)^2.
    \]

\STATE Apply the representative-selection rule to
$(\mathcal V_e^0,\widehat s_e)$ and retain at most $q_{\max}$
pairwise representatives.  Denote the selected set by $\mathcal C_e$.

\STATE Update the aggregated representative set
\[
  \mathcal V \gets \mathcal V\cup\mathcal C_e .
\]
\ENDFOR

\STATE Apply the representative-selection rule to the aggregated set
$\mathcal V$ and retain at most $K_{\max}$ global representatives.
Denote the resulting set by $\mathcal C$.

\STATE For each $\bm r\in\mathcal C$, set
$\widehat p_{\bm r}:=\bm r^T\widehat{\bm U}$.

\STATE \textbf{return} $\mathcal C$ and
$\{\widehat p_{\bm r}\mid \bm r\in\mathcal C\}$.
\end{algorithmic}
\end{algorithm}

In practice, we set $K_{\max}=m$.  The parameter $q_{\max}$ controls how many
representatives are retained from each coordinate pair.  A natural default is
$q_{\max}=m$, which keeps the pairwise candidate sets large enough to cover
the case $|R_e|>1$; the final aggregation step then reduces the union of
pairwise representatives to at most $m$ global representatives.

In the square case, the global representative-selection step can also use the
structural implication $R=\Theta^{-1}$.  If a candidate subset
$\mathcal C=\{\bm r_1,\ldots,\bm r_m\}$ is selected, write
$R_{\mathcal C}:=(\bm r_1,\ldots,\bm r_m)^T$.  Since a valid inverse should
satisfy $R_{\mathcal C}^{-1}\approx\Theta$, the matrix
$R_{\mathcal C}^{-1}$ should be close to a nonnegative row-stochastic matrix.
This motivates the following condition-aware global score over $m$-subsets of
the pairwise representatives:
\begin{equation}
\label{eq:stable-global-selection-score}
\begin{aligned}
\mathrm{Score}(\mathcal C)
&=
\frac{1}{m}\sum_{\bm r\in\mathcal C}
\frac{\widehat s(\bm r)}{s_{\mathrm{med}}+\varepsilon}
+\lambda_{\mathrm{cond}}\log \operatorname{cond}(R_{\mathcal C})  \\
&\quad
+\lambda_{\mathrm{neg}}
\bigl\|[-R_{\mathcal C}^{-1}]_+\bigr\|_1
+\lambda_{\Delta}
\bigl\|R_{\mathcal C}^{-1}
      -\Pi_{\Delta}(R_{\mathcal C}^{-1})\bigr\|_F^2 .
\end{aligned}
\end{equation}
Here $\widehat s(\bm r)$ is the validation PM-MMD score of the pairwise
representative that produced $\bm r$, $s_{\mathrm{med}}$ is the median
validation score in the aggregated candidate set, $\Pi_{\Delta}$ denotes
row-wise projection onto the probability simplex, and $[\cdot]_+$ is the
positive part applied entrywise.  The additional terms favor candidate sets
whose inverse is well conditioned and compatible with a mixing matrix.
Appendix~\ref{app:stable-selection-guarantee} gives a deterministic
selection-stability statement: if the target candidate subset has a positive
population score gap and the empirical score is uniformly close to its
population counterpart, then the condition-aware selector returns that subset
from the candidate pool.  The resulting mixing-matrix error is then controlled
by a standard inverse-perturbation bound.

\begin{remark}[Implementation choices]
The optimization and representative-selection steps in
Algorithm~\ref{alg:pmmd} are modular.  The candidate-generation routine
may be implemented by homotopy search, multi-start constrained optimization,
projected-gradient methods, or any other solver that returns low-loss
candidates on $\Delta_{\bar r}$.  The representative-selection rule is also
not tied to a particular clustering method such as $k$-means.  For instance,
one may cluster low-selection-loss candidates and select representative
medoids, or use a greedy diversity rule that scans candidates in increasing
order of selection loss and keeps a candidate only if it is sufficiently far
from previously selected representatives.
In Section~\ref{sec:numerical-experiments}, 
the controlled and semi-synthetic experiments use the latter greedy
validation-diversity rule.  The raw gated-pool three-component DLBCL experiment uses the
condition-aware global rule in~\eqref{eq:stable-global-selection-score} after
pairwise representative selection.  The concrete filtering, deduplication,
validation scoring, and aggregation steps are described in
Appendix~\ref{app:implementation-details}.
\end{remark}

\paragraph{Operational diagnostic.}
The marginal independence is directly scored on
candidate affine combinations.  The held-out PM-MMD value
$\widehat s_e(\bm r)$ in Algorithm~\ref{alg:pmmd} is therefore both a
selection criterion and a diagnostic for whether the candidate combination
$\bm r^T\bm U$ exhibits the product-marginal signal on pair $e=(s,t)$.

\paragraph{Completion after representative selection.}
Suppose that, after representative selection and reindexing, the first-stage
outputs are written as
\[
  \{(\widehat{\bm r}_i,\widehat p_i)\}_{i=1}^K,
  \qquad
  \widehat p_i:=\widehat{\bm r}_i^T\widehat{\bm U}.
\]
If \(K=m\), then all weight vectors have been recovered and we set
\[
  \widehat R
  :=
  (\widehat{\bm r}_1,\ldots,\widehat{\bm r}_m)^T,
  \qquad
  \widehat\Theta:=\widehat R^{-1}.
\]
If \(K=m-1\), exactly one component remains unresolved, and the completion
rule is Algorithm~\ref{alg:completion-one-missing} in
Appendix~\ref{app:one-missing-component}. 
If \(K=m-2\), the residual two-component completion rule is Algorithm~\ref{alg:completion-two-missing} in Appendix~\ref{app:two-missing-components}.  
These completion steps use mixture-proportion estimation and row-sum or residual-deflation identities, not additional marginal-independence minimization.

\section{Near Marginal Independence Property}
\label{sec:near-marginal-independence}

We consider a nearly independent setting. 
For simplicity, we restrict attention to marginal densities on 
$\mathcal{X}_1\times\mathcal{X}_2$. 
Assume that, for some $i$, the joint density $p_i(x_1,x_2)$ is close to, but not equal to, 
$p_i(x_1)p_i(x_2)$. 

The mixture densities generating the observations are given by
\begin{align*}
 U_\ell(x_1,x_2)
 &=
 \sum_{i=1}^{m}\theta_{\ell i}p_i(x_1,x_2),
 \qquad \ell\in[L].
\end{align*}
Throughout this section, we assume $L=m$ and that the mixing matrix
$\Theta=(\theta_{\ell i}) \in [0,1]^{L\times L}$ is invertible. 
Then $(\Theta^T)^{-1}$ is well defined. 
The optimal weight vector is obtained by minimizing 
\begin{align*}
 \mathrm{MMD}(\bm{r}^T\U) 
 =
 \bigl\|
 \phi(\bm{r}^T\U)
 -
 \phi\bigl(\Pi_{\mathcal{E}}[\bm{r}^T\U]\bigr)
 \bigr\|_{\mathcal{H}_k}. 
\end{align*}
In the nearly independent setting, the minimum value of $\mathrm{MMD}(\bm{r}^T\U)$ may not be zero. 
Let us quantify the discrepancy from the marginal independence assumption. 
For a given $\eta>0$, define 
\begin{align*}
 R_{12}^{(\eta)}
 =
 \Bigl\{
 i\in[m]\Bigm|
 \bigl\|\phi(p_i)-\phi\bigl(\Pi_{\mathcal{E}}[p_i]\bigr)\bigr\|_{\mathcal{H}_k}
 < K\eta
 \Bigr\}.
\end{align*}
We investigate the landscape of $\mathrm{MMD}(\bm{r}^T\U)$ under the above perturbation. 

We also define, using the same mixing matrix $\Theta$, a modified mixture
$\bar{U}_\ell$ by replacing, for each $i\in R_{12}^{(\eta)}$, the component
$p_i(x_1,x_2)$ with its independent counterpart $p_i(x_1)p_i(x_2)$ while
keeping the other components unchanged, that is,
\begin{align*}
 \bar{U}_\ell(x_1,x_2) 
 &=
 \sum_{i\in R_{12}^{(\eta)}} \theta_{\ell i}p_i(x_1)p_i(x_2)
 +
 \sum_{i\not\in R_{12}^{(\eta)}} \theta_{\ell i}p_i(x_1,x_2),
 \qquad \ell\in[L]. 
\end{align*}
Then 
\[
 \Pi_{\mathcal{E}}[U_\ell]
 =
 \Pi_{\mathcal{E}}[\bar{U}_\ell]
\]
holds for all $\ell$. 
Let $\bm{U}=(U_1,\ldots,U_L)^T$ and $\bar{\bm{U}}=(\bar{U}_1,\ldots,\bar{U}_L)^T$, 
and for each $i\in R_{12}^{(\eta)}$ define the weight vector $\bm{r}_i = (\Theta^T)^{-1}\bm{e}_i$. 
Then 
\begin{align*}
 p_i(x_1,x_2)&=\bm{r}_i^T\bm{U}(x_1,x_2),
 \qquad 
 p_i(x_1)p_i(x_2)=\bm{r}_i^T\bar{\bm{U}}(x_1,x_2),
 \qquad i\in R_{12}^{(\eta)}.
\end{align*}

The next lemma shows that replacing weakly dependent components by their independent counterparts changes 
the MMD objective only slightly. 
\begin{lemma}
 \label{lem:near-independent-replacement}
 Let the kernel satisfy $\sup_{\x}k(\x,\x)\le K^2$. 
 For any $\bm{r}\in\Delta_{\bar{r}}$, it holds that 
 \begin{align*}
  \bigl|\mathrm{MMD}(\bm{r}^T\U)-\mathrm{MMD}(\bm{r}^T\bar{\U})\bigr| < K\bar{r}\eta .
 \end{align*}
\end{lemma}
The proof is deferred to Appendix~\ref{app:mmd-estimation-proofs}. 

Let us define $d(\bm{r})$, $\bar{d}(\bm{r})$ and $\widehat{d}(\bm{r})$ by 
\begin{align*}
  d(\bm{r}):=(\mathrm{MMD}(\bm{r}^T\U))^2, \quad
  \bar{d}(\bm{r}):=(\mathrm{MMD}(\bm{r}^T\bar{\U}))^2, \quad\text{and}\quad 
  \widehat{d}(\bm{r}):=(\mathrm{MMD}(\bm{r}^T\widehat{\U}))^2, 
\end{align*}
where $\widehat{\U}$ is the collection of empirical distributions defined by samples from
$U_1,\ldots,U_L$. 
Lemma~\ref{lem:near-independent-replacement} then yields 
\begin{align}
 \label{eq:near-independence-loss-gap}
  \|d-\bar{d}\|_\infty < 4K^2\bar{r}^3\eta
\end{align}
on $\Delta_{\bar{r}}$, since both $\mathrm{MMD}(\bm{r}^T\U)$ 
and $\mathrm{MMD}(\bm{r}^T\bm{\bar U})$ for $\bm{r}\in\Delta_{\bar r}$ are bounded by $2K\bar r^2$. 
\begin{theorem}
 \label{thm:near-independence-stability}
Assume $\sup_{\x}k(\x,\x)\leq K^2$. 
For each $i \in R_{12}^{(\eta)}$, let $\bm{r}_i := (\Theta^T)^{-1}\bm{e}_i$. 
Assume that the zero characterization of Theorem~\ref{thm:marginal-independence-identifiability}
holds for the modified exact model
\(\bar{\bm{U}}\), with the index set \(R_{12}^{(\eta)}\).
Assume that, for each $i\in R_{12}^{(\eta)}$, there exists $\lambda_i>0$ such that, 
on some neighborhood of $\bm{r}_i$ within $\Delta_{\bar{r}}$,  
all eigenvalues of the Hessian matrix of $\bar{d}(\bm{r})/K^2$, 
restricted to the tangent space of $\Delta_{\bar{r}}$, are bounded below by $\lambda_i$. 
Then the following statements hold. 
 \begin{enumerate}
  \item[(a)] For some small $c_0>0$, assume that 
	     \begin{align*}
	      8K^2\bar{r}^3\eta \leq c< c_0. 
	     \end{align*}
	     Then, the level set 
	     $$\{\bm{r}\in\Delta_{\bar{r}}\,\mid\,d(\bm{r})\leq c\}$$ 
	     consists of $|R_{12}^{(\eta)}|$ 
	     disjoint subsets, 
	     and each subset contains exactly one $\bm{r}_i$, $i\in R_{12}^{(\eta)}$. 
	     Let $\bm{r}_*$ be
	     a local minimizer of $d(\bm{r})$ 
	     on $\{\bm{r}\in\Delta_{\bar{r}}\,\mid\,d(\bm{r})\leq c\}$. 
	     Then there exists a $p_i$, $i\in R_{12}^{(\eta)}$, such that 
	     \begin{align*}
	      \frac{\big\|\phi(\bm{r}_*^T\U)-\phi(p_i)\big\|_{\mathcal{H}_k}}{K}
	      \leq 
	      3\bar{r}\eta\sqrt{\frac{m}{\lambda_i}}. 
	     \end{align*}
	     
  \item[(b)] Suppose that 
	     \begin{align*}
	      \sup_{\bm{r}\in\Delta_{\bar{r}}}\bigl|\mathrm{MMD}(\bm{r}^T \U)-\mathrm{MMD}(\bm{r}^T 
	      \widehat{\U})\bigr| < K\zeta_n
	     \end{align*}
	     holds with high probability. 
	     For some small $c_0>0$, assume that 
	     $8K^2\bar{r}^2(\bar{r}\eta+\zeta_n)\leq c'<c_0$. 
	     Then, the level set 
	     $$\{\bm{r}\in\Delta_{\bar{r}}\,\mid\,\widehat{d}(\bm{r})\leq c'\}$$
	     consists of 
	     $|R_{12}^{(\eta)}|$ disjoint subsets, 
	     and each subset contains exactly one $\bm{r}_i$, $i\in R_{12}^{(\eta)}$. 
	     Let $\widehat{\bm{r}}$ be
	     any local minimizer of $\widehat{d}(\bm{r})$ 
	     on $\{\bm{r}\in\Delta_{\bar{r}}\,\mid\,\widehat{d}(\bm{r})\leq c'\}$. 
	     Then there exists a $p_i$, $i\in R_{12}^{(\eta)}$, such that 
	     \begin{align*}
	      \frac{\big\|\phi(\widehat{\bm{r}}^T\U)-\phi(p_i)\big\|_{\mathcal{H}_k}}{K}
	      \leq  
	      3(\bar{r}\eta+\zeta_n)\sqrt{\frac{m}{\lambda_i}} 
	     \end{align*}
	     holds with high probability. 
 \end{enumerate}
\end{theorem}
The proof is deferred to Appendix~\ref{app:mmd-estimation-proofs}. 

\begin{remark}
The constant $c_0$ in Theorem~\ref{thm:near-independence-stability} 
is determined from $\bar{d}$ in a similar way as in the proof of Theorem~\ref{thm:empirical-level-set}.
\end{remark}

These results show that our identifiability and estimation guarantees remain stable under small violations
of the marginal independence assumptions. 
They also show that the proposed independence criterion continues to select mixtures that stay
close to the true components.

\section{Numerical Experiments}
\label{sec:numerical-experiments}

This section reports a selective set of experiments: 
controlled sanity checks, a three-component flow-cytometry task from raw gated pools, and a binary flow-cytometry calibration
against BBE~\cite{garg2021mixture},
DEDPUL~\cite{ivanov2019dedpul}, and
KM2~\cite{ramaswamy2016mixture}.
Full sample-size plots, semi-synthetic benchmark
tables, oracle clean-component comparisons, and additional implementation
details are deferred to Appendix~\ref{app:additional-numerical-results}
and Appendix~\ref{app:implementation-details}.

Throughout, the observed distributions are unlabeled mixtures sharing the same
latent components.  The true mixing matrix is used only for evaluation, such as
relative Frobenius error and diagnostic vectors $\bm r^T\Theta$.  It is not used
in PM-MMD candidate generation or representative selection.  Unless otherwise
stated, PM-MMD uses an optimization split to generate candidates and an
independent validation split for representative selection.

\subsection{Controlled sanity checks}

The controlled experiments are included mainly to verify the identifying
mechanism.  First, a symmetric Gaussian design violates the no-cancellation
condition in Theorem~\ref{thm:marginal-independence-identifiability}; as expected,
validation cannot remove the resulting population-level spurious low-MMD
solutions.  Second, asymmetric Gaussian designs satisfying the intended
pairwise no-cancellation structure show decreasing error and stable vertex
recovery.  Third, a non-Gaussian design checks that PM-MMD is exploiting
marginal independence rather than Gaussian cluster geometry.  The full
sample-size table, simplex plots, and data-generating parameters are given in
Appendices~\ref{app:controlled-synthetic-results},
\ref{app:gmm-designs}, and~\ref{app:nongauss-design}.

\begin{table}[t]
\centering
\small
\caption[Controlled sanity-check experiments.]{
Controlled sanity-check experiments.  Errors are relative Frobenius errors of
estimated mixing matrices, reported as mean $\pm$ standard error over ten seeds.
``Rec.'' denotes PM-MMD vertex recovery: the number of seeds for which the
selected diagnostic directions recover the intended simplex vertices, 
when such a recovery diagnostic applies.
``MMD degeneracy'' refers to low-MMD diagnostic directions that do not recover component vertices.}
\label{tab:controlled-summary-main}
\begin{tabular}{llccc}
\hline
Design & Role & $n$ & rel. Frobenius err.  & Rec. \\
\hline
Symmetric GMM & violates no-cancellation & -- & MMD degeneracy & fails \\
GMM, $|R_e|=1$ & identifiable singleton pairs & 20000 & $0.0323\pm0.0033$ & $10/10$ \\
GMM, $|R_e|=2$ & multi-component pair sets & 10000 & $0.0594\pm0.0089$ & $10/10$ \\
Non-Gaussian & non-Gaussian components & 1000 & $0.1338\pm0.0166$ & $10/10$ \\
\hline
\end{tabular}
\end{table}

Table~\ref{tab:controlled-summary-main} supports two points.  The
no-cancellation condition is a real identifiability condition rather than a
numerical artifact, and PM-MMD remains effective when the components are
skewed, multimodal, or heavy-tailed.  In the non-Gaussian design, the best
clustering/factorization reference at $n=1000$ has error $0.4444\pm0.0019$
whereas PM-MMD has error $0.1338\pm0.0166$; see
Table~\ref{tab:nongauss-controlled} in the appendix.

\subsection{Three-component DLBCL experiment from raw gated pools}
\label{subsec:raw-dlbcl-three-component}

We next evaluate a three-component DLBCL flow-cytometry experiment from raw gated pools.  This
experiment tests the multi-component setting directly: it uses three
manual-gated component pools but does not artificially enforce marginal
independence by coordinate resampling.  The manual labels are used only to
construct the empirical component pools and to evaluate recovery; no labels,
clean component samples, or mixing weights are supplied to PM-MMD or to the
non-oracle baselines.

The data are the \texttt{cytometree} DLBCL subset~\cite{commenges2018cytometree},
with marker coordinates FL1, FL2, and FL4.  The binary calibration in
Section~\ref{subsec:dlbcl-binary-prior-shift} uses the two largest gated
populations from the same subset.
We use the three gated populations with class counts $47$, $604$, and
$4873$, and form three unlabeled mixtures with
\[
\Theta
=0.70 I_3+\frac{0.30}{3}\mathbf 1\mathbf 1^T
=
\begin{pmatrix}
0.8&0.1&0.1\\
0.1&0.8&0.1\\
0.1&0.1&0.8
\end{pmatrix}.
\]
PM-MMD uses the marker pairs FL1--FL4 and FL2--FL4 (zero-indexed pairs
$(0,2)$ and $(1,2)$), and uses the condition-aware global selection rule in
\eqref{eq:stable-global-selection-score}.  The pair choice is treated as a
prescribed marker-pair choice; the true mixing matrix and labels are used only
for evaluation.

We compare PM-MMD with several baseline families.  First, we use distributional
baselines that directly operate on the observed mixture samples: Gaussian
mixture fitting via EM~\cite{dempster1977maximum}, $k$-means~\cite{macqueen1967some},
spectral clustering~\cite{ng2002spectral}, histogram NMF~\cite{lee1999learning},
histogram archetypal analysis~\cite{cutler1994archetypal}, and a
pairwise-residue MPE heuristic.  
The pairwise-residue baseline is a same-information heuristic. 
It uses only the observed unlabeled mixtures, but the binary MPE guarantees do not directly extend to this pairwise reduction in the multi-class setting. 
The method applies a BBE-style binary MPE subroutine~\cite{garg2021mixture} to ordered pairs of observed mixtures, 
constructs pseudo-residual distributions following the mutual-contamination decontamination viewpoint~\cite{blanchard14:_decon_mutual_contam_model,katzsamuels2019decontamination},
and estimates the mixing matrix from these residual candidates.  It is included
to test the natural pairwise-reduction strategy suggested by binary MPE.

We also include classifier-posterior transition-matrix baselines, following the
transition-matrix viewpoint in noisy-label learning~\cite{liu2016importance,patrini2017making,li2021volminnet}.
For this comparison only, the mixture index is denoted by $Y'$ and treated as
an observed noisy label.  We fit a cross-fitted multinomial logistic classifier
for $P(Y'\mid X)$ and use the resulting posterior vectors
$\widehat P(Y'=\cdot\mid X=x)$ to estimate the forward transition matrix
$T_{ji}=P(Y'=j\mid Y=i)$.  We evaluate four lightweight posterior-simplex
constructions: anchor posterior, posterior archetypes, posterior $k$-means, and
minimum-volume simplex fitting.  The first three use high-posterior points,
extreme posterior points, and posterior cluster centers, respectively.  The
minimum-volume variant fits a small-volume simplex to the posterior cloud,
analogously to minimum-volume transition estimation.  Given the empirical prior
$\widehat q_j=\widehat P(Y'=j)$, all transition baselines compute
$\widehat\pi$ from $\widehat q=\widehat T\widehat\pi$ and then set
\[
  \widehat\Theta_{ji}
  =
  \frac{\widehat T_{ji}\widehat\pi_i}{\widehat q_j}.
\]
They use the same observed samples as PM-MMD, but their identifying structure is
posterior-simplex geometry rather than marginal independence.

\begin{table}[t]
\centering
\small
\caption[Raw gated-pool three-component DLBCL flow-cytometry experiment.]{
Raw gated-pool three-component DLBCL flow-cytometry experiment at $\rho=0.70$.
Entries are relative Frobenius errors of the estimated mixing matrix over ten
random seeds.  The oracle binary MPE row uses clean component samples and is
included only as an oracle calibration.  The transition rows use
classifier-posterior estimates of $P(Y'\mid X)$ and posterior-simplex estimates
of $P(Y'\mid Y)$ under the same observed-data sample budget as PM-MMD.  All
other non-oracle methods directly use the observed mixture samples.  
At this $\rho=0.70$ setting, the min-volume transition baseline has the lowest
non-oracle mean error, while PM-MMD with stable selection has the second-lowest
non-oracle mean error.
``Max vtx. dist.'' denotes the mean maximum distance of
diagnostic vectors to their nearest simplex vertices, and ``Cov.'' denotes the
mean number of distinct nearest vertices.  Boldface and underlining mark the
lowest and second-lowest non-oracle mean errors, respectively.}
\label{tab:raw-dlbcl-three-component}
\resizebox{\linewidth}{!}{
\begin{tabular}{lccccc}
\hline
Method & Mean $\pm$ SE & Median & Max & Max vtx. dist. & Cov. \\
\hline
\multicolumn{6}{l}{\textit{Proposed method and ablation}}\\
PM-MMD, stable selection
& \secondval{0.1078 \pm 0.0121} & $0.1028$ & $0.1877$ & $0.1367$ & $3.0$ \\
PM-MMD, greedy selection
& $0.2217 \pm 0.0756$ & $0.1240$ & $0.7367$ & $0.5281$ & $3.0$ \\
\hline
\multicolumn{6}{l}{\textit{Noisy-label transition baselines}}\\
Transition matrix, min-volume simplex
& \bestval{0.0984 \pm 0.0038} & $0.0923$ & $0.1254$ & $0.1262$ & $3.0$ \\
Transition matrix, anchor posterior
& $0.1487 \pm 0.0042$ & $0.1472$ & $0.1818$ & $0.1604$ & $3.0$ \\
Transition matrix, posterior archetypes
& $0.2052 \pm 0.0025$ & $0.2063$ & $0.2215$ & $0.2304$ & $3.0$ \\
Transition matrix, posterior $k$-means
& $0.2787 \pm 0.0054$ & $0.2838$ & $0.2952$ & $0.7328$ & $3.0$ \\
\hline
\multicolumn{6}{l}{\textit{MPE-based methods}}\\
Oracle binary MPE (BBE-style)
& $0.0866 \pm 0.0114$ & $0.0901$ & $0.1344$ & $0.1226$ & $3.0$ \\
Pairwise-residue MPE (BBE-style)
& $0.1818 \pm 0.0047$ & $0.1749$ & $0.2069$ & $0.1874$ & $3.0$ \\
\hline
\multicolumn{6}{l}{\textit{Classical clustering and factorization baselines}}\\
Histogram NMF
& $0.2670 \pm 0.0081$ & $0.2775$ & $0.3015$ & $0.2444$ & $3.0$ \\
Histogram archetypal analysis
& $0.3015 \pm 0.0000$ & $0.3015$ & $0.3015$ & $0.2449$ & $3.0$ \\
$k$-means
& $0.4226 \pm 0.0072$ & $0.4197$ & $0.4574$ & $0.9746$ & $3.0$ \\
Gaussian mixture fit
& $0.4433 \pm 0.0457$ & $0.5088$ & $0.6013$ & $1.9860$ & $3.0$ \\
Spectral clustering
& $1.1291 \pm 0.0147$ & $1.1226$ & $1.2063$ & $1.74\times 10^3$ & $2.0$ \\
\hline
\end{tabular}}
\end{table}

Table~\ref{tab:raw-dlbcl-three-component} shows that 
the min-volume transition baseline has the lowest non-oracle mean error at the current setting.
PM-MMD with stable global selection has the second-lowest non-oracle mean error.
It gives lower error than pairwise-residue MPE, NMF, archetypal analysis,
Gaussian mixture fitting, $k$-means, and spectral clustering.  
This result suggests that the classifier-posterior geometry of $Y'$ contains useful information for these DLBCL markers 
when $\rho$ is moderately large.  It also shows that marginal independence provides a different recovery signal from posterior-simplex geometry.  The oracle BBE-style baseline remains a calibration because it uses clean component samples.


The relative performance changes with the mixture-separation parameter.  To
check that the conclusion is not an artifact of the single setting,
we repeat the same raw gated-pool DLBCL experiment with
\[
\Theta=\Theta(\rho)=\rho I_3+\frac{1-\rho}{3}\mathbf 1\mathbf 1^T.
\]
Table~\ref{tab:raw-dlbcl-rho-sweep-main} reports a compact subset of the
sweep using an expanded PM-MMD search radius $\bar r=10$; the full sweep is
shown in Figure~\ref{fig:raw-dlbcl-rho-sweep} in
Appendix~\ref{app:raw-dlbcl-rho-sweep}.

\begin{table}[tb]
\centering
\small
\caption[Raw DLBCL mixture-separation sweep.]{
Raw gated-pool DLBCL mixture-separation sweep.  Entries are relative Frobenius
errors, reported as mean $\pm$ standard error over ten seeds.  The table shows
four representative values of $\rho$; the full sweep is shown in
Appendix~\ref{app:raw-dlbcl-rho-sweep}.  PM-MMD uses stable selection and
$\bar r=10$ in this sweep.  Boldface and underlining mark the lowest and
second-lowest mean errors in each row, respectively.}
\label{tab:raw-dlbcl-rho-sweep-main}
\resizebox{\linewidth}{!}{
\begin{tabular}{ccccc}
\hline
$\rho$ & PM-MMD, stable & Transition min-volume & Transition anchor & Pairwise-residue MPE \\
\hline
$0.40$ & \bestval{0.2070 \pm 0.0175} & $0.2541 \pm 0.0071$ & \secondval{0.2507 \pm 0.0099} & $0.3705 \pm 0.0064$ \\
$0.50$ & \bestval{0.1679 \pm 0.0177} & \secondval{0.1913 \pm 0.0055} & $0.2207 \pm 0.0056$ & $0.3051 \pm 0.0066$ \\
$0.60$ & \secondval{0.1503 \pm 0.0181} & \bestval{0.1387 \pm 0.0029} & $0.1873 \pm 0.0040$ & $0.2436 \pm 0.0071$ \\
$0.70$ & \secondval{0.1078 \pm 0.0121} & \bestval{0.0984 \pm 0.0038} & $0.1487 \pm 0.0042$ & $0.1818 \pm 0.0047$ \\
\hline
\end{tabular}}
\end{table}

The sweep shows that the relative ordering depends on the mixture-separation
parameter.  At small separations, PM-MMD has larger variability because the
signed inverse weights are large and empirical PM-MMD minimization can produce
more spurious low-loss candidates.  At intermediate separations, where
$\rho=0.40$ and $0.50$, PM-MMD has the lowest mean error among the listed
non-oracle methods.  In this range, the marginal-independence signal is usable,
whereas the posterior-simplex geometry is less pronounced.  For larger
separations, where $\rho=0.60$ and $0.70$, min-volume transition estimation has
the lowest mean error.  This is consistent with the fact that the mixture index
$Y'$ becomes more informative about the latent component as the mixtures become
more separated.  These results suggest that PM-MMD and classifier-based
transition estimation exploit different recovery signals.  Additional
implementation details and post-hoc diagnostic remarks are reported in
Appendix~\ref{app:raw-dlbcl-details}.

\FloatBarrier

\subsection{Binary flow-cytometry calibration against recent MPE methods}
\label{subsec:dlbcl-binary-prior-shift}

As an additional real-data-based binary calibration, we use the two largest
manual-gated DLBCL cell populations, with class counts $4873$ and $604$.  For
each shift strength $\rho$, we generate two unlabeled mixtures with
\[
\Theta=\Theta(\rho)=\rho I_2+\frac{1-\rho}{2}\mathbf 1\mathbf 1^T.
\]
We use the same representative grid $\rho\in\{0.40,0.50,0.60,0.70\}$ as in
Table~\ref{tab:raw-dlbcl-rho-sweep-main} to align the mixture-separation
parameter, while noting that this binary calibration is not row-wise comparable
to the three-component experiment.  Because $m=2$, the mutual-contamination
identities for binary MPE apply directly.  We compare PM-MMD with
KM2~\cite{ramaswamy2016mixture}, DEDPUL~\cite{ivanov2019dedpul}, and
BBE~\cite{garg2021mixture}, using only the two unlabeled mixtures.

\begin{table}[H]
\centering
\small
\caption[DLBCL flow-cytometry prior-shift experiment.]{
DLBCL flow-cytometry prior-shift experiment with two manual-gated cell
populations.  Entries are relative Frobenius errors of the estimated mixing
matrix, reported as mean $\pm$ standard error over ten random seeds.  Boldface
and underlining mark the lowest and second-lowest errors in each row.  The last
column reports the PM-MMD vertex recovery rate, defined as recovery of both
vertices with maximum vertex distance at most $0.35$.}
\label{tab:dlbcl-binary-prior-shift}
\begin{tabular}{lccccc}
\hline
Shift $\rho$ & PM-MMD & BBE & DEDPUL & KM2 & PM rec. \\
\hline
0.40
& 0.2065 $\pm$ 0.0823
& \bestval{0.0314 \pm 0.0046}
& 0.1121 $\pm$ 0.0141
& \secondval{0.0370 \pm 0.0049}
& $9/10$ \\
0.50
& 0.0623 $\pm$ 0.0161
& \secondval{0.0310 \pm 0.0060}
& 0.0868 $\pm$ 0.0103
& \bestval{0.0291 \pm 0.0052}
& $10/10$ \\
0.60
& 0.0522 $\pm$ 0.0138
& \bestval{0.0148 \pm 0.0030}
& 0.0750 $\pm$ 0.0061
& \secondval{0.0287 \pm 0.0054}
& $10/10$ \\
0.70
& 0.0392 $\pm$ 0.0099
& \bestval{0.0139 \pm 0.0029}
& 0.0523 $\pm$ 0.0075
& \secondval{0.0319 \pm 0.0055}
& $10/10$ \\
\hline
\end{tabular}
\end{table}

Table~\ref{tab:dlbcl-binary-prior-shift} confirms that binary MPE methods 
gives lower errors in this favorable two-component setting: BBE is best for three shifts,
and KM2 is slightly best at $\rho=0.50$.  
PM-MMD improves as the two mixtures separate and recovers both vertices in all ten seeds for $\rho\ge0.50$, but it is less accurate than the best binary MPE method. 
This experiment therefore calibrates PM-MMD against binary MPE baselines in their natural two-component setting. 

\paragraph{Additional semi-synthetic and oracle comparisons.}
Appendix~\ref{app:semisynthetic-results} reports the full
semi-synthetic benchmark tables.  Those results show that PM-MMD is most useful
when marginal-independence structure is present and simple Gaussian cluster
geometry is not dominant.  The same appendix also reports oracle
clean-component KM2, DEDPUL, and BBE comparisons, which give accurate estimates when clean component samples are supplied but use more information than PM-MMD.

\FloatBarrier
\section{Concluding Remarks}
\label{sec:conclusion}

This paper studied unlabeled multi-mixture decontamination through independence
and marginal-independence signals.  We established identifiability conditions
under which marginally independent components can be recovered by signed affine
combinations of the observed mixtures.  We also proposed an empirical procedure
based on product-marginal MMD and validation-based representative selection.

The numerical experiments illustrate both the identifying mechanism and the
scope of the method.  The controlled experiments show that the no-cancellation
condition is essential, and that PM-MMD can use marginal-independence structure
beyond Gaussian cluster geometry.  The raw gated-pool three-component DLBCL
experiment provides a real-data-based multi-component check without artificially
constructing independent coordinates.  In this task, min-volume transition
estimation gives lower errors at larger mixture separations, whereas PM-MMD
gives lower errors at intermediate separations among the non-oracle methods
considered.  This comparison separates two recovery mechanisms:
posterior-simplex geometry for classifier-based transition baselines, and
marginal independence for PM-MMD.

The binary DLBCL prior-shift experiment provides an additional calibration
against recent binary MPE methods.  In this two-component setting, which is
well matched to mutual-contamination methods, BBE and KM2 give lower errors than
PM-MMD in most cases.  PM-MMD nevertheless recovers both vertices reliably for
moderate to large prior shifts.  Additional semi-synthetic and oracle
clean-component comparisons in the appendix further distinguish the unlabeled
multi-mixture setting considered here from clean-component binary MPE settings.

Future work includes finite-sample analysis of the validation and
condition-aware selection steps, automatic selection of informative coordinate
pairs, and improved numerical solvers for the nonconvex empirical MMD objective.

\bibliographystyle{plain}
\bibliography{allref}

\appendix
\startcontents[appendices]
\clearpage
\input appendix

\end{document}

%% file: appendix.tex
\section*{Appendix Contents}

\begingroup
\setcounter{tocdepth}{2}
\printcontents[appendices]{}{1}{}
\endgroup

\section{Notation}
\label{app:notation}

This appendix summarizes the notation used most frequently in the paper.
All densities are understood with respect to the product reference measure
introduced in Section~\ref{sec:product-form-identifiability}.

\begin{table}[H] 
\centering
\small
\renewcommand{\arraystretch}{1.18}
\caption{Main notation.}
\begin{tabular}{p{0.24\linewidth}p{0.68\linewidth}}
\hline
Notation & Meaning \\
\hline
$[q]$ & Index set $\{1,\ldots,q\}$. \\
$m,L,d$ & Number of latent components, observed mixtures, and coordinates. \\
$\mathcal X=\prod_{k=1}^d\mathcal X_k$ & Product sample space. \\
$\bm x_\sigma$, $\mathcal X_\sigma$ & Coordinates indexed by $\sigma\subset[d]$ and the corresponding product space. \\
$\mathcal S$ & Signed total-mass-one density class, $\{q\in L^1:\int q=1\}$. \\
$\mathcal P$ & Probability density class, $\{q\in\mathcal S:q\ge0\}$. \\
$\mathcal E$ & Product-form signed density class in $\mathcal S$. \\
$\mathcal S_\sigma,\mathcal P_\sigma,\mathcal E_\sigma$ & Coordinate-subset analogues on $\mathcal X_\sigma$. \\
$\bm 1_q$, $\bm e_i$ & All-ones vector in $\mathbb R^q$ and the $i$th standard basis vector. \\
$p_i$, $p_{ik}$ & The $i$th latent component density and its $k$th marginal. \\
$\bm p=(p_1,\ldots,p_m)^T$ & Vector of component densities. \\
$\mathcal M$ & Affine hull generated by $\{p_i\}_{i=1}^m$. \\
$\Pi_{\mathcal E}[q]$ & Product-marginal of $q$. \\
$U_\ell$, $\bm U$ & Observed mixture density and the vector $(U_1,\ldots,U_L)^T$. \\
$\Theta=(\theta_{\ell i})$ & $L\times m$ mixing matrix, with $\Theta\bm 1_m=\bm 1_L$. \\
$\bm r$ & Affine coefficient vector over observed mixtures; typically $\bm r^T\bm 1_L=1$. \\
$\bm\beta=\Theta^T\bm r$ & Induced affine coefficient vector over latent components. \\
$R_{st}$ & Component indices whose $(s,t)$-marginal is independent. \\
$\mathcal U$, $\mathcal U_{st}$ & Affine sets generated by the full observable mixtures and their $(s,t)$-marginals. \\
$\mathcal I$, $\mathcal J$ & Identified and unresolved component indices, $\mathcal I=\bigcup_{s<t}R_{st}$ and $\mathcal J=[m]\setminus\mathcal I$. \\
$\Theta_{\mathcal I}$, $\Theta_{\mathcal J}$ & Submatrices of $\Theta$ with columns indexed by $\mathcal I$ and $\mathcal J$. \\
$\kappa(F\mid H)$ & Mixture proportion of $H$ in $F$. \\
$q_{\ell i}$ & Residual density after removing component $p_i$ from $U_\ell$, when $\theta_{\ell i}<1$. \\
$k$, $\mathcal H_k$, $\phi$ & Kernel, its RKHS, and the kernel mean embedding. \\
$\mathrm{MMD}(q)$ & Product-marginal MMD loss of $q$. \\
$\widehat U_\ell$, $\widehat{\bm U}$ & Empirical distribution from $U_\ell$ and the collection of empirical mixtures. \\
$\Delta_{\bar r}$ & Extended simplex $\{\bm r:\bm r^T\bm 1_L=1,\ \|\bm r\|_1\le\bar r\}$. \\
$d(\bm r)$, $\widehat d(\bm r)$ & Population and empirical squared PM-MMD losses. \\
$\bm r_i$ & Target affine weight vector in the square case, $\bm r_i=(\Theta^T)^{-1}\bm e_i$. \\
$n_\ell$, $N_{1/2}$ & Sample size from $U_\ell$ and $N_{1/2}=\sum_{\ell\in[L]}n_\ell^{-1/2}$. \\
$R_{12}^{(\eta)}$ & Components whose $(1,2)$-marginal is within $K\eta$ of product form. \\
$\bar{\bm U}$, $\bar d$ & Modified exact mixtures and their squared PM-MMD loss in the near-independence analysis. \\
\hline
\end{tabular}
\end{table}

\section{Proof of Theorem in Section~\ref{sec:product-form-identifiability} }
\label{app:proof-product-form-identifiability}

\subsection{Proof of Theorem~\ref{thm:product-form-affine-identifiability}}
\begin{proof}
[Proof of Theorem~\ref{thm:product-form-affine-identifiability}]
The inclusion $\{p_i:i\in[m]\}\subseteq\mathcal{M}\cap\mathcal{E}$ is
immediate.  Conversely, take $p_{\bm\theta}\in\mathcal{M}\cap\mathcal{E}$ and set
$I=\{i:\theta_i\ne0\}$.  The affine constraint makes $I$ nonempty.  If
$|I|\ge2$, choose $s,t$ as in the subset-rank assumption.  Let
$f_1,\ldots,f_{\delta_s(I)}$ and $g_1,\ldots,g_{\delta_t(I)}$ be bases of the
spans of $\{p_{is}:i\in I\}$ and $\{p_{it}:i\in I\}$, respectively, and let
$A\in\mathbb{R}^{\delta_s(I)\times |I|}$ and $B\in \mathbb{R}^{\delta_t(I)\times |I|}$ 
be the corresponding coefficient matrices. 
Thus $\operatorname{rank}(A)=\delta_s(I)$ and
$\operatorname{rank}(B)=\delta_t(I)$.  In the linearly independent product
basis $\{f_a(x_s)g_b(x_t)\}_{a,b}$, the coefficient matrix of
\begin{align*}
 p_{\bm\theta}(x_s,x_t)
 =\sum_{i\in I}\theta_i p_{is}(x_s)p_{it}(x_t)
\end{align*}
is $A\operatorname{diag}(\theta_i:i\in I)B^T$.  Since
$p_{\bm\theta}\in\mathcal{E}$, this marginal is the product of its $s$- and
$t$-marginals.  Both have total mass one, so the coefficient matrix has rank
one; the diagonal matrix is invertible by the definition of $I$.  Following
Cohen's rank
argument~\cite{cohen2014conditions}, Sylvester's rank inequality gives
\begin{align}
1
 &=\operatorname{rank}\bigl(A\operatorname{diag}(\theta_i:i\in I)B^T\bigr)\notag\\
 &\ge \operatorname{rank}(A)+\operatorname{rank}(B)-|I|\notag\\
 &=\delta_s(I)+\delta_t(I)-|I|.
 \label{eq:product-coefficient-identity}
\end{align}
Equation~\eqref{eq:product-coefficient-identity} contradicts the subset-rank
assumption.  Hence $I=\{i\}$, and the affine
constraint gives $\bm\theta=\bm e_i$.  Consequently,
$\mathcal{M}\cap\mathcal{E}=\{p_i:i\in[m]\}$.
\end{proof}

\section{Proof of Theorems in Section~\ref{sec:mixing-vector-recovery}}
\label{app:proofs-mixing-vector-recovery}

\subsection{Proof of Theorem~\ref{thm:marginal-independence-identifiability}}
\begin{proof}
 [Proof of Theorem~\ref{thm:marginal-independence-identifiability}]
For $i \in R_{st}$, set $\bm{r} = \Theta(\Theta^T \Theta)^{-1} \bm{e}_i$.
Then
\begin{align*}
  \sum_{\ell\in[L]} r_\ell U_\ell(x_s,x_t)
  &= \sum_{j\in[m]} \sum_{\ell\in[L]} r_\ell \theta_{\ell j} p_j(x_s,x_t) \\
  &= \sum_{j\in[m]} (\bm{e}_i)_j\,p_j(x_s,x_t)
   = p_i(x_s,x_t) = p_i(x_s)\,p_i(x_t),
\end{align*}
and
\begin{align*}
  \bm{r}^T \bm{1}_L
  = \bm{e}_i^T (\Theta^T \Theta)^{-1} \Theta^T \bm{1}_L
  = \bm{e}_i^T \bm{1}_m
  = 1.
\end{align*}
Hence
\begin{align*}
  \mathcal{U}_{st} \cap \mathcal{E}_{st}
  \supset \{\, p_i(x_s)\,p_i(x_t) \mid i \in R_{st} \,\}.
\end{align*}
For the converse inclusion, suppose that
\begin{align*}
  \sum_{\ell\in[L]} r_\ell U_\ell(x_s,x_t)
  = \sum_{i\in[m]} \beta_i p_i(x_s,x_t) \in \mathcal{E}_{st},
\end{align*}
where $\beta_i = \sum_{\ell\in[L]} r_\ell \theta_{\ell i}$, $i\in[m]$.
Since $\sum_{\ell\in[L]}r_\ell=1$ and
$\Theta\bm{1}_m=\bm{1}_L$, we have $\sum_{i\in[m]}\beta_i=1$.
By assumption \eqref{eq:no-cancellation}, we must have
$\beta_i = 0$ for all $i \notin R_{st}$.
Consequently, $\sum_{j\in R_{st}}\beta_j=1$, so $R_{st}$ is nonempty for any
such candidate, and
\begin{align}
  \sum_{j\in R_{st}} \beta_j p_j(x_s,x_t)
  = \sum_{j\in R_{st}} \beta_j p_j(x_s)\,p_j(x_t)
  \in \mathcal{E}_{st}.
  \label{eq:proof-marginal-independence-mixture}
\end{align}
Applying Theorem~\ref{thm:product-form-affine-identifiability} with $d=2$ to
the product family in \eqref{eq:proof-marginal-independence-mixture}, together
with the subset-rank assumption, gives $\beta_i=1$ for some $i\in R_{st}$ and
$\beta_j=0$ for every $j\in R_{st}\setminus\{i\}$.  Together with
$\beta_j=0$ for $j\notin R_{st}$, this yields $\bm\beta=\bm e_i$.  Therefore
\begin{align*}
  \sum_{\ell\in[L]} r_\ell U_\ell(x_s,x_t)
  = p_i(x_s)\,p_i(x_t)
\end{align*}
for some $i \in R_{st}$, and we obtain
\begin{align*}
  \mathcal{U}_{st} \cap \mathcal{E}_{st}
  \subset \{\, p_i(x_s)\,p_i(x_t) \mid i \in R_{st} \,\}.
\end{align*}
Combining the two inclusions yields the claim.
\end{proof}

\subsection{A finite-table genericity statement for no-cancellation}
\label{app:generic-no-cancellation}

The following proposition makes precise the finite-dimensional genericity
intuition used in Section~\ref{subsec:marginal-independence-recovery}. 
It is not
a replacement for the continuous formulation of
Theorem~\ref{thm:marginal-independence-identifiability}; rather, it shows that the
signed cancellations excluded by condition~\eqref{eq:no-cancellation}
are lower-dimensional exceptional events in finite-grid or histogram models.

\begin{proposition}[Generic no-cancellation for finite bivariate tables]
Fix a coordinate pair $(s,t)$ and represent its bivariate marginal on an
$a\times b$ finite grid.  Let
\[
  \mathcal S_{ab}
  :=
  \left\{
    Q\in\mathbb R^{a\times b}
    \,\middle|\,
    \sum_{u=1}^a\sum_{v=1}^b Q_{uv}=1
  \right\}
\]
be the affine space of signed total-mass-one tables, and let
\[
  \mathcal E_{ab}^{\mathrm{ind}}
  :=
  \left\{
    \bm u\bm v^T
    \,\middle|\,
    \bm u\in\mathbb R^a,\ \bm v\in\mathbb R^b,\ 
    \mathbf 1^T\bm u=1,\ \mathbf 1^T\bm v=1
  \right\}
  \subset \mathcal S_{ab}
\]
be the discrete product-form model.  Then
\[
  \operatorname{codim}_{\mathcal S_{ab}}
  \mathcal E_{ab}^{\mathrm{ind}}
  =
  (a-1)(b-1).
\]

Let $R\subset[m]$ be the set of target components that are independent on
$(s,t)$.  For $i\in R$, fix product-form tables
$P_i^{\mathrm{ind}}\in\mathcal E_{ab}^{\mathrm{ind}}$.  For $j\notin R$,
let
\[
  P_j(\eta)\in\mathcal S_{ab},\qquad \eta\in\mathcal H,
\]
be a smooth family of non-target bivariate marginal tables, where
$\mathcal H$ is a smooth parameter manifold of dimension $h$.  Let
$\mathcal B$ be a smooth coefficient manifold of dimension $d_B$ contained in
\[
  \left\{
    \bm\beta\in\mathbb R^m
    \,\middle|\,
    \sum_{i=1}^m\beta_i=1,\ 
    (\beta_j)_{j\notin R}\ne \bm 0
  \right\}.
\]
Define
\[
  F:\mathcal H\times\mathcal B\to\mathcal S_{ab},
  \qquad
  F(\eta,\bm\beta)
  =
  \sum_{i\in R}\beta_i P_i^{\mathrm{ind}}
  +
  \sum_{j\notin R}\beta_j P_j(\eta).
\]
Assume that $F$ is transverse to
$\mathcal E_{ab}^{\mathrm{ind}}$.  If
\[
  (a-1)(b-1)>d_B,
\]
then the exceptional parameter set
\[
  \mathcal N
  :=
  \left\{
    \eta\in\mathcal H
    \,\middle|\,
    \exists \bm\beta\in\mathcal B
    \text{ such that }
    F(\eta,\bm\beta)\in\mathcal E_{ab}^{\mathrm{ind}}
  \right\}
\]
has Lebesgue measure zero in $\mathcal H$, in every local coordinate chart.
Consequently, for almost every component parameter $\eta$, no coefficient
vector in $\mathcal B$ produces a nontrivial signed affine cancellation into a
product-form bivariate marginal.
\end{proposition}

\begin{proof}
The affine space $\mathcal S_{ab}$ has dimension $ab-1$, since an
$a\times b$ table has $ab$ entries and one linear total-mass constraint.
The map
\[
  \Psi:
  \left\{
    \bm u\in\mathbb R^a\,\middle|\, \mathbf 1^T\bm u=1
  \right\}
  \times
  \left\{
    \bm v\in\mathbb R^b\,\middle|\, \mathbf 1^T\bm v=1
  \right\}
  \to
  \mathcal S_{ab},
  \qquad
  \Psi(\bm u,\bm v)=\bm u\bm v^T,
\]
parametrizes $\mathcal E_{ab}^{\mathrm{ind}}$.  The domain of $\Psi$ has
dimension $(a-1)+(b-1)=a+b-2$.  The parametrization is injective, because
for $Q=\bm u\bm v^T$ with $\mathbf 1^T\bm u=\mathbf 1^T\bm v=1$, the row and
column sums of $Q$ recover $\bm u$ and $\bm v$, respectively.

The differential of $\Psi$ is also injective.  Indeed, suppose that
\[
  \delta\bm u\,\bm v^T+\bm u\,\delta\bm v^T=0,
  \qquad
  \mathbf 1^T\delta\bm u=0,\quad
  \mathbf 1^T\delta\bm v=0.
\]
Multiplying the first display on the right by $\mathbf 1$ gives
\[
  \delta\bm u\,(\bm v^T\mathbf 1)
  +\bm u\,(\delta\bm v^T\mathbf 1)
  =
  \delta\bm u=0,
\]
because $\mathbf 1^T\bm v=1$ and $\mathbf 1^T\delta\bm v=0$.  Multiplying on
the left by $\mathbf 1^T$ then gives $\delta\bm v=0$.  Hence
$\mathcal E_{ab}^{\mathrm{ind}}$ is a smooth embedded submanifold of
$\mathcal S_{ab}$ with dimension $a+b-2$.  Therefore
\[
  \operatorname{codim}_{\mathcal S_{ab}}
  \mathcal E_{ab}^{\mathrm{ind}}
  =
  (ab-1)-(a+b-2)
  =
  (a-1)(b-1).
\]

Now define the failure set in the joint parameter--coefficient space by
\[
  \mathcal Z
  :=
  F^{-1}(\mathcal E_{ab}^{\mathrm{ind}})
  =
  \left\{
    (\eta,\bm\beta)\in\mathcal H\times\mathcal B
    \,\middle|\,
    F(\eta,\bm\beta)\in\mathcal E_{ab}^{\mathrm{ind}}
  \right\}.
\]
By the transversality assumption and the preimage theorem,
$\mathcal Z$ is a smooth submanifold of $\mathcal H\times\mathcal B$ of
dimension
\[
  \dim\mathcal Z
  =
  \dim\mathcal H+\dim\mathcal B
  -
  \operatorname{codim}_{\mathcal S_{ab}}
  \mathcal E_{ab}^{\mathrm{ind}}.
\]
Using the codimension computed above, this gives
\[
  \dim\mathcal Z
  =
  h+d_B-(a-1)(b-1).
\]
Let $\pi_{\mathcal H}:\mathcal H\times\mathcal B\to\mathcal H$ be the
projection.  The exceptional parameter set is
\[
  \mathcal N=\pi_{\mathcal H}(\mathcal Z).
\]
Since a smooth map cannot increase local dimension,
\[
  \dim\mathcal N
  \le
  \dim\mathcal Z
  =
  h+d_B-(a-1)(b-1).
\]
If $(a-1)(b-1)>d_B$, then
\[
  \dim\mathcal N < h=\dim\mathcal H.
\]
Thus $\mathcal N$ is locally contained in the image of a lower-dimensional
smooth manifold and has Lebesgue measure zero in every coordinate chart of
$\mathcal H$.  Hence, for almost every $\eta\in\mathcal H$, there is no
$\bm\beta\in\mathcal B$ such that
$F(\eta,\bm\beta)\in\mathcal E_{ab}^{\mathrm{ind}}$.  This proves the claim.
\end{proof}

\section{Completion after partial first-stage recovery}
\label{app:partial-completion_theorems}

Let us consider the recovery of $p_i\in\mathcal{P}$, $i\in[m]$, when
\begin{align*}
  \bigcup_{\substack{s,t\in[d]\\ s<t}} R_{st} \neq [m].
\end{align*}
Define
\begin{align*}
  \mathcal{I} := \bigcup_{\substack{s,t\in[d]\\ s<t}} R_{st},
  \qquad
  \mathcal{J} := [m]\setminus\mathcal{I}.
\end{align*}
Assume throughout this subsection that 
$\Theta\in[0,1]^{L\times m}$ has full column rank.
For each $i\in\mathcal{I}$, Theorem~\ref{thm:marginal-independence-identifiability}
identifies the component density $p_i$.
When $L>m$, however, the affine coefficient vector
$\bm r_i\in\mathbb{R}^L$ satisfying $\bm r_i^T\U = p_i$ is generally not
unique, because one may add any vector in $\ker(\Theta^T)$. 
Thus the identifiable object at this stage is the density $p_i$ itself,
whereas the representing coefficient vector should be treated only as an
auxiliary right-inverse variable.

For later use, let $E_{\mathcal{I}} := [\,\e_i\,]_{i\in\mathcal{I}}
\in \mathbb{R}^{m\times |\mathcal{I}|}$, and choose any matrix
$R_{\mathcal{I}}\in\mathbb{R}^{L\times |\mathcal{I}|}$ satisfying
\begin{align*}
  \Theta^T R_{\mathcal{I}} = E_{\mathcal{I}}.
\end{align*}
Such a matrix exists because $\Theta$ has full column rank; for example, one
may take
\begin{align*}
  R_{\mathcal{I}} = \Theta(\Theta^T\Theta)^{-1} E_{\mathcal{I}}.
\end{align*}
Let $\Theta_{\mathcal{I}}$ and $\Theta_{\mathcal{J}}$ denote the submatrices
of $\Theta$ formed by the columns indexed by $\mathcal{I}$ and
$\mathcal{J}$, respectively, and let
$\bm p_{\mathcal{I}} := (p_i)_{i\in\mathcal{I}}^T$ and
$\bm p_{\mathcal{J}} := (p_j)_{j\in\mathcal{J}}^T$ be column vectors whose entries are probability densities.
For any $R_{\mathcal{I}}\in\mathbb{R}^{L\times |\mathcal{I}|}$ satisfying
$\Theta^T R_{\mathcal{I}} = E_{\mathcal{I}}$, we have
\begin{align}
  R_{\mathcal{I}}^T \Theta_{\mathcal{I}} = I_{|\mathcal{I}|},
  \qquad
  R_{\mathcal{I}}^T \Theta_{\mathcal{J}} = O,
  \qquad
  R_{\mathcal{I}}^T \U(\x) = \bm p_{\mathcal{I}}(\x),
  \qquad
  \bm{1}_L^T R_{\mathcal{I}} = \bm{1}_{|\mathcal{I}|}^T.
\end{align}

To recover the mixing matrix $\Theta$, let us consider irreducibility conditions.
\begin{definition}
[One-sided irreducibility{\cite{blanchard2010semi}}]
For any two probability distributions $F$ and $H$, define
\begin{align}
 \kappa(F\mid H)
  := \sup\Bigl\{ \alpha\in[0,1]
      \Bigm| F = \alpha H + (1-\alpha) G,\text{ for some distribution } G
    \Bigr\}.
\end{align}
We say that $F$ is \emph{irreducible with respect to} $H$ if
$\kappa(F\mid H)=0$.
When probability densities $f,g\in\mathcal{P}$ exist, 
we also write $\kappa(f\mid h)$ for the same quantity.
\end{definition}
\begin{remark}
We refer to the condition that $F$ is irreducible with respect to $H$
as \emph{one-sided irreducibility}, in order to distinguish it from the
multi-distribution notion of \emph{joint irreducibility} introduced below.
\end{remark}
\begin{definition}
[Joint irreducibility\cite{blanchard14:_decon_mutual_contam_model,scott13:_class_asymm_label_noise}]
\label{def:joint-irreducibility}
The family $\{p_j\}_{j\in\mathcal{J}}\subset\mathcal{P}$ is said to be
\emph{jointly irreducible} if, for every $(\beta_j)_{j\in\mathcal{J}}
\in \mathbb{R}^{|\mathcal{J}|}$ satisfying
\begin{align*}
  \sum_{j\in\mathcal{J}} \beta_j p_j \in \mathcal{P},
\end{align*}
it follows that $\beta_j \ge 0$ for all $j\in\mathcal{J}$.
\end{definition}
For the recovery of $\Theta_{\mathcal I}$, we use one-sided irreducibility
through mixture proportions of the form $\kappa(U_\ell\mid p_i)$, whereas
the recovery of the remaining block $\Theta_{\mathcal J}$ relies on joint
irreducibility of $\{p_j\}_{j\in\mathcal J}$.

\begin{theorem}
\label{thm:recover-known-columns}
Suppose that for each $i\in\mathcal{I}$ and $\ell\in[L]$ with
$\theta_{\ell i}<1$, the probability density
\begin{align}
 \label{eq:def-residual-after-removal}
  q_{\ell i}(\x)
  := \frac{U_\ell(\x)-\theta_{\ell i} p_i(\x)}{1-\theta_{\ell i}}
    \in\mathcal{P}
\end{align}
is one-sided irreducible with respect to $p_i$. Then
\begin{align*}
  \theta_{\ell i} = \kappa(U_\ell \mid p_i),
  \qquad \ell\in[L],\ i\in\mathcal{I}.
\end{align*}
Consequently, the submatrix $\Theta_{\mathcal{I}}$ is identifiable.
\end{theorem}
If $\theta_{\ell i}=1$, then $U_\ell=p_i$, and therefore $\kappa(U_\ell\mid p_i)=1=\theta_{\ell i}$. 

\begin{proof}
 [Proof of Theorem~\ref{thm:recover-known-columns}]
Note that
\begin{align*}
  U_\ell = \theta_{\ell i} p_i + (1-\theta_{\ell i}) q_{\ell i},
\end{align*}
so $\kappa(U_\ell\mid p_i) \ge \theta_{\ell i}$. If
$\kappa(U_\ell\mid p_i) > \theta_{\ell i}$, then for some
$\alpha > \theta_{\ell i}$ and some $g\in\mathcal{P}$,
\begin{align*}
  U_\ell = \alpha p_i + (1-\alpha) g.
\end{align*}
Subtracting $\theta_{\ell i} p_i$ and renormalizing gives
\begin{align*}
  q_{\ell i}
  = \frac{\alpha-\theta_{\ell i}}{1-\theta_{\ell i}} p_i
    + \frac{1-\alpha}{1-\theta_{\ell i}} g,
\end{align*}
which contradicts the irreducibility of $q_{\ell i}$ with respect to $p_i$.
Hence $\kappa(U_\ell\mid p_i)=\theta_{\ell i}$.
\end{proof}

Without an additional irreducibility-type assumption, the block
$\Theta_{\mathcal{I}}$ is not identifiable in general from the observable
mixtures, even when the vectors $\bm{r}_i$, $i \in \mathcal{I}$ 
in Theorem~\ref{thm:marginal-independence-identifiability}, are known. 
An exception is the full square case $L=m$ with $|\mathcal{I}|=m$, 
in which $\Theta=\Theta_{\mathcal{I}}=R^{-1}$.
Theorem~\ref{thm:recover-known-columns} removes this ambiguity under a
one-sided irreducibility assumption by identifying each $\theta_{\ell i}$ from
the observable one-dimensional mixture, 
$U_\ell = \theta_{\ell i} p_i + (1-\theta_{\ell i}) q_{\ell i}$.

The remaining recovery problem is reduced to 
the unresolved block $\Theta_{\mathcal{J}}$. 
\begin{theorem}
\label{thm:remaining-components-convex-hull}
Assume that $\{p_j\}_{j\in\mathcal{J}}$ is jointly irreducible in the
sense of Definition~\ref{def:joint-irreducibility}. 
Define
\begin{align*}
  \mathcal{U}_{\mathcal{J}}
  := \Bigl\{
       \bm{a}^T \U(\x) \in \mathcal{P}
       \Bigm| \bm{a}\in\mathbb{R}^L,
       \ \bm{a}^T \bm{1}_L = 1,
       \ \bm{a}^T \Theta_{\mathcal{I}} = 0
     \Bigr\}.
\end{align*}
Then
\begin{align*}
  \mathcal{U}_{\mathcal{J}} = \operatorname{conv}\{ p_j \mid j\in\mathcal{J} \}.
\end{align*}
Consequently,
\begin{align*}
  \operatorname{ext}(\mathcal{U}_{\mathcal{J}})
  = \{ p_j \mid j\in\mathcal{J} \},
\end{align*}
so all remaining components $p_j$, $j\in\mathcal{J}$, are identifiable, 
where $\operatorname{ext}(\mathcal U_{\mathcal J})$ denotes the set of extreme 
points of the convex set $\mathcal U_{\mathcal J}$.
Moreover, $\Theta_{\mathcal{J}}$ is identifiable, and hence the full mixing
matrix $\Theta$ is identifiable.
\end{theorem}
\begin{proof}
[Proof of Theorem~\ref{thm:remaining-components-convex-hull}]
Let $q = \bm{a}^T \U \in \mathcal{U}_{\mathcal{J}}$ and define
$\bm{\beta} = \Theta^T \bm{a}$. Since $\U = \Theta \bm{p}$,
\begin{align*}
  q = \bm{a}^T \Theta \bm{p} = \sum_{i\in[m]} \beta_i p_i.
\end{align*}
The constraint $\bm{a}^T \Theta_{\mathcal{I}} = 0$ implies
$\beta_i = 0$ for all $i\in\mathcal{I}$, and
$\bm{a}^T \bm{1}_L = 1$ implies $\sum_{j\in\mathcal{J}} \beta_j = 1$.
By the joint irreducibility of $\{p_j\}_{j\in\mathcal{J}}$, we obtain
$\beta_j \ge 0$ for all $j\in\mathcal{J}$. Hence
$q \in \operatorname{conv}\{ p_j \mid j\in\mathcal{J} \}$.
Conversely, let
\begin{align*}
  q = \sum_{j\in\mathcal{J}} \beta_j p_j
  \in \operatorname{conv}\{ p_j \mid j\in\mathcal{J} \},
\end{align*}
where $\beta_j \ge 0$ and $\sum_{j\in\mathcal{J}} \beta_j = 1$.
Define $\bm{\beta}\in\mathbb{R}^m$ by setting $\beta_i=0$ for
$i\in\mathcal{I}$, and let
\begin{align*}
  \bm{a} := \Theta(\Theta^T\Theta)^{-1} \bm{\beta}.
\end{align*}
Then
\begin{align*}
  \bm{a}^T \Theta = \bm{\beta}^T,
  \qquad
  \bm{a}^T \Theta_{\mathcal{I}} = 0,
  \qquad
  \bm{a}^T \bm{1}_L
  = \bm{\beta}^T (\Theta^T\Theta)^{-1}\Theta^T \bm{1}_L
  = \bm{\beta}^T \bm{1}_m = 1,
\end{align*}
since $\Theta \bm{1}_m = \bm{1}_L$ implies
$(\Theta^T\Theta)^{-1}\Theta^T \bm{1}_L = \bm{1}_m$.
Therefore $q = \bm{a}^T \U \in \mathcal{U}_{\mathcal{J}}$.
This proves
$\mathcal{U}_{\mathcal{J}} = \operatorname{conv}\{ p_j \mid j\in\mathcal{J} \}$.
Since joint irreducibility implies affine independence, 
the convex hull of $\{p_j\}_{j\in\mathcal{J}}$ 
has exactly these points as extreme points. 
Hence, the components $p_j$, $j\in\mathcal{J}$, are identifiable.

It remains to identify $\Theta_{\mathcal{J}}$.
For each $\ell\in[L]$, the deflated observable density
\begin{align*}
  \widetilde U_\ell
  := U_\ell - \sum_{i\in\mathcal{I}} \theta_{\ell i} p_i
  = \sum_{j\in\mathcal{J}} \theta_{\ell j} p_j
\end{align*}
is identifiable, because both $U_\ell$ and the block
$\Theta_{\mathcal{I}}\bm p_{\mathcal{I}}$ are identifiable.
Let
\begin{align*}
  c_\ell := \sum_{j\in\mathcal{J}} \theta_{\ell j}
  = 1 - \sum_{i\in\mathcal{I}} \theta_{\ell i}.
\end{align*}
Since $\Theta_{\mathcal{I}}$ is already known, the scalar $c_\ell$ is known as
well.
If $c_\ell=0$, then necessarily $\theta_{\ell j}=0$ for all
$j\in\mathcal{J}$.
Assume now that $c_\ell>0$, and write
\begin{align*}
  \frac{\widetilde U_\ell}{c_\ell}
  = \sum_{j\in\mathcal{J}} \lambda_{\ell j} p_j,
  \qquad
  \lambda_{\ell j} := \frac{\theta_{\ell j}}{c_\ell},
  \qquad
  \sum_{j\in\mathcal{J}} \lambda_{\ell j} = 1.
\end{align*}
Thus $\widetilde U_\ell/c_\ell$ is an affine combination of the identifiable
family $\{p_j\}_{j\in\mathcal{J}}$.
Because this family is affinely independent, that affine representation is
unique.
Hence the coefficients $\lambda_{\ell j}$ are uniquely determined, and then so
are
\begin{align*}
  \theta_{\ell j} = c_\ell\lambda_{\ell j}, \qquad j\in\mathcal{J}.
\end{align*}
Therefore every row of $\Theta_{\mathcal{J}}$ is identifiable.
Since $\Theta_{\mathcal{I}}$ was already identified by
Theorem~\ref{thm:recover-known-columns}, the full mixing matrix 
$\Theta = (\Theta_{\mathcal{I}}\ \Theta_{\mathcal{J}})$ is identifiable.
\end{proof}

Additional assumptions for completing the mixing matrix $\Theta$ after component recovery
are summarized below.
\begin{assumption}
Let $\mathcal{J}=[m]\setminus\mathcal{I}$. We impose the following conditions:
\begin{itemize}
 \item[(i)] For each $i\in\mathcal{I}$, the probability density $p_i\in\mathcal{P}$ contains at least one pair of independent variables.
 \item[(ii)] For each $i\in\mathcal{I}$ and $\ell\in[L]$, 
 the probability density $q_{\ell i}\in\mathcal{P}$ in \eqref{eq:def-residual-after-removal}
      is one-sided irreducible with respect to $p_i$.
 \item[(iii)] The collection $\{p_j\}_{j\in\mathcal{J}}$ is jointly irreducible. 
\end{itemize}
\end{assumption}

\begin{example}[Assumption (ii) alone does not identify the binary decomposition]
Let $L=m=2$ and $\mathcal I=\{1\}$, so that $\mathcal J=\{2\}$ and
assumption (iii) is automatic. Write
\[
U_\ell(x)=\theta_{\ell1}p_1(x)+(1-\theta_{\ell1})p_2(x),
\qquad \ell=1,2,
\]
where $0<\theta_{11},\theta_{21}<1$ and $\theta_{11}\neq\theta_{21}$.
Then $q_{11}=q_{21}=p_2$, so assumption (ii) reduces to the one-sided
irreducibility of $p_2$ with respect to $p_1$.
This condition alone does not identify the binary decomposition. Indeed, for
any $\lambda\in(\max\{\theta_{11},\theta_{21}\},1]$, let
\begin{align*}
p_1^{(\lambda)}:=\lambda p_1+(1-\lambda)p_2. 
\end{align*}
Then
\begin{align*}
U_\ell(x)
=
\frac{\theta_{\ell1}}{\lambda}p_1^{(\lambda)}(x)
+
\left(1-\frac{\theta_{\ell1}}{\lambda}\right)p_2(x),
\qquad \ell=1,2, 
\end{align*}
and this alternative decomposition still satisfies assumption (ii). Hence
assumption (ii) alone does not identify the binary decomposition.
The point is that assumption~(ii) should be viewed only as a condition for recovering the
mixing coefficients associated with an already identified component. 
In the present framework, the component $p_1$ is first identified from assumption~(i) via 
Theorem~\ref{thm:marginal-independence-identifiability} and assumption~(ii) is then used to identify the 
corresponding column of $\Theta_{\mathcal{I}}$ 
through the relation $\theta_{\ell 1}=\kappa(U_\ell \mid p_1)$.
By contrast, suppose that $\{p_1,p_2\}$ is 
mutually irreducible~\cite{scott13:_class_asymm_label_noise}; that is, 
each component is irreducible with respect to the other. Then the binary
decomposition is identifiable up to permutation.
This is a different sufficient condition 
from the marginal-independence route used here.
\end{example}

\section{Completion Algorithms} 
\label{app:partial-completion}

\subsection{Completion when exactly one component remains}
\label{app:one-missing-component}

Suppose that, after representative selection and reindexing, the first-stage
outputs are
\[
  \{(\widehat{\bm r}_i,\widehat p_i)\}_{i=1}^{m-1},
  \qquad
  \widehat p_i:=\widehat{\bm r}_i^T\widehat{\bm U},
\]
so that exactly one component remains unresolved.  In this case, the first
\(m-1\) columns of the mixing matrix can be estimated by binary mixture
proportion estimation against the recovered components, and the last column is
then obtained from the row-sum constraint.
At the population level, this uses the mixture-proportion identity in
Theorem~\ref{thm:recover-known-columns} for the recovered components and the
row-sum constraint for the remaining column.

\begin{algorithm}[!t]
\caption{Completion when exactly one component remains}
\label{alg:completion-one-missing}
\begin{algorithmic}[1]
\REQUIRE Reindexed first-stage outputs
\(\{(\widehat{\bm r}_i,\widehat p_i)\}_{i=1}^{m-1}\),
empirical mixtures \(\{\widehat U_\ell\}_{\ell=1}^m\),
an estimator \(\widehat\kappa(f\mid h)\) of the mixture proportion
\(\kappa(f\mid h)\), and a threshold \(\eta\ge 0\).
\ENSURE An estimate \(\widehat\Theta\in[0,1]^{m\times m}\) and component
estimates \(\{\widehat p_i\}_{i=1}^m\).

\STATE Estimate the first \(m-1\) columns by
\[
  \widetilde\theta_{\ell i}
  :=
  \widehat\kappa(\widehat U_\ell\mid \widehat p_i),
  \qquad
  \ell\in[m],\ i\in[m-1].
\]

\STATE Recover the last column by the row-sum constraint:
\[
  \widetilde\theta_{\ell m}
  :=
  1-\sum_{i=1}^{m-1}\widetilde\theta_{\ell i},
  \qquad
  \ell\in[m].
\]

\STATE If needed, project each row
\[
  \widetilde{\bm\theta}_\ell
  =
  (\widetilde\theta_{\ell1},\dots,\widetilde\theta_{\ell m})^T
\]
onto the simplex \(\Delta_1^m\), and denote the resulting matrix by
\(\widehat\Theta=(\widehat\theta_{\ell i})\).

\STATE Let
\[
  I_\eta:=\{\ell\in[m]\mid \widehat\theta_{\ell m}>\eta\}.
\]
Estimate the remaining component by
\[
  \widehat p_m
  =
  \frac{
    \sum_{\ell\in I_\eta}
    \left(
      \widehat U_\ell
      -
      \sum_{i=1}^{m-1}\widehat\theta_{\ell i}\widehat p_i
    \right)
  }{
    \sum_{\ell\in I_\eta}\widehat\theta_{\ell m}
  }.
\]

\STATE \textbf{return}
\(\widehat\Theta\) and
\(\{\widehat p_1,\dots,\widehat p_{m-1},\widehat p_m\}\).
\end{algorithmic}
\end{algorithm}
The threshold \(\eta\) is used only as a numerical stabilization device:
if \(\widehat\theta_{\ell m}\) is very small, then the deflation formula for
\(\widehat p_m\) can be unstable.  At the population level, one may simply
take \(\eta=0\).

\subsection{Completion when exactly two components remain}
\label{app:two-missing-components}

We record a simple extension of Section~\ref{subsec:pmmd-algorithm} to the
case where the first stage recovers $m-2$ components.
After reindexing, suppose that the first-stage outputs are
\[
  \{(\widehat{\bm r}_i,\widehat p_i)\}_{i=1}^{m-2},
\]
so that the unresolved indices are $\mathcal J=\{m-1,m\}$.
At the population level, write
\[
  c_\ell
  :=
  1-\sum_{i=1}^{m-2}\theta_{\ell i}
  = \theta_{\ell,m-1}+\theta_{\ell m},
  \qquad \ell\in[m],
\]
and, whenever $c_\ell>0$, define the normalized residual mixture
\[
  V_\ell
  :=
  \frac{U_\ell-\sum_{i=1}^{m-2}\theta_{\ell i}p_i}{c_\ell}.
\]
Then
\[
  V_\ell
  =
  \lambda_\ell p_{m-1}+(1-\lambda_\ell)p_m,
  \qquad
  \lambda_\ell:=\frac{\theta_{\ell,m-1}}{c_\ell}\in[0,1].
\]
Hence each $V_\ell$ lies on the line segment
$\operatorname{conv}\{p_{m-1},p_m\}=\mathcal U_{\mathcal J}$ from
Theorem~\ref{thm:remaining-components-convex-hull}.
The next lemma shows that any two distinct points on this line segment determine
its two endpoints through mixture proportions.

\begin{lemma}
\label{lem:two-point-residue}
Let $p,q\in\mathcal P$ be jointly irreducible.
For $0\le \mu<\lambda\le 1$, define
\begin{align*}
  f:=\lambda p+(1-\lambda)q,
  \qquad
  h:=\mu p+(1-\mu)q. 
\end{align*}
Then
\begin{align*}
  \kappa(f\mid h)=\frac{1-\lambda}{1-\mu},
  \qquad
  p=\frac{f-\kappa(f\mid h)h}{1-\kappa(f\mid h)},
\end{align*}
 and
\begin{align*}
  \kappa(h\mid f)=\frac{\mu}{\lambda},
  \qquad
  q=\frac{h-\kappa(h\mid f)f}{1-\kappa(h\mid f)}. 
\end{align*}
\end{lemma}

\begin{proof}
We have the decomposition
\begin{align*}
 f
  =
  \frac{1-\lambda}{1-\mu}h
  +
 \left(1-\frac{1-\lambda}{1-\mu}\right)p,
\end{align*}
so $\kappa(f\mid h)\ge (1-\lambda)/(1-\mu)$.
Conversely, if
\begin{align*}
  f=\alpha h+(1-\alpha)g
 \qquad\text{for some }\alpha\in[0,1],\ g\in\mathcal P, 
\end{align*}
then
\begin{align*}
 g  =  \frac{\lambda-\alpha\mu}{1-\alpha}p  +  \frac{(1-\lambda)-\alpha(1-\mu)}{1-\alpha}q.
\end{align*}
The coefficients on the right-hand side sum to one.
Since $g\in\mathcal P$ and $\{p,q\}$ is jointly irreducible, both coefficients
must be nonnegative.
In particular,
\begin{align*}
  (1-\lambda)-\alpha(1-\mu)\ge 0, 
\end{align*}
which implies $\alpha\le (1-\lambda)/(1-\mu)$.
Hence
\begin{align*}
  \kappa(f\mid h)=\frac{1-\lambda}{1-\mu}. 
\end{align*}
Substituting this value into the above display gives
\begin{align*}
 f-\kappa(f\mid h)h  = \left(1-\kappa(f\mid h)\right)p, 
\end{align*}
and therefore
\begin{align*}
  p=\frac{f-\kappa(f\mid h)h}{1-\kappa(f\mid h)}. 
\end{align*}
The formulas for $\kappa(h\mid f)$ and $q$ are obtained by the same argument,
now writing
\begin{align*}
 h
  =
  \frac{\mu}{\lambda}f
  +
  \left(1-\frac{\mu}{\lambda}\right)q. 
\end{align*}
\end{proof}

The previous lemma reduces the unresolved two-component problem to a
one-dimensional residual demixing step.

\begin{proposition}
Assume $L=m$, that $\Theta\in[0,1]^{m\times m}$ is invertible, and that after
reindexing the first-stage outputs correspond to the recovered block
$\mathcal I=[m-2]$.
Let $\mathcal J=\{m-1,m\}$.
Suppose that the assumptions of
Theorem~\ref{thm:recover-known-columns} hold for the recovered block and that
$\{p_{m-1},p_m\}$ is jointly irreducible.
Then the following statements hold.
\begin{enumerate}
  \item[(i)] For every $\ell\in[m]$ with $c_\ell>0$,
  \[
    V_\ell\in\operatorname{conv}\{p_{m-1},p_m\}.
  \]
  \item[(ii)] There exist $a,b\in[m]$ such that $c_a>0$, $c_b>0$, and
  $V_a\neq V_b$.
  \item[(iii)] For any such pair, after relabeling if necessary so that
  $V_a=\lambda_a p_{m-1}+(1-\lambda_a)p_m$ and
  $V_b=\lambda_b p_{m-1}+(1-\lambda_b)p_m$ with $\lambda_a>\lambda_b$,
  \[
    p_{m-1}
    =
    \frac{V_a-\kappa(V_a\mid V_b)V_b}{1-\kappa(V_a\mid V_b)},
    \qquad
    p_m
    =
    \frac{V_b-\kappa(V_b\mid V_a)V_a}{1-\kappa(V_b\mid V_a)}.
  \]
  \item[(iv)] For each $\ell\in[m]$ with $c_\ell>0$,
  \[
    \theta_{\ell,m-1}
    =
    c_\ell\,\kappa(V_\ell\mid p_{m-1}),
    \qquad
    \theta_{\ell,m}
    =
    c_\ell-\theta_{\ell,m-1}.
  \]
  Consequently, the two remaining components and the last two columns of
  $\Theta$ are identifiable.
\end{enumerate}
\end{proposition}

\begin{proof}
Part (i) follows directly from the definition of $V_\ell$.
For part (ii), write
\[
  V_\ell
  =
  \lambda_\ell p_{m-1}+(1-\lambda_\ell)p_m,
  \qquad
  \lambda_\ell=\frac{\theta_{\ell,m-1}}{c_\ell}.
\]
If all $V_\ell$ with $c_\ell>0$ were equal, then all $\lambda_\ell$ would be the
same value $\lambda\in[0,1]$.
Hence
\[
  \theta_{\ell,m-1}=\lambda c_\ell,
  \qquad
  \theta_{\ell,m}=(1-\lambda)c_\ell
\]
for all $\ell$, so the two columns of $\Theta_{\mathcal J}$ would be
proportional.
This contradicts the invertibility of $\Theta$.
Thus there exist $a,b$ with $V_a\neq V_b$.
Part (iii) is exactly Lemma~\ref{lem:two-point-residue} applied to the pair
$(V_a,V_b)$.
For part (iv), if $c_\ell>0$ then
\[
  V_\ell
  =
  \lambda_\ell p_{m-1}+(1-\lambda_\ell)p_m.
\]
For two distributions, joint irreducibility implies one-sided irreducibility in
both directions.
Therefore,
Theorem~\ref{thm:recover-known-columns} applied to the two-component mixture
$V_\ell$ yields
\[
  \lambda_\ell=\kappa(V_\ell\mid p_{m-1}).
\]
Multiplying by $c_\ell$ gives the claimed formula for
$\theta_{\ell,m-1}$, and the identity for $\theta_{\ell,m}$ follows from
$\theta_{\ell,m-1}+\theta_{\ell,m}=c_\ell$.
If $c_\ell=0$, then automatically
$\theta_{\ell,m-1}=\theta_{\ell,m}=0$.
This proves the claim.
\end{proof}

The proposition suggests the following empirical completion rule.
Any two distinct normalized residual mixtures would suffice at the population
level. Empirically, we first form the normalized residuals $\widehat V_\ell$ for
all rows with positive residual mass, and then choose the farthest pair only
among rows whose estimated residual mass exceeds the threshold $\eta$.
This avoids amplifying noise when the normalization denominator is too small.

\begin{algorithm}[!t]
\caption{Completion when exactly two components remain}
\label{alg:completion-two-missing}
\begin{algorithmic}[1]
\REQUIRE Reindexed first-stage outputs
$\{(\widehat{\bm r}_i,\widehat p_i)\}_{i=1}^{m-2}$,
empirical mixtures $\{\widehat U_\ell\}_{\ell=1}^m$,
an estimator $\widehat\kappa(f\mid h)$ of the mixture proportion
$\kappa(f\mid h)$, and a threshold $\eta\ge 0$.
\ENSURE An estimate $\widehat\Theta\in[0,1]^{m\times m}$ and component estimates
$\{\widehat p_i\}_{i=1}^m$.

\STATE Estimate the first $m-2$ columns by
$\widetilde\theta_{\ell i}:=\widehat\kappa(\widehat U_\ell\mid \widehat p_i)$,
$\ell\in[m]$, $i\in[m-2]$.

\STATE Define the residual masses
$\widetilde c_\ell:=1-\sum_{i=1}^{m-2}\widetilde\theta_{\ell i}$,
and whenever $\widetilde c_\ell>0$ set
\[
  \widehat V_\ell
  :=
  \frac{\widehat U_\ell-\sum_{i=1}^{m-2}\widetilde\theta_{\ell i}\widehat p_i}
       {\widetilde c_\ell}.
\]

\STATE Choose
\[
  (a,b)
  \in
  \arg\max_{\substack{\ell,\ell'\in[m]\\
                       \widetilde c_\ell>\eta,\ \widetilde c_{\ell'}>\eta}}
  \bigl\|\phi(\widehat V_\ell)-\phi(\widehat V_{\ell'})\bigr\|_{\mathcal H_k}.
\]

\STATE Set
\[
  \widehat\alpha:=\widehat\kappa(\widehat V_a\mid \widehat V_b),
  \qquad
  \widehat\beta:=\widehat\kappa(\widehat V_b\mid \widehat V_a),
\]
and recover the remaining two components by
\[
  \widehat p_{m-1}
  :=
  \frac{\widehat V_a-\widehat\alpha\,\widehat V_b}{1-\widehat\alpha},
  \qquad
  \widehat p_m
  :=
  \frac{\widehat V_b-\widehat\beta\,\widehat V_a}{1-\widehat\beta}.
\]

\STATE For each $\ell\in[m]$ with $\widetilde c_\ell>0$, define
\[
  \widetilde\theta_{\ell,m-1}
  :=
  \widetilde c_\ell\,\widehat\kappa(\widehat V_\ell\mid \widehat p_{m-1}),
  \qquad
  \widetilde\theta_{\ell,m}
  :=
  \widetilde c_\ell-\widetilde\theta_{\ell,m-1}.
\]
If $\widetilde c_\ell\le 0$, set
$\widetilde\theta_{\ell,m-1}=\widetilde\theta_{\ell,m}:=0$.

\STATE If needed, project each row
$\widetilde{\bm\theta}_\ell
=(\widetilde\theta_{\ell1},\dots,\widetilde\theta_{\ell m})^T$
onto the simplex $\Delta_1^m$, and denote the resulting matrix by
$\widehat\Theta=(\widehat\theta_{\ell i})$.

\STATE \textbf{return}
$\widehat\Theta$ and
$\{\widehat p_1,\dots,\widehat p_{m-2},\widehat p_{m-1},\widehat p_m\}$.
\end{algorithmic}
\end{algorithm}

As in Algorithm~\ref{alg:completion-one-missing}, we keep the presentation simple
and do not add a separate failure case.
The threshold $\eta$ is used only in Step~3 as a numerical stabilization
device: if $\widetilde c_\ell$ is very small, then the normalized residual
$\widehat V_\ell$ can be unstable because it divides by $\widetilde c_\ell$.
At the population level, one may simply take $\eta=0$.


\section{Proof of Theorems in Section~\ref{sec:mmd-estimation}}
\label{app:mmd-estimation-proofs}

\subsection{Proof of Theorem~\ref{thm:pmmd-uniform-bound}}
\begin{proof}
[Proof of Theorem~\ref{thm:pmmd-uniform-bound}]
Write each observation from the bivariate marginal of $U_\ell$ as
$\bm x_{\ell i}=(X_{\ell i},Y_{\ell i})\in\mathcal X_1\times\mathcal X_2$, and set
\[
  \psi(x,y):=k((x,y),\cdot)\in\mathcal H_k,
  \qquad
  \|\psi(x,y)\|_{\mathcal H_k}\le K.
\]
Let
\begin{align*}
  \mu_\ell
  &:=\mathbb E\psi(X_\ell,Y_\ell),
  &
  \widehat\mu_\ell
  &:=\frac1{n_\ell}\sum_{i=1}^{n_\ell}\psi(X_{\ell i},Y_{\ell i}),\\
  \nu_{\ell\ell'}
  &:=\mathbb E\psi(X_\ell,\widetilde Y_{\ell'}),
  &
  \widehat\nu_{\ell\ell'}
  &:=\frac1{n_\ell n_{\ell'}}
    \sum_{i=1}^{n_\ell}\sum_{j=1}^{n_{\ell'}}
    \psi(X_{\ell i},Y_{\ell' j}),
\end{align*}
where $X_\ell$ has the first marginal of $U_\ell$, $\widetilde Y_{\ell'}$ has the
second marginal of $U_{\ell'}$, and $X_\ell$ and $\widetilde Y_{\ell'}$ are independent.
For $\ell=\ell'$, this convention means that $\nu_{\ell\ell}$ is the embedding of
the product of the two one-dimensional marginals, not the embedding of the joint
law of $(X_\ell,Y_\ell)$.

For any $\bm r\in\Delta_{\bar r}$,
\begin{align*}
 \mathrm{MMD}(\bm r^T\U)
 =\Bigl\|\sum_{\ell}r_\ell\mu_\ell
          -\sum_{\ell,\ell'}r_\ell r_{\ell'}\nu_{\ell\ell'}\Bigr\|_{\mathcal H_k},\quad
 \mathrm{MMD}(\bm r^T\widehat\U)
 =\Bigl\|\sum_{\ell}r_\ell\widehat\mu_\ell
          -\sum_{\ell,\ell'}r_\ell r_{\ell'}\widehat\nu_{\ell\ell'}\Bigr\|_{\mathcal H_k}.
\end{align*}
Hence, by the reverse triangle inequality,
\begin{align}
\label{eq:uniform-bound-basic-decomp}
 \left|\mathrm{MMD}(\bm r^T\U)-\mathrm{MMD}(\bm r^T\widehat\U)\right|
 &\le
 \sum_\ell |r_\ell|\,\|\widehat\mu_\ell-\mu_\ell\|_{\mathcal H_k}
 +
 \sum_{\ell,\ell'} |r_\ell r_{\ell'}|\,
 \|\widehat\nu_{\ell\ell'}-\nu_{\ell\ell'}\|_{\mathcal H_k}.
\end{align}
We now control the two types of empirical embeddings.

First, for the ordinary empirical embedding, changing one observation in
$\widehat\mu_\ell$ changes $\|\widehat\mu_\ell-\mu_\ell\|_{\mathcal H_k}$ by at most
$2K/n_\ell$. Moreover,
\begin{align*}
 \mathbb E\|\widehat\mu_\ell-\mu_\ell\|_{\mathcal H_k}
 \le
 \left(\mathbb E\|\widehat\mu_\ell-\mu_\ell\|_{\mathcal H_k}^2\right)^{1/2}
 \le \frac{K}{\sqrt{n_\ell}}.
\end{align*}
McDiarmid's inequality therefore gives, for any $a>0$,
\begin{align}
\label{eq:mu-embedding-bound}
 \|\widehat\mu_\ell-\mu_\ell\|_{\mathcal H_k}
 \le
 \frac{2K(1+\sqrt a)}{\sqrt{n_\ell}}
\end{align}
with probability at least $1-e^{-a}$.

Second, consider the product empirical embedding. If $\ell\ne\ell'$, then
$\mathbb E\widehat\nu_{\ell\ell'}=\nu_{\ell\ell'}$. Replacing one observation from
sample $\ell$ changes $\widehat\nu_{\ell\ell'}$ by at most $2K/n_\ell$, and replacing
one observation from sample $\ell'$ changes it by at most $2K/n_{\ell'}$.
The Hilbert-space Efron--Stein 
inequality~\cite{pinelis1994optimum,boucheron2013concentration} yields
\begin{align*}
 \mathbb E\|\widehat\nu_{\ell\ell'}-\nu_{\ell\ell'}\|_{\mathcal H_k}
 \le
 \left\{2K^2\left(\frac1{n_\ell}+\frac1{n_{\ell'}}\right)\right\}^{1/2}
 \le
 2K\left(\frac1{\sqrt{n_\ell}}+\frac1{\sqrt{n_{\ell'}}}\right).
\end{align*}
Together with McDiarmid's inequality, this implies
\begin{align}
\label{eq:nu-offdiag-bound}
 \|\widehat\nu_{\ell\ell'}-\nu_{\ell\ell'}\|_{\mathcal H_k}
 \le
 2K(1+\sqrt a)
 \left(\frac1{\sqrt{n_\ell}}+\frac1{\sqrt{n_{\ell'}}}\right)
\end{align}
with probability at least $1-e^{-a}$.

If $\ell=\ell'$, the all-pairs estimator contains diagonal terms. In this case
\begin{align*}
 \mathbb E\widehat\nu_{\ell\ell}
 =
 \left(1-\frac1{n_\ell}\right)\nu_{\ell\ell}
  +\frac1{n_\ell}\mu_\ell,
\end{align*}
and hence
\begin{align*}
 \|\mathbb E\widehat\nu_{\ell\ell}-\nu_{\ell\ell}\|_{\mathcal H_k}
 \le \frac{2K}{n_\ell}.
\end{align*}
Replacing one observation changes $\widehat\nu_{\ell\ell}$ by at most $4K/n_\ell$.
Thus the Hilbert-space Efron--Stein inequality gives
\begin{align*}
 \mathbb E\|\widehat\nu_{\ell\ell}-\mathbb E\widehat\nu_{\ell\ell}\|_{\mathcal H_k}
 \le \frac{2\sqrt2K}{\sqrt{n_\ell}},
\end{align*}
and McDiarmid's inequality gives
\begin{align*}
 \|\widehat\nu_{\ell\ell}-\nu_{\ell\ell}\|_{\mathcal H_k}
 \le
 \frac{8K(1+\sqrt a)}{\sqrt{n_\ell}}
\end{align*}
with probability at least $1-e^{-a}$. Combining this diagonal bound with
\eqref{eq:nu-offdiag-bound}, we have the following uniform bound for all ordered
pairs $\ell,\ell'$:
\begin{align}
\label{eq:nu-all-bound}
 \|\widehat\nu_{\ell\ell'}-\nu_{\ell\ell'}\|_{\mathcal H_k}
 \le
 4K(1+\sqrt a)
 \left(\frac1{\sqrt{n_\ell}}+\frac1{\sqrt{n_{\ell'}}}\right).
\end{align}

Set $a=\log(2L^2/\delta)$.
By the union bound, \eqref{eq:mu-embedding-bound} holds for all $\ell\in[L]$ and
\eqref{eq:nu-all-bound} holds for all ordered pairs $(\ell,\ell')\in[L]^2$ with
probability at least $1-\delta$.
On this event, \eqref{eq:uniform-bound-basic-decomp} implies, uniformly over
$\bm r\in\Delta_{\bar r}$,
\begin{align*}
 \left|\mathrm{MMD}(\bm r^T\U)-\mathrm{MMD}(\bm r^T\widehat\U)\right|
 &\le
 2K(1+\sqrt a)\sum_\ell |r_\ell|n_\ell^{-1/2} \\
 &\quad +
 4K(1+\sqrt a)
 \sum_{\ell,\ell'} |r_\ell r_{\ell'}|
 \left(n_\ell^{-1/2}+n_{\ell'}^{-1/2}\right)\\
 &=
 2K(1+\sqrt a)\sum_\ell |r_\ell|n_\ell^{-1/2}
 +8K(1+\sqrt a)\|\bm r\|_1
       \sum_\ell |r_\ell|n_\ell^{-1/2}.
\end{align*}
Since $\|\bm r\|_1\le\bar r$, $\bar r\ge1$, and
$\sum_\ell |r_\ell|n_\ell^{-1/2}\le \bar r N_{1/2}$, the right-hand side is bounded by
\begin{align*}
 10K\bar r^2(1+\sqrt a)N_{1/2}.
\end{align*}
This proves the theorem.
\end{proof}

\subsection{Proof of Theorem~\ref{thm:population-level-set}}
\begin{proof}
[Proof of Theorem~\ref{thm:population-level-set}]
 The assumption on the Hessian matrix ensures that
 there exist $\lambda_i>0$ and a neighborhood $N_i$ of $\bm{r}_i$ such that
\begin{align*}
   d(\bm{r}) \geq \lambda_i \|\bm{r}-\bm{r}_i\|^2  \quad\text{for all } \bm{r}\in N_i\cap\Delta_{\bar{r}}. 
\end{align*}
 Hence, for $c\in(0,c_0)$ with sufficiently small $c_0>0$, we have 
 \begin{align*}
  \{\, \bm{r}\in\Delta_{\bar{r}} \mid d(\bm{r})\leq c \,\}\cap N_i
  \subset 
  B_i := \bigl\{\bm{r}\in\Delta_{\bar{r}} \mid \|\bm{r}-\bm{r}_i\| < \sqrt{c_0/\lambda_i}\bigr\}
  \subset N_i\cap \Delta_{\bar{r}}. 
 \end{align*}
 The sets $B_i$, $i\in R_{12}$, are pairwise disjoint. 
By the zero characterization stated before Theorem~\ref{thm:population-level-set}, 
there are no zeros of $d$ on 
$\Delta_{\bar r}\setminus\{\bm r_i\mid i\in R_{12}\}$. 
Hence, by the compactness of $\Delta_{\bar r}$, 
\[
  \zeta_0:=\min\bigl\{ d(\bm{r}) \mid \bm{r}\in
  \Delta_{\bar{r}}\setminus\cup_{i\in R_{12}} B_i\bigr\}
\]
is strictly positive.
 Then, for $c\in(0,\zeta_0)$ we have
 \[
   \{\, \bm{r}\in\Delta_{\bar{r}} \mid d(\bm{r})\leq c \,\}
   \subset \bigcup_{i\in R_{12}} B_i. 
 \]
Each set $\{\, \bm{r}\in\Delta_{\bar{r}} \mid d(\bm{r})\leq c \,\}\cap B_i$ is non-empty and convex, and hence
 connected, for sufficiently small~$c$.  
Hence, $\{\, \bm{r}\in\Delta_{\bar{r}} \mid d(\bm{r})\leq c \,\}$ consists of $|R_{12}|$ disjoint subsets. 
\end{proof}

\subsection{Proof of Theorem~\ref{thm:empirical-level-set}}
\begin{proof}
[Proof of Theorem~\ref{thm:empirical-level-set}]
 The Hessian matrix of $d(\bm{r})$ at ${\bm r}_i, i\in{R_{12}}$ is positive definite 
 on the tangent space of $\Delta_{\bar{r}}$. 
 Hence, at each zero $\bm{r}_i$
 we can choose mutually disjoint open balls $B_i$ centered at $\bm{r}_i$ 
 such that $d$ is strongly convex on $B_i \cap \Delta_{\bar{r}}$. 
 By continuity of the Hessian for quartic polynomial, 
 there exists $\epsilon_0' > 0$ such that 
 if a quartic polynomial $d_1$ satisfies $\|d - d_1\|_\infty \le \epsilon_0'$, 
 then $d_1$ is also strongly convex on $B_i \cap \Delta_{\bar{r}}$. 
 Here $\epsilon_0'$ is determined only by $d$.

 By the zero characterization stated before Theorem~\ref{thm:population-level-set}, 
there are no zeros of $d$ on 
$\Delta_{\bar r}\setminus\{\bm r_i\mid i\in R_{12}\}$. 
Define
\begin{align*}
  \zeta_0 := \min_{\bm{r}\in\Delta_{\bar{r}} \setminus \cup_{i\in R_{12}} B_i} d(\bm{r}) > 0, 
  \qquad 
  c_0 := \tfrac{\zeta_0}{2}, 
  \qquad 
  \epsilon_0 := \tfrac{\zeta_0}{8} \wedge \epsilon_0'. 
\end{align*}
 Note that $c_0$ and $\epsilon_0$ are determined from the function $d$. 
 Assume $\|d - \widehat{d}\|_\infty \le \epsilon \le \epsilon_0$ and $2\epsilon \le c \le c_0$. 
 Then for any $\bm{r} \in \Delta_{\bar{r}} \setminus \cup_{i\in R_{12}} B_i$,
 \begin{align*}
  \widehat{d}(\bm{r}) \ge d(\bm{r}) - \epsilon 
  \ge \zeta_0 - \epsilon 
  \ge \tfrac{7}{8}\zeta_0 = \tfrac{7}{4}c_0 > c, 
 \end{align*}
 hence $\{\bm{r}\in\Delta_{\bar{r}} \mid \widehat{d}(\bm{r}) \le c\} \subset \cup_{i\in R_{12}} B_i$. 

 Moreover, since $d(\bm{r}_i) = 0$, we have $\widehat{d}(\bm{r}_i) \le \epsilon < c$ 
 for every $i$, so each set
 \[
   A_i := \{\bm{r}\in\Delta_{\bar{r}} \mid \widehat{d}(\bm{r}) \le c\} \cap B_i 
 \]
 is nonempty. In particular, $\bm{r}_i \in A_i$. 
 Because the $B_i$ are disjoint, the $A_i$ are also pairwise disjoint, and
 \[
   \{\bm{r}\in\Delta_{\bar{r}} \mid \widehat{d}(\bm{r}) \le c\}
   = \bigcup_{i \in R_{12}} A_i .
 \]
 Since $\widehat{d}$ is strongly convex on each $A_i \subset B_i$, 
 the set $\{\bm{r}\in\Delta_{\bar{r}} \mid \widehat{d}(\bm{r}) \le c\} \cap B_i$ is convex and therefore connected. 
 Hence $\{\bm{r}\in\Delta_{\bar{r}} \mid \widehat{d}(\bm{r}) \le c\}$ consists of $|R_{12}|$ disjoint subsets, 
 and each subset contains exactly one of the points $\bm{r}_i$, $i\in R_{12}$. 
\end{proof}

\subsection{Proof of Theorem~\ref{thm:near-independence-stability}}
First, let us show the proof of Lemma~\ref{lem:near-independent-replacement}. 

\begin{proof}[Proof of Lemma~\ref{lem:near-independent-replacement}]
Since 
$\Pi_{\mathcal{E}}[\bm{r}^T\U]=\Pi_{\mathcal{E}}[\bm{r}^T\bar{\U}]$, 
for all $\bm{r}\in\Delta_{\bar{r}}$ we have 
\begin{align*}
 \mathrm{MMD}(\bm{r}^T\U) 
&\leq 
 \mathrm{MMD}(\bm{r}^T\bar{\U}) 
 +
 \bigl\|\phi(\bm{r}^T\U)-\phi(\bm{r}^T\bar{\U})\bigr\|_{\mathcal{H}_k} \\
&=
 \mathrm{MMD}(\bm{r}^T\bar{\U})
 +
 \Bigl\|\sum_{i\in R_{12}^{(\eta)}}\sum_\ell r_\ell\theta_{\ell i}
 \bigl(\phi(p_i)-\phi(\Pi_{\mathcal{E}}[p_i])\bigr)
 \Bigr\|_{\mathcal{H}_k}\\
 &\leq 
 \mathrm{MMD}(\bm{r}^T\bar{\U}) + \sum_{\ell\in[L]}|r_\ell| 
 \sum_{i\in R_{12}^{(\eta)}}\theta_{\ell i} K\eta\\
&\leq 
 \mathrm{MMD}(\bm{r}^T\bar{\U}) + K\bar{r}\eta. 
\end{align*}
In the same way, we obtain 
\begin{align*}
 \mathrm{MMD}(\bm{r}^T\U) \geq \mathrm{MMD}(\bm{r}^T\bar{\U}) -K\bar{r}\eta,
\end{align*}
and hence 
\begin{align*}
 \big|\mathrm{MMD}(\bm{r}^T\U) - \mathrm{MMD}(\bm{r}^T\bar{\U})\big|
 < K\bar{r}\eta
 \qquad \text{for all } \bm{r}\in \Delta_{\bar{r}}. 
\end{align*} 
\end{proof}

The proof of Theorem~\ref{thm:near-independence-stability} is shown below. 
\begin{proof}
[Proof of Theorem~\ref{thm:near-independence-stability}]

We first record a stability fact used in both parts.  Choose pairwise disjoint
neighborhoods $B_i$, $i\in R_{12}^{(\eta)}$, of $\bm r_i$, contained in
the neighborhoods in the Hessian assumption.  Since the restrictions of quartic
polynomials to the affine hull of $\Delta_{\bar r}$ form a finite-dimensional
space, there exists $\epsilon_0'>0$ such that every quartic polynomial $g$
with
\[
  \|g-\bar d\|_{\infty,\Delta_{\bar r}}\le \epsilon_0'
\]
is strongly convex on each $B_i\cap\Delta_{\bar r}$.  By the zero
characterization assumed for the modified exact model, $\bar d$ has zeros on
$\Delta_{\bar r}$ exactly at $\bm r_i$, $i\in R_{12}^{(\eta)}$.  Hence
\[
  \zeta_0
  :=
  \min_{\bm r\in\Delta_{\bar r}\setminus\cup_{i\in R_{12}^{(\eta)}}B_i}
  \bar d(\bm r)>0 .
\]
Take $c_0>0$ so small that $c_0\le \zeta_0/2$ and
$c_0/2\le\epsilon_0'$.

For part (a), put $\epsilon_\eta:=4K^2\bar r^3\eta$.  By
\eqref{eq:near-independence-loss-gap},
$\|d-\bar d\|_{\infty,\Delta_{\bar r}}<\epsilon_\eta$.  Since
$8K^2\bar r^3\eta\le c<c_0$, we have
$\epsilon_\eta\le c/2<c_0/2\le\epsilon_0'$.  Therefore $d$ is strongly
convex on each $B_i\cap\Delta_{\bar r}$.  Moreover,
\[
  d(\bm r_i)<c,
  \qquad
  d(\bm r)>c
  \quad
  \text{for }\bm r\in\Delta_{\bar r}\setminus\cup_{i\in R_{12}^{(\eta)}}B_i .
\]
Thus, with
\[
  A_i:=\{\bm r\in\Delta_{\bar r}:d(\bm r)\le c\}\cap B_i,
\]
the level set $\{\bm r\in\Delta_{\bar r}:d(\bm r)\le c\}$ is the disjoint
union of the nonempty connected sets $A_i$, and each $A_i$ contains exactly
one $\bm r_i$.  Let $\bm r_*$ be a local minimizer of $d$ on this level
set, and suppose $\bm r_*\in A_i$.  Since $d$ is convex on $A_i$,
$\bm r_*$ is the minimizer of $d$, equivalently of
$\mathrm{MMD}(\bm r^T\U)$, over $A_i$.
Then Lemma~\ref{lem:near-independent-replacement} leads to 
 \begin{align*}
  \mathrm{MMD}(\bm{r}_*^T\bar{\U})
  &\leq 
  \mathrm{MMD}(\bm{r}_*^T\U) + K\bar{r}\eta\\
  &\leq 
  \mathrm{MMD}(\bm{r}_i^T\U) + K\bar{r}\eta\\
  &\leq 
  \mathrm{MMD}(\bm{r}_i^T\bar{\U}) + 2K\bar{r}\eta =2K\bar{r}\eta. 
 \end{align*}
The assumption on the Hessian matrix ensures that 
\begin{align*}
 \frac{\lambda_i K^2}{2}\|\bm{r}_*-\bm{r}_i\|^2
 \leq 
 \bar{d}(\bm{r}_*)
 \leq 
 (2K\bar{r}\eta)^2 
 \quad \Longrightarrow\quad
 \|\bm{r}_*-\bm{r}_i\|\leq \frac{3\bar{r}\eta}{\sqrt{\lambda_i}}. 
\end{align*}
The discrepancy between $\bm{r}_*^T\U$ and $p_i=\bm{r}_i^T\U$ is bounded as 
\begin{align*}
 \big\|\phi(\bm{r}_*^T\U)-\phi(p_i)\big\|_{\mathcal{H}_k}
& \leq 
 \sum_\ell|r_{*,\ell}-r_{i,\ell}|\sum_{j}\theta_{\ell j}\big\|\phi(p_j)\big\|_{\mathcal{H}_k}\\
& \leq 
 \|\bm{r}_*-\bm{r}_i\|_1 K\\
&\leq  
 3K\bar{r}\eta\sqrt{\frac{m}{\lambda_i}}. 
\end{align*}
This proves (a).


For part (b), work on the event in the statement.  Since both
$\mathrm{MMD}(\bm r^T\U)$ and $\mathrm{MMD}(\bm r^T\widehat\U)$ are bounded
by $2K\bar r^2$ on $\Delta_{\bar r}$, the event implies
\[
  \|\widehat d-d\|_{\infty,\Delta_{\bar r}}
  \le 4K^2\bar r^2\zeta_n .
\]
Together with \eqref{eq:near-independence-loss-gap}, this gives
\[
  \|\widehat d-\bar d\|_{\infty,\Delta_{\bar r}}
  \le 4K^2\bar r^2(\bar r\eta+\zeta_n)=:\epsilon_{\eta,n}.
\]
Since $8K^2\bar r^2(\bar r\eta+\zeta_n)\le c'<c_0$, we have
$\epsilon_{\eta,n}\le c'/2<c_0/2\le\epsilon_0'$.  Hence the same argument as
in part (a) gives
\[
  \{\bm r\in\Delta_{\bar r}:\widehat d(\bm r)\le c'\}
  =\bigcup_{i\in R_{12}^{(\eta)}}\widehat A_i,
  \qquad
  \widehat A_i:=\{\bm r\in\Delta_{\bar r}:\widehat d(\bm r)\le c'\}\cap B_i,
\]
where the $\widehat A_i$ are nonempty, pairwise disjoint, and connected, and
each contains exactly one $\bm r_i$.  Moreover, $\widehat d$ is strongly
convex on each $B_i\cap\Delta_{\bar r}$.  Thus any local minimizer
$\widehat{\bm r}\in\widehat A_i$ is the minimizer of $\widehat d$,
equivalently of $\mathrm{MMD}(\bm r^T\widehat\U)$, over $\widehat A_i$.
Then, it holds that 
 \begin{align*}
  \mathrm{MMD}(\widehat{\bm{r}}^T\bar{\U})
  &\leq 
  \mathrm{MMD}(\widehat{\bm{r}}^T\U) + K\bar{r}\eta\\
  &\leq 
  \mathrm{MMD}(\widehat{\bm{r}}^T\widehat{\U}) + K\bar{r}\eta + K\zeta_n\\
  &\leq 
  \mathrm{MMD}(\bm{r}_i^T\widehat{\U}) + K\bar{r}\eta + K\zeta_n\\
  &\leq 
  \mathrm{MMD}(\bm{r}_i^T\U) + K\bar{r}\eta + 2K\zeta_n\\
  &\leq 
  \mathrm{MMD}(\bm{r}_i^T\bar{\U}) + 2K\bar{r}\eta + 2K\zeta_n
  =
  2K\bar{r}\eta + 2K\zeta_n
 \end{align*}
 with high probability. 
 The assumption on the Hessian matrix ensures that 
\begin{align*}
 \|\widehat{\bm{r}}-\bm{r}_i\|
 \leq \frac{3(\bar{r}\eta+\zeta_n)}{\sqrt{\lambda_i}}. 
\end{align*}
The discrepancy between $\widehat{\bm{r}}^T\U$ and $p_i=\bm{r}_i^T\U$ is bounded as 
\begin{align*}
 \big\|\phi(\widehat{\bm{r}}^T\U)-\phi(p_i)\big\|_{\mathcal{H}_k}
& \leq 
 \sum_\ell|\widehat{r}_{\ell}-r_{i,\ell}|\sum_{j}\theta_{\ell j}\big\|\phi(p_j)\big\|_{\mathcal{H}_k}\\
& \leq 
 \|\widehat{\bm{r}}-\bm{r}_i\|_1 K\\
&\leq 
 3K(\bar{r}\eta+\zeta_n)\sqrt{\frac{m}{\lambda_i}}. 
\end{align*}
This proves (b).
\end{proof}

\section{Additional numerical results}
\label{app:additional-numerical-results}

The main text reports a compact selection of experiments.  This appendix keeps
all detailed numerical results that are useful for reproducibility and for
checking failure modes, but are less central to the top-level empirical message.

\subsection{Controlled synthetic mixture models}
\label{app:controlled-synthetic-results}

We evaluate PM-MMD on controlled synthetic mixture models.  The purpose is to
separate identifiability issues from numerical and finite-sample effects.  The
Gaussian experiments specify the marginal-independence structure explicitly
through the covariance matrices, while the controlled non-Gaussian design checks
whether the method relies on marginal-independence structure rather than
Gaussian cluster geometry.

In all controlled synthetic experiments, the observed distributions are mixtures
\begin{align*}
  U_\ell=\sum_{i=1}^m \theta_{\ell i}p_i .
\end{align*}
Candidate weight vectors are generated by constrained numerical minimization of
the empirical pairwise MMD objective on the training split, under the constraints
\begin{align*}
  \sum_\ell r_\ell=1,  \qquad  \|\bm r\|_1\le \bar{r}. 
\end{align*}
Representatives are selected on an independent validation split using the
representative-selection step in Algorithm~\ref{alg:pmmd}.  
Unless otherwise stated, we use an RBF product kernel on the selected coordinate
pair after coordinate-wise median/MAD scaling; the RBF bandwidths are chosen by
a median-distance heuristic on the scaled training samples~\cite{garreau2017large}.
In the synthetic experiments, recovery
is evaluated through the diagnostic vectors $\bm r^T\Theta$: successful recovery
corresponds to these vectors being close to simplex vertices.  The exact
Gaussian parameters used in the GMM experiments are listed in
Appendix~\ref{app:gmm-designs}.

We first consider a symmetric Gaussian design in which the assumptions of
Theorem~\ref{thm:marginal-independence-identifiability} and
Corollary~\ref{cor:full-recovery-marginal-independence} are intentionally violated.
In this design, all components have zero mean and unit marginal variances, and
each component has one independent coordinate pair.  However, for each target
pair, the two non-target components have identical bivariate Gaussian
marginals.  For example, for the pair $(1,2)$, the non-target components
$p_2$ and $p_3$ have the same $(1,2)$-marginal.  Hence, for any nonzero scalar
$a$, $\e_1+a(\e_2-\e_3)$ 
defines an affine component-weight vector whose $(1,2)$-marginal coincides 
with that of $p_1$.  This gives an independent affine combination with
nonzero coefficients outside $R_{12}$, directly violating
condition~\eqref{eq:no-cancellation} in
Theorem~\ref{thm:marginal-independence-identifiability}.  The subset-rank
condition is vacuous in this design because every target set $R_{st}$ is a
singleton.  The resulting
signed cancellation is a population-level non-identifiability and is therefore
not removed by validation.  This experiment confirms that
condition~\eqref{eq:no-cancellation} is not merely technical.

We then consider two asymmetric GMM designs constructed to avoid the above
signed cancellation.  The non-target bivariate Gaussian marginals are made
distinct for every diagnostic pair, and the univariate Gaussian marginals are
also separated across components.  These choices are intended to match the
pairwise no-cancellation condition in
Theorem~\ref{thm:marginal-independence-identifiability}.  In the multi-component
design, the separated target marginals also satisfy the subset-rank condition
in Corollary~\ref{cor:full-recovery-marginal-independence}; in the singleton
design, that condition is vacuous.

The first asymmetric design is a singleton case with $d=m=3$:
\[
  R_{12}=\{1\},
  \qquad
  R_{13}=\{2\},
  \qquad
  R_{23}=\{3\}.
\]
The second is a multi-component case with $d=m=4$:
\[
  R_{12}=\{1,2\},
  \qquad
  R_{13}=\{3,4\},
  \qquad
  R_{23}=\{1,3\},
  \qquad
  R_{24}=\{2,4\}.
\]
In the latter case, the pairwise step intentionally retains several
validation-good and mutually separated representatives from each coordinate
pair, and the global aggregation step selects four global representatives.
We use the non-oracle selection rule in the $\bm r$-space.

\begin{table}[!tbp]
\centering
\small
\caption{Summary of the controlled GMM experiments.  Errors are reported as
mean $\pm$ standard error over ten random seeds.  ``Rec.'' denotes vertex
recovery: singleton recovery requires nearest-vertex agreement for all
diagnostic pairs, and multi-component recovery requires global recovery of all
components.  ``Max vtx. dist.'' denotes the mean maximum nearest-vertex
distance.  Pairwise vertex coverage $R_e\subseteq\widehat R_e$ was $10/10$ for
every listed sample size.}
\label{tab:gmm-summary}
\begin{tabular}{llccc}
\hline
Design & $n$ & Rec. & Relative error & Max vtx. dist. \\
\hline
$|R_e|=1$ & 2000  & $10/10$ & $0.0925 \pm 0.0110$ & $0.1263 \pm 0.0099$ \\
$|R_e|=1$ & 5000  & $10/10$ & $0.0636 \pm 0.0075$ & $0.0942 \pm 0.0095$ \\
$|R_e|=1$ & 10000 & $10/10$ & $0.0487 \pm 0.0054$ & $0.0684 \pm 0.0072$ \\
$|R_e|=1$ & 20000 & $10/10$ & $0.0323 \pm 0.0033$ & $0.0487 \pm 0.0038$ \\
\hline
$|R_e|=2$ & 2000  & $4/10$  & $0.3864 \pm 0.0280$ & $0.5882 \pm 0.0243$ \\
$|R_e|=2$ & 5000  & $7/10$  & $0.2390 \pm 0.0701$ & $0.2865 \pm 0.0707$ \\
$|R_e|=2$ & 10000 & $10/10$ & $0.0594 \pm 0.0089$ & $0.0983 \pm 0.0159$ \\
$|R_e|=2$ & 20000 & $10/10$ & $0.0380 \pm 0.0041$ & $0.0543 \pm 0.0050$ \\
\hline
\end{tabular}
\end{table}

Table~\ref{tab:gmm-summary} summarizes the results.  In the singleton case, the method recovered the correct component vertex for every coordinate pair in all runs, even at $n=2000$.  Both the relative Frobenius error of 
the estimated mixing matrix and the maximum distance of the selected diagnostic
vectors $\bm r^T\Theta$ to their target vertices decrease with $n$.
Figure~\ref{fig:gmm-singleton-simplex} overlays the selected diagnostic vectors from ten independent runs at $n=20000$.  The points concentrate near the corresponding simplex vertices, illustrating the stability of the singleton recovery.

The multi-component case is more demanding.  Pairwise coverage was successful for every pair, every seed, and all tested sample sizes, meaning that the pairwise representative sets always contained the true target components.
However, the global aggregation step is more sensitive to sample size.  At $n=2000$ and $n=5000$, all four components were recovered in only $4$ out of $10$ and $7$ out of $10$ seeds, respectively.  From $n=10000$ onward, global recovery was successful in all seeds, and the relative Frobenius error continued to decrease at $n=20000$.

\begin{figure}[!t] 
\centering
\includegraphics[width=0.4\linewidth]{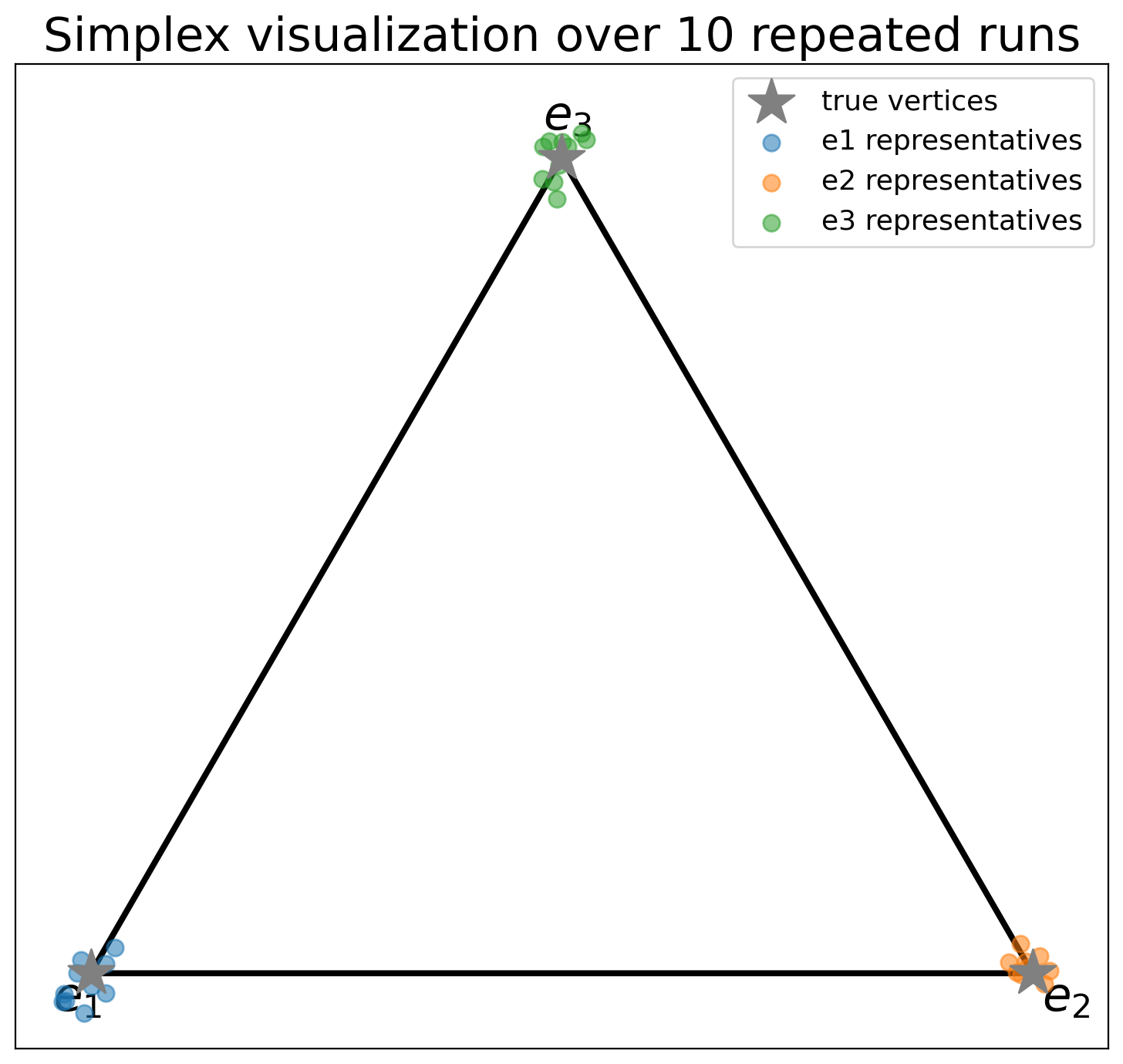}
\caption{Simplex visualization of the singleton GMM experiment over ten
independent runs at $n=20000$.  Each point represents a selected diagnostic
vector $\bm r^T\Theta$ from one coordinate pair and one random seed.  The gray
stars indicate the true simplex vertices.  The concentration of the selected
points near the corresponding vertices indicates stable component recovery.}
\label{fig:gmm-singleton-simplex}
\end{figure}

\begin{figure}[!t]
\centering
\begin{minipage}{0.49\linewidth}
\centering
\includegraphics[width=\linewidth]{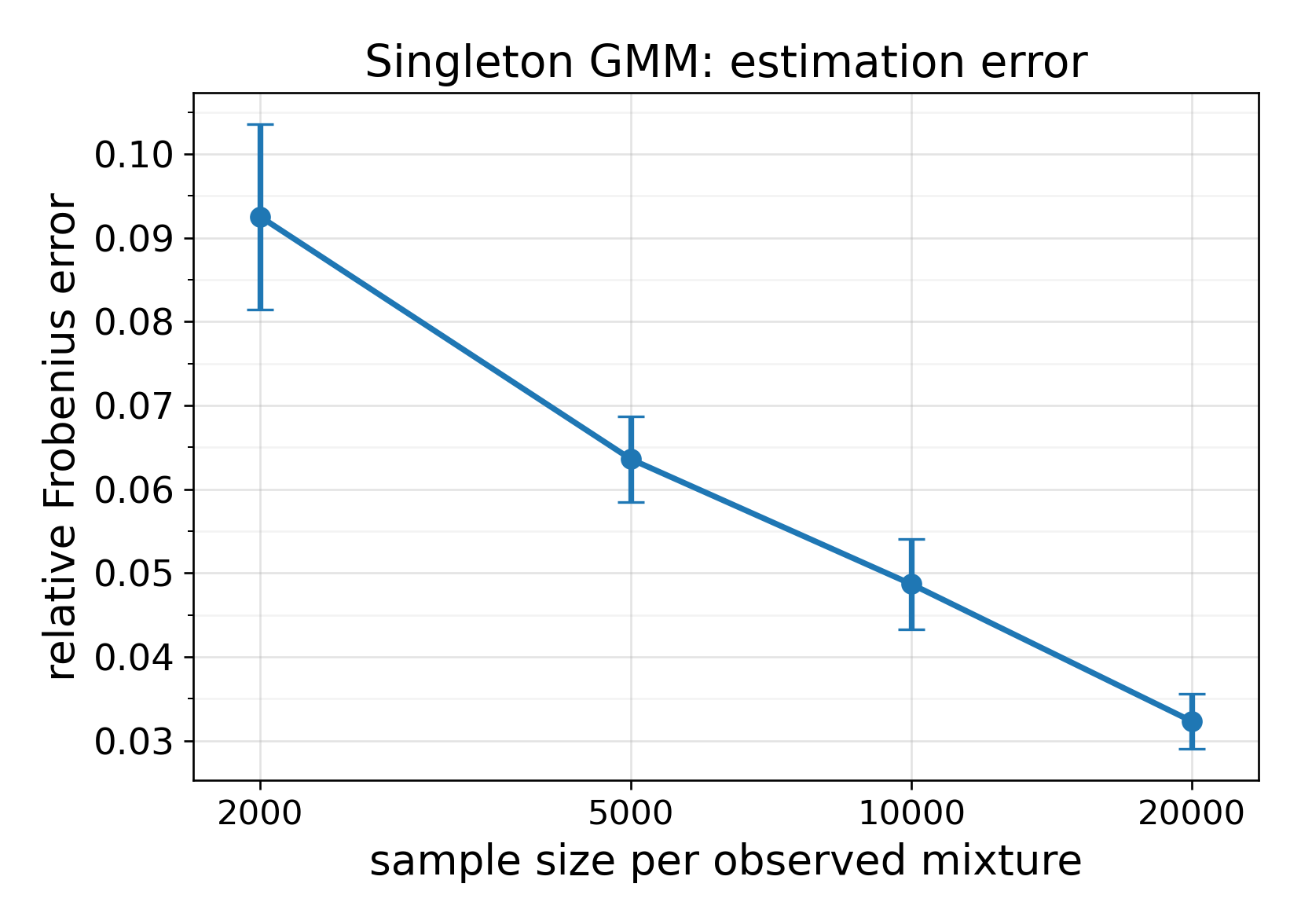}

\vspace{1mm}
{\small (a) $|R_e|=1$: mixing-matrix error.}
\end{minipage}
\hfill
\begin{minipage}{0.49\linewidth}
\centering
\includegraphics[width=\linewidth]{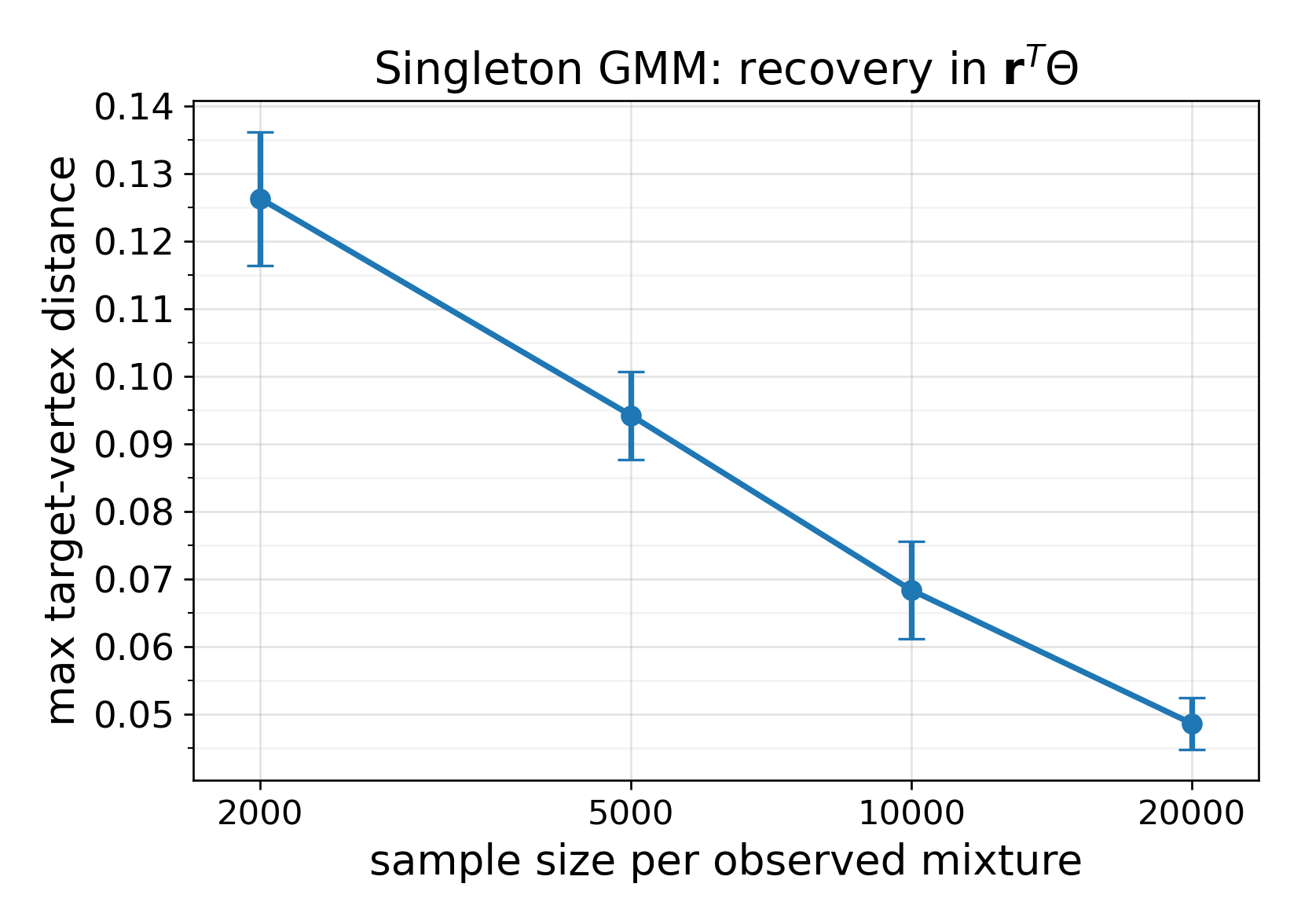}

\vspace{1mm}
{\small (b) $|R_e|=1$: vertex recovery in $\bm r^T\Theta$.}
\end{minipage}

\vspace{3mm}

\begin{minipage}{0.49\linewidth}
\centering
\includegraphics[width=\linewidth]{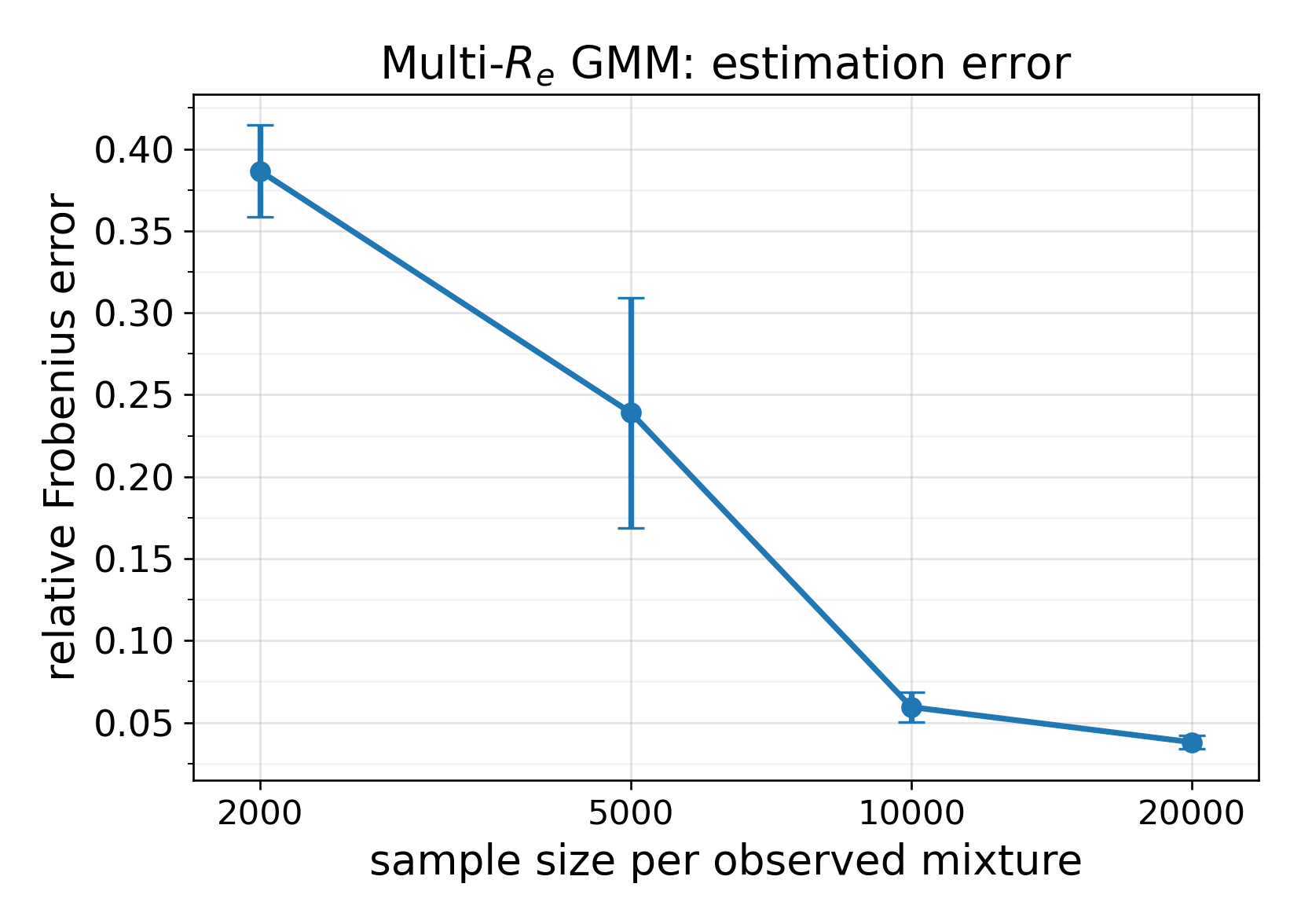}

\vspace{1mm}
{\small (c) $|R_e|=2$: mixing-matrix error.}
\end{minipage}
\hfill
\begin{minipage}{0.49\linewidth}
\centering
\includegraphics[width=\linewidth]{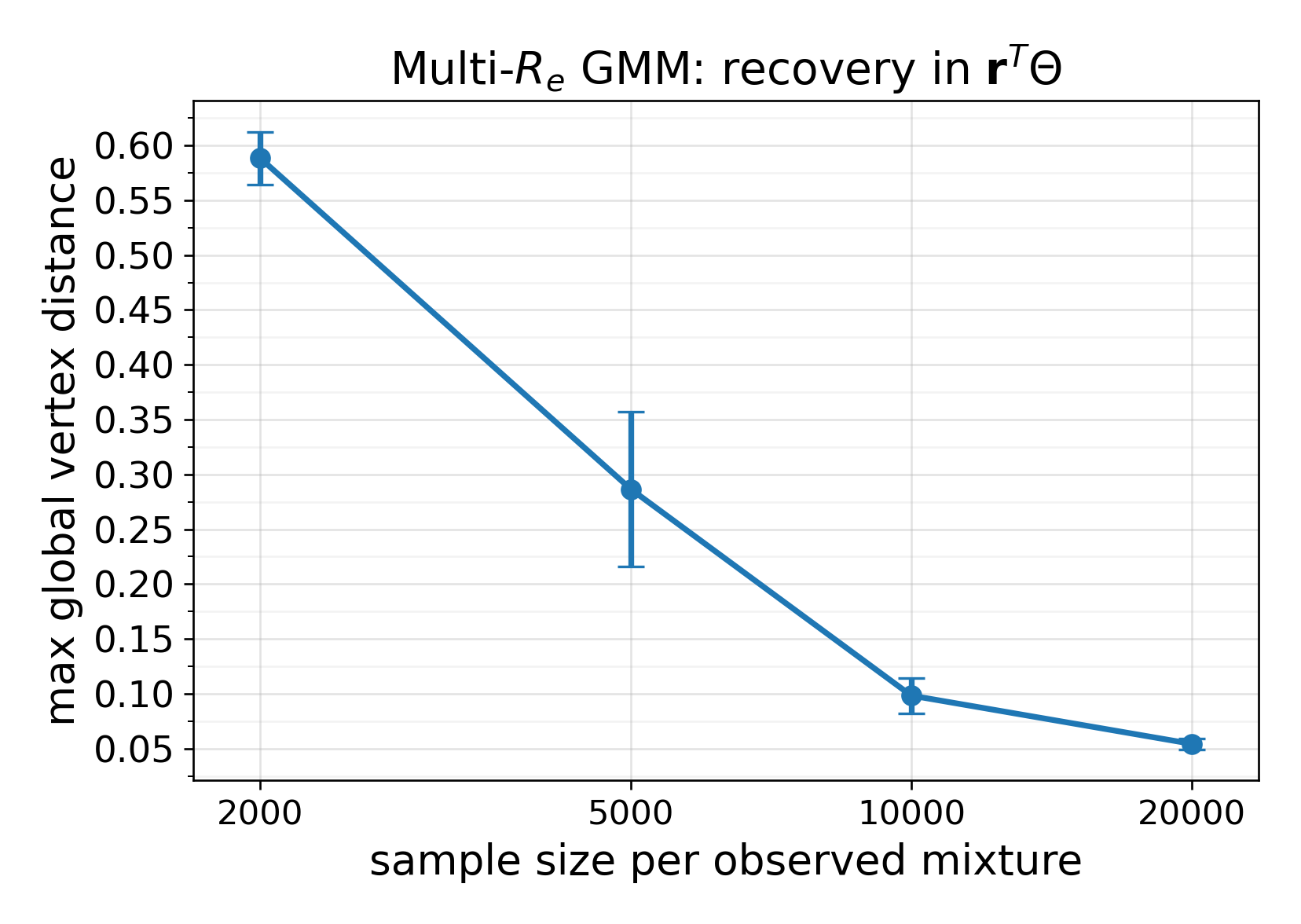}

\vspace{1mm}
{\small (d) $|R_e|=2$: global vertex recovery in $\bm r^T\Theta$.}
\end{minipage}

\caption{Sample-size dependence in the controlled GMM experiments.  Panels
(a) and (b) show the singleton setting $|R_e|=1$, and panels (c) and (d) show
the multi-component setting $|R_e|=2$.  The left column reports the relative
Frobenius error of the estimated mixing matrix, and the right column reports
the maximum distance of the selected diagnostic vectors $\bm r^T\Theta$ to
their target simplex vertices.  Error bars show standard errors over ten
random seeds.}
\label{fig:gmm-sample-size}
\end{figure}

The trends in Figure~\ref{fig:gmm-sample-size} show a clear difference between the singleton and multi-component settings.  The singleton case is stable even at small sample sizes, whereas the multi-component case exhibits unstable global aggregation at $n=2000$ and $n=5000$ before stabilizing from $n=10000$ onward. 
These results indicate that validation-diversity selection is effective under the identifiability conditions, while the more difficult
$|R_e|>1$ setting requires larger samples for stable global aggregation.

\paragraph{Controlled non-Gaussian components.}
We also consider a controlled non-Gaussian design with the same singleton
marginal-independence pattern
\[
  R_{12}=\{1\},\qquad
  R_{13}=\{2\},\qquad
  R_{23}=\{3\}.
\]
For each component, the target coordinate pair is generated 
independently from asymmetric or multimodal non-Gaussian marginals, 
including shifted gamma, log-normal, Student-$t$, and two-component Gaussian mixture distributions.
The remaining coordinate is then generated as a nonlinear function of the target pair plus component-specific noise.  The exact component-wise construction is given in Appendix~\ref{app:nongauss-design}.  
Hence the components are skewed or multimodal, and are not well described as single Gaussian clusters.  The mixing matrix is the same as in the singleton GMM experiment.

\begin{table}[t]
\centering
\small
 \caption{Controlled non-Gaussian mixture experiment.  Entries are relative
 Frobenius errors of the estimated mixing matrix, reported as mean $\pm$
 standard error over ten random seeds. 
Boldface indicates the lowest error in
each row, and underlining indicates the second lowest error.  The last column,
``PM rec.'', reports PM-MMD vertex recovery over ten seeds; a seed is recovered when all diagnostic directions have nearest-vertex agreement with their target vertices.}
\label{tab:nongauss-controlled}
\begin{tabular}{lccccc}
\hline
$n$ & PM-MMD & GM fit & $k$-means & NMF & PM rec. \\
\hline
500
& \textbf{0.2470 $\pm$ 0.0740}
& 0.5027 $\pm$ 0.0080
& \underline{0.4363 $\pm$ 0.0042}
& 0.7107 $\pm$ 0.0105
& $9/10$ \\
1000
& \textbf{0.1338 $\pm$ 0.0166}
& 0.5036 $\pm$ 0.0127
& \underline{0.4444 $\pm$ 0.0019}
& 0.6464 $\pm$ 0.0044
& $10/10$ \\
2000
& \textbf{0.1996 $\pm$ 0.1049}
& 0.5193 $\pm$ 0.0107
& \underline{0.4427 $\pm$ 0.0022}
& 0.6038 $\pm$ 0.0065
& $9/10$ \\
\hline
\end{tabular}
\end{table}

Table~\ref{tab:nongauss-controlled} shows that PM-MMD remains effective when the component distributions are non-Gaussian.
At $n=1000$, the method recovers the correct vertices in all ten seeds and substantially outperforms Gaussian mixture fitting, 
$k$-means, and histogram NMF.  
This supports the interpretation that the proposed method exploits marginal independence rather than Gaussian cluster geometry.  
The single failures at $n=500$ and $n=2000$ indicate that the non-Gaussian setting is still affected by finite-sample selection and aggregation, but PM-MMD remains the best method among the clustering and factorization baselines in all reported sample sizes.

\subsection{Semi-synthetic benchmark data and clean-component calibrations}
\label{app:semisynthetic-results}

We also evaluate the method on semi-synthetic benchmark data.  These
experiments are not pure real-data experiments: the component identities and
the mixing matrix are controlled, while the component marginals are derived
from class-conditional empirical distributions of benchmark data, namely
Digits, Wine, and Dry Bean
\cite{alpaydin1998optdigits,aeberhard1992wine,drybean2020uci}.  This allows
us to evaluate the diagnostic vectors $\bm r^T\Theta$ and the mixing-matrix
error beyond Gaussian component models.  We use the same candidate-generation
and representative-selection pipeline as in
Section~\ref{app:controlled-synthetic-results}.

The construction enforces the singleton marginal-independence pattern
\[
  R_{12}=\{1\},\qquad
  R_{13}=\{2\},\qquad
  R_{23}=\{3\},
\]
by independently resampling the target coordinate pair within each component
and adding controlled dependence to the remaining coordinate.  Details of the
resampling scheme and feature choices are given in
Appendix~\ref{app:semisynthetic-benchmark-designs}.

We first compare PM-MMD with three reference methods that do not use marginal
independence: Gaussian mixture fitting using EM
\cite{dempster1977maximum}, $k$-means clustering
\cite{macqueen1967some}, and histogram NMF
\cite{lee1999learning}.  These methods use only the observed mixtures and
serve as parametric clustering, geometric clustering, and distributional
factorization references, respectively.

We report both unwarped and cubic-warped versions of the semi-synthetic
benchmarks.  The unwarped setting uses the benchmark-derived empirical
marginals directly.  In the cubic-warped setting, after robust coordinate-wise
standardization, each coordinate $z$ is transformed to $z+0.3z^3$.  Since this
transformation is coordinate-wise, it preserves the target marginal
independences while changing the component geometry.

\providecommand{\bestval}[1]{#1}
\renewcommand{\bestval}[1]{\ensuremath{\bm{#1}}}
\providecommand{\secondval}[1]{#1}
\renewcommand{\secondval}[1]{\ensuremath{\underline{#1}}}

\begin{table}[!tbp]
\centering
\scriptsize
\caption[Semi-synthetic benchmark results.]{
Semi-synthetic benchmark results.  Entries are relative Frobenius errors of
the estimated mixing matrix, reported as mean $\pm$ standard error over ten
random seeds.  
Boldface indicates the lowest error in each row, and
underlining indicates the second lowest error.  The last column, ``PM rec.'',
reports PM-MMD vertex recovery over ten seeds, i.e., the number of seeds for
which all three selected diagnostic vectors $\bm r^T\Theta$ are nearest to
their target simplex vertices.}
\label{tab:semisynthetic-benchmarks}
\resizebox{\linewidth}{!}{
\begin{tabular}{lllcccccc}
\hline
Data & Setting & Warp & $n$ & PM-MMD & GM fit & $k$-means & NMF & PM rec. \\
\hline
Digits & $\{0,1,2\}$, feat. $\{59,61,3\}$
& none
& 500
& \bestval{0.1905 \pm 0.0257}
& 0.7234 $\pm$ 0.0056
& 0.7460 $\pm$ 0.0063
& \secondval{0.6339 \pm 0.0169}
& $10/10$ \\

Digits & $\{0,1,2\}$, feat. $\{59,61,3\}$
& cubic
& 500
& \bestval{0.4271 \pm 0.0579}
& 0.6708 $\pm$ 0.0485
& 1.0052 $\pm$ 0.0034
& \secondval{0.6339 \pm 0.0169}
& $7/10$ \\

Digits & $\{0,1,2\}$, feat. $\{59,61,3\}$
& none
& 1000
& \bestval{0.1203 \pm 0.0160}
& 0.7346 $\pm$ 0.0026
& 0.7533 $\pm$ 0.0037
& \secondval{0.6206 \pm 0.0081}
& $10/10$ \\

Digits & $\{0,1,2\}$, feat. $\{59,61,3\}$
& cubic
& 1000
& \bestval{0.3292 \pm 0.0947}
& 0.7555 $\pm$ 0.0567
& 1.0062 $\pm$ 0.0026
& \secondval{0.6206 \pm 0.0081}
& $9/10$ \\
\hline
Wine & $\{0,1,2\}$, feat. $\{6,12,11\}$
& none
& 500
& \secondval{0.2679 \pm 0.1010}
& \bestval{0.0508 \pm 0.0056}
& 0.4308 $\pm$ 0.0050
& 0.6587 $\pm$ 0.0129
& $8/10$ \\

Wine & $\{0,1,2\}$, feat. $\{6,12,11\}$
& cubic
& 500
& \secondval{0.4512 \pm 0.1107}
& \bestval{0.4074 \pm 0.0371}
& 0.7130 $\pm$ 0.0445
& 0.6587 $\pm$ 0.0129
& $5/10$ \\

Wine & $\{0,1,2\}$, feat. $\{6,12,11\}$
& none
& 1000
& \secondval{0.0902 \pm 0.0110}
& \bestval{0.0342 \pm 0.0034}
& 0.4222 $\pm$ 0.0033
& 0.6147 $\pm$ 0.0068
& $10/10$ \\

Wine & $\{0,1,2\}$, feat. $\{6,12,11\}$
& cubic
& 1000
& \bestval{0.0873 \pm 0.0123}
& \secondval{0.3590 \pm 0.0238}
& 0.6520 $\pm$ 0.0232
& 0.6147 $\pm$ 0.0068
& $10/10$ \\
\hline
Dry Bean-A & $\{6,4,2\}$, feat. $\{6,11,9\}$
& none
& 500
& \secondval{0.2073 \pm 0.0235}
& \bestval{0.0704 \pm 0.0083}
& 0.4839 $\pm$ 0.0083
& 0.6824 $\pm$ 0.0132
& $10/10$ \\

Dry Bean-A & $\{6,4,2\}$, feat. $\{6,11,9\}$
& cubic
& 500
& \bestval{0.3697 \pm 0.0921}
& \secondval{0.6517 \pm 0.0216}
& 1.1012 $\pm$ 0.0065
& 0.6824 $\pm$ 0.0132
& $9/10$ \\

Dry Bean-B & $\{6,5,4\}$, feat. $\{1,6,0\}$
& none
& 500
& \secondval{0.2239 \pm 0.1022}
& \bestval{0.0520 \pm 0.0063}
& 0.4892 $\pm$ 0.0049
& 0.6258 $\pm$ 0.0224
& $9/10$ \\

Dry Bean-B & $\{6,5,4\}$, feat. $\{1,6,0\}$
& cubic
& 500
& 0.6380 $\pm$ 0.1664
& \bestval{0.4366 \pm 0.0512}
& 1.0739 $\pm$ 0.0048
& \secondval{0.6258 \pm 0.0224}
& $5/10$ \\
\hline
\end{tabular}}
\end{table}

Table~\ref{tab:semisynthetic-benchmarks} shows that the benchmark geometry
matters.  On the Digits benchmark, PM-MMD is the best method in both the
unwarped and cubic-warped settings, although the cubic warp makes vertex
recovery less stable at $n=500$.  On the Wine benchmark, Gaussian mixture
fitting is strongest in the unwarped setting, but PM-MMD becomes the best
method at $n=1000$ after the cubic warp.  The Dry Bean results show that the
effect of the cubic warp is feature-dependent: in setting A, the cubic warp
makes PM-MMD the best method among the reference methods considered, whereas
in setting B, PM-MMD becomes less stable and Gaussian mixture fitting remains
stronger.

These results also indicate when PM-MMD is expected to be useful.  The method
is not a clustering procedure; it relies on pairwise marginal independence as
the identifying signal.  It is most advantageous when the components are not
well captured by simple Gaussian clusters, but the selected coordinate pairs
still exhibit a clear contrast between independent target components and
dependent non-target components.  When the benchmark geometry is favorable to
Gaussian clustering, or when the selected features do not contain a strong
dependence contrast, Gaussian mixture fitting can remain stronger and PM-MMD
can become unstable.

\paragraph{Oracle-assisted binary MPE baselines.}
We next compare with standard binary MPE methods.  Methods such as
KM2~\cite{ramaswamy2016mixture}, DEDPUL~\cite{ivanov2019dedpul}, and
BBE~\cite{garg2021mixture} are designed for binary mixture proportion
estimation or positive-unlabeled learning, where samples from a mixture and
samples from one clean component are available.  Since the present
multi-mixture problem provides only unlabeled observed mixtures, these methods
are not directly applicable under the same information structure.  For
calibration, we therefore evaluate them as oracle-assisted references.  For
each observed mixture and each true component, clean samples from the
component are supplied to the binary MPE method, and the resulting estimates
are assembled into a mixing matrix.  These baselines use stronger information
than PM-MMD and should not be interpreted as direct competitors under the
same information condition.

\begin{table}[!tbp] 
\centering
\scriptsize
\caption[Oracle-assisted binary MPE baselines.]{
Oracle-assisted binary MPE baselines on the semi-synthetic benchmarks.  These
methods are given clean samples from each true component and estimate each
entry of the mixing matrix by a binary MPE problem.  They therefore use
stronger information than PM-MMD and are included as oracle-assisted
references rather than direct competitors under the same information
structure.  Entries are relative Frobenius errors of the estimated mixing
matrix, reported as mean $\pm$ standard error over ten random seeds.}
\label{tab:oracle-binary-mpe}
\resizebox{\linewidth}{!}{
\begin{tabular}{llccccc}
\hline
Data & Warp & $n$ & PM-MMD & KM2-oracle & DEDPUL-oracle & BBE-oracle \\
\hline
Digits & none & 500
& 0.1905 $\pm$ 0.0257
& 0.1693 $\pm$ 0.0083
& \secondval{0.0655 \pm 0.0041}
& \bestval{0.0553 \pm 0.0043} \\
Digits & cubic & 500
& 0.4271 $\pm$ 0.0579
& 0.1499 $\pm$ 0.0114
& \secondval{0.0654 \pm 0.0040}
& \bestval{0.0552 \pm 0.0044} \\
Digits & none & 1000
& 0.1203 $\pm$ 0.0160
& 0.1442 $\pm$ 0.0108
& \secondval{0.0930 \pm 0.0073}
& \bestval{0.0261 \pm 0.0032} \\
Digits & cubic & 1000
& 0.3292 $\pm$ 0.0947
& 0.1318 $\pm$ 0.0120
& \secondval{0.0927 \pm 0.0071}
& \bestval{0.0261 \pm 0.0032} \\
\hline
Wine & none & 500
& 0.2679 $\pm$ 0.1010
& 0.0877 $\pm$ 0.0124
& \secondval{0.0853 \pm 0.0053}
& \bestval{0.0585 \pm 0.0060} \\
Wine & cubic & 500
& 0.4512 $\pm$ 0.1107
& \secondval{0.0758 \pm 0.0106}
& 0.0851 $\pm$ 0.0051
& \bestval{0.0586 \pm 0.0060} \\
Wine & none & 1000
& 0.0902 $\pm$ 0.0110
& \secondval{0.0831 \pm 0.0097}
& 0.1455 $\pm$ 0.0108
& \bestval{0.0313 \pm 0.0027} \\
Wine & cubic & 1000
& 0.0873 $\pm$ 0.0123
& \secondval{0.0741 \pm 0.0093}
& 0.1456 $\pm$ 0.0110
& \bestval{0.0313 \pm 0.0027} \\
\hline
Dry Bean-A & none & 500
& 0.2073 $\pm$ 0.0235
& 0.1311 $\pm$ 0.0110
& \secondval{0.0729 \pm 0.0050}
& \bestval{0.0594 \pm 0.0049} \\
Dry Bean-A & cubic & 500
& 0.3697 $\pm$ 0.0921
& 0.1349 $\pm$ 0.0087
& \secondval{0.0736 \pm 0.0050}
& \bestval{0.0594 \pm 0.0049} \\
\hline
Dry Bean-B & none & 500
& 0.2239 $\pm$ 0.1022
& 0.1326 $\pm$ 0.0091
& \secondval{0.0779 \pm 0.0067}
& \bestval{0.0646 \pm 0.0059} \\
Dry Bean-B & cubic & 500
& 0.6380 $\pm$ 0.1664
& 0.1727 $\pm$ 0.0101
& \secondval{0.0775 \pm 0.0064}
& \bestval{0.0647 \pm 0.0059} \\
\hline
\end{tabular}}
\end{table}

Table~\ref{tab:oracle-binary-mpe} shows that oracle-assisted binary MPE
methods, especially BBE, are very strong when clean component samples are
provided.  This is expected because these methods are evaluated under a
stronger information condition than PM-MMD.  The comparison clarifies the
role of the proposed method: PM-MMD is designed for the setting where only
the observed mixtures are available, and it uses marginal-independence
structure rather than clean component samples.

\paragraph{Clean-component binary variant of PM-MMD.}
We also evaluate a clean-component binary variant of PM-MMD.  In this
setting,
\[
  U_1=p_{\mathrm{clean}},
  \qquad
  U_2=0.55\,p_{\mathrm{clean}}+0.45\,p_{\mathrm{res}} .
\]
The modification is straightforward: the affine weight vector is restricted
to the one-dimensional residual path, and the PM-MMD of the
residual candidate is minimized.  We evaluate three binary pairs,
\(p_0\to p_1\), \(p_1\to p_2\), and \(p_2\to p_0\), so that each component
appears once as the residual component.  In each pair, the residual component
is constructed to be marginally independent on its designated coordinate pair,
so the clean-MI MMD estimator is correctly specified as a
marginal-independence estimator.

\begin{table}[!tbp]
\centering
\small
\caption[Clean-component binary MPE comparison.]{
Clean-component binary MPE comparison averaged over three binary pairs and
twelve semi-synthetic benchmark settings.  The first three rows report
row-wise summary errors: relative Frobenius errors of the binary mixing
matrix for the first two rows, and the absolute error of the estimated mixing
proportion for the third row.  The last two rows count how often each method
has the lowest and second-lowest row-wise relative Frobenius error.}
\label{tab:clean-binary-summary}
\begin{tabular}{lcccc}
\hline
Summary & clean-MI MMD & KM2 & DEDPUL & BBE \\
\hline
Mean relative error
& 0.0537
& \secondval{0.0470}
& 0.1632
& \bestval{0.0209} \\
Median relative error
& 0.0424
& \secondval{0.0363}
& 0.1462
& \bestval{0.0196} \\
Mean $|\widehat\theta-\theta|$
& 0.0466
& \secondval{0.0408}
& 0.1416
& \bestval{0.0182} \\
Best count
& 3
& 1
& 0
& 32 \\
Second count
& 12
& 20
& 0
& 4 \\
\hline
\end{tabular}
\end{table}

Table~\ref{tab:clean-binary-summary} shows that standard binary MPE methods
are very strong in the clean-component setting.  BBE is the most stable method
overall, and KM2 is also strong.  The clean-MI variant of PM-MMD is
competitive in some pairs, but it is not uniformly better than BBE or KM2.
This is consistent with the information structure: clean-component binary MPE
is precisely the setting for which existing binary MPE methods are designed.

We also computed a reverse-contamination diagnostic for these clean-binary
pairs by estimating the proportion of the clean component in a pure residual
component.  BBE and KM2 gave near-zero estimates on average
(\(0.0033\) and \(0.0135\), respectively), suggesting that these clean-binary
settings are operationally favorable to irreducibility-based binary MPE.
Thus these experiments should be read as a calibration in a favorable
clean-component setting for existing MPE methods.  The main setting of this
paper is different: clean component samples are not available, and marginal
independence is used to recover components from unlabeled mixtures.

\FloatBarrier
\subsection{Mixture-separation sweep on raw gated-pool DLBCL}
\label{app:raw-dlbcl-rho-sweep}

We supplement the raw gated-pool DLBCL experiment with a sweep over the
mixture-separation parameter
\[
  \Theta(\rho)=\rho I_3+\frac{1-\rho}{3}\mathbf 1\mathbf 1^T .
\]
The sweep uses the same three gated component pools and the same marker pairs
as the main DLBCL experiment.  PM-MMD uses the condition-aware stable selector
and an expanded search radius $\bar r=10$ for this stress test.  The transition
baselines use the same pooled observed samples and the same cross-fitted
logistic posterior construction as in the main table.

The parameter $\rho$ changes the difficulty of the two recovery routes in
different ways.  For PM-MMD, the exact inverse weights grow as $\rho$ decreases:
\[
  \left\|\bigl(\Theta(\rho)^{-1}\bigr)_{i\cdot}\right\|_1
  =
  \frac{4-\rho}{3\rho}.
\]
Thus small $\rho$ amplifies signed-affine estimation error and makes spurious
low-PM-MMD candidates more likely.  For the transition baselines, larger
$\rho$ makes the observed mixture index $Y'$ more informative about the latent
component $Y$, so the classifier posterior cloud reveals the transition simplex
more clearly.

\begin{figure}[H]
\centering
\includegraphics[width=0.82\linewidth]{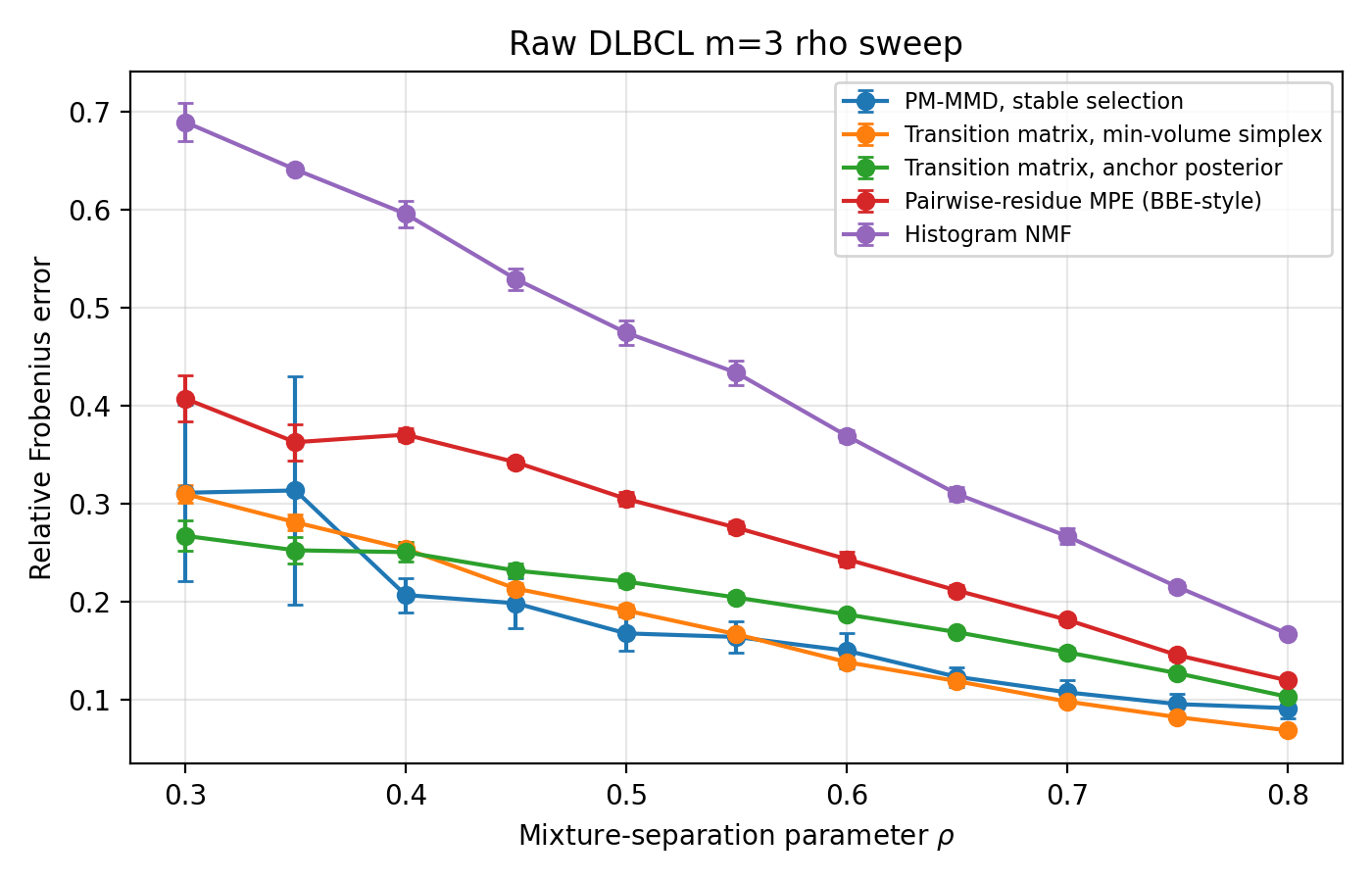}
\caption[Mixture-separation sweep on raw gated-pool DLBCL.]{
Mixture-separation sweep on the raw gated-pool DLBCL experiment.  Points show
mean relative Frobenius error over ten random seeds, and error bars show
standard errors.  PM-MMD is strongest at intermediate mixture separations,
whereas the min-volume transition baseline becomes strongest for larger
$\rho$.  At very small $\rho$, PM-MMD is unstable even with $\bar r=10$.}
\label{fig:raw-dlbcl-rho-sweep}
\end{figure}

\section{Gaussian designs used in the numerical experiments}
\label{app:gmm-designs}

This appendix lists the Gaussian component parameters used in the controlled
experiments in Section~\ref{app:controlled-synthetic-results}.  For a Gaussian component
$p_i=\mathcal N(\mu_i,\Sigma_i)$, marginal independence of two coordinates is
equivalent to zero covariance between them.

\subsection{Symmetric Gaussian design}

The symmetric design uses $d=m=3$, zero means, and unit marginal variances.
The mixing matrix is
\[
\Theta
=
\begin{pmatrix}
0.55 & 0.35 & 0.10\\
0.20 & 0.65 & 0.15\\
0.15 & 0.30 & 0.55
\end{pmatrix}.
\]
The component means are $\mu_1=\mu_2=\mu_3=0$, and for $\rho=0.6$ the covariance
matrices are
\[
\Sigma_1=
\begin{pmatrix}
1 & 0 & \rho\\
0 & 1 & \rho\\
\rho & \rho & 1
\end{pmatrix},\quad
\Sigma_2=
\begin{pmatrix}
1 & \rho & 0\\
\rho & 1 & \rho\\
0 & \rho & 1
\end{pmatrix},\quad
\Sigma_3=
\begin{pmatrix}
1 & \rho & \rho\\
\rho & 1 & 0\\
\rho & 0 & 1
\end{pmatrix}.
\]
Thus $R_{12}=\{1\}$, $R_{13}=\{2\}$, and $R_{23}=\{3\}$.  However, for each
pair, the two non-target components have identical bivariate Gaussian
marginals.  For instance, for the pair $(1,2)$, the $(1,2)$-marginals of
$p_2$ and $p_3$ are identical.  This creates the signed cancellation described
in Section~\ref{app:controlled-synthetic-results}.

\subsection{Asymmetric singleton design}

The singleton design uses $d=m=3$ and the same mixing matrix $\Theta$ as above.
The target sets are
\[
  R_{12}=\{1\},\qquad R_{13}=\{2\},\qquad R_{23}=\{3\}.
\]
The Gaussian means are
\[
\mu_1=(0,0.8,-0.6)^T,\quad
\mu_2=(1.2,-0.7,0.4)^T,\quad
\mu_3=(-0.8,1.4,1.1)^T.
\]
We specify each covariance as
\[
  \Sigma_i=D_i^{1/2} C_i D_i^{1/2},
\]
where $D_i$ is the diagonal matrix of marginal variances and $C_i$ is the
correlation matrix.  The parameters are
\[
D_1=\mathrm{diag}(1,2,0.5),\quad
C_1=
\begin{pmatrix}
1 & 0 & 0.65\\
0 & 1 & 0.55\\
0.65 & 0.55 & 1
\end{pmatrix},
\]
\[
D_2=\mathrm{diag}(2.5,0.6,1.4),\quad
C_2=
\begin{pmatrix}
1 & -0.60 & 0\\
-0.60 & 1 & -0.60\\
0 & -0.60 & 1
\end{pmatrix},
\]
and
\[
D_3=\mathrm{diag}(0.7,1.7,2.8),\quad
C_3=
\begin{pmatrix}
1 & 0.70 & -0.50\\
0.70 & 1 & 0\\
-0.50 & 0 & 1
\end{pmatrix}.
\]
Only the target component has zero covariance on its corresponding pair, while
the non-target bivariate Gaussian marginals are distinct.

\subsection{Asymmetric multi-component design}

The multi-component design uses $d=m=4$ and
\[
\Theta
=
\begin{pmatrix}
0.7 & 0.1 & 0.1 & 0.1\\
0.1 & 0.7 & 0.1 & 0.1\\
0.1 & 0.1 & 0.7 & 0.1\\
0.1 & 0.1 & 0.1 & 0.7
\end{pmatrix}.
\]
The target sets are
\[
  R_{12}=\{1,2\},\qquad
  R_{13}=\{3,4\},\qquad
  R_{23}=\{1,3\},\qquad
  R_{24}=\{2,4\}.
\]
The means are
\[
\begin{aligned}
\mu_1&=(0,0.7,-0.5,1.0)^T,\\
\mu_2&=(1.2,-0.6,0.4,-0.8)^T,\\
\mu_3&=(-0.9,1.2,1.0,0.3)^T,\\
\mu_4&=(0.5,-1.1,1.3,-0.2)^T .
\end{aligned}
\]
Again write $\Sigma_i=D_i^{1/2}C_iD_i^{1/2}$.  The variance matrices are
\[
\begin{aligned}
D_1&=\mathrm{diag}(1,1.4,0.8,1.7),\\
D_2&=\mathrm{diag}(2,0.7,1.3,0.9),\\
D_3&=\mathrm{diag}(0.8,1.9,0.6,1.5),\\
D_4&=\mathrm{diag}(1.6,1.1,1.8,0.7).
\end{aligned}
\]
The correlation matrices are
\[
C_1=
\begin{pmatrix}
1 & 0 & 0.45 & -0.25\\
0 & 1 & 0 & 0.35\\
0.45 & 0 & 1 & 0.20\\
-0.25 & 0.35 & 0.20 & 1
\end{pmatrix},
\]
\[
C_2=
\begin{pmatrix}
1 & 0 & -0.35 & 0.30\\
0 & 1 & 0.50 & 0\\
-0.35 & 0.50 & 1 & -0.25\\
0.30 & 0 & -0.25 & 1
\end{pmatrix},
\]
\[
C_3=
\begin{pmatrix}
1 & -0.40 & 0 & 0.25\\
-0.40 & 1 & 0 & -0.45\\
0 & 0 & 1 & 0.30\\
0.25 & -0.45 & 0.30 & 1
\end{pmatrix},
\]
and
\[
C_4=
\begin{pmatrix}
1 & 0.55 & 0 & -0.30\\
0.55 & 1 & -0.35 & 0\\
0 & -0.35 & 1 & 0.25\\
-0.30 & 0 & 0.25 & 1
\end{pmatrix}.
\]
The zero entries in these correlation matrices induce exactly the target
marginal independences listed above.  The remaining bivariate Gaussian
marginals are chosen to be distinct, which avoids the signed cancellation
present in the symmetric design.

\section{Controlled non-Gaussian design}
\label{app:nongauss-design}

This appendix describes the controlled non-Gaussian design used in
Section~\ref{app:controlled-synthetic-results}.  The design uses $m=d=3$ and the same
mixing matrix as the singleton GMM experiment,
\[
\Theta
=
\begin{pmatrix}
0.55 & 0.35 & 0.10\\
0.20 & 0.65 & 0.15\\
0.15 & 0.30 & 0.55
\end{pmatrix}.
\]
The target marginal-independence pattern is
\[
  R_{12}=\{1\},\qquad
  R_{13}=\{2\},\qquad
  R_{23}=\{3\}.
\]
For each component, the target coordinate pair is generated independently,
and the remaining coordinate is generated as a nonlinear function of the
target pair plus non-Gaussian noise.  All primitive random variables and noise
variables below are mutually independent.  For a variable $A$, write
$\widetilde A$ for its standardized version.  In the implementation this
standardization is performed using the empirical mean and standard deviation
of the generated sample.

For component $p_1$, the independent target pair is $(X_1,X_2)$.  We draw
\[
X_1 \sim
0.58\,N(-1.4,0.35^2)+0.42\,N(1.1,0.25^2),
\]
and
\[
X_2 \sim \Gamma(2,0.55)-0.65,
\]
where the gamma distribution is parameterized by shape and scale.  The third
coordinate is
\[
X_3
=
-0.6
+0.75\widetilde X_1
+0.45\tanh(1.2\widetilde X_2)
+0.35\widetilde X_1\widetilde X_2
+0.2\,T_5,
\]
where $T_5$ is a Student-$t$ random variable with five degrees of freedom.
Thus $X_1$ and $X_2$ are marginally independent under $p_1$.

For component $p_2$, the independent target pair is $(X_1,X_3)$.  We draw
\[
X_1
\sim
\operatorname{Lognormal}(-0.15,0.75^2)
-\exp\{-0.15+0.75^2/2\}
+0.65,
\]
and
\[
X_3
\sim
0.48\,N(-0.8,0.28^2)+0.52\,N(1.55,0.50^2).
\]
The second coordinate is
\[
X_2
=
0.35
-0.65\widetilde X_1
+0.70\sin(0.9\widetilde X_3)
+0.28\widetilde X_1\widetilde X_3
+\varepsilon_2,
\]
where $\varepsilon_2$ has a Laplace distribution with location $0$ and scale
$0.175$.  Thus $X_1$ and $X_3$ are marginally independent under $p_2$.

For component $p_3$, the independent target pair is $(X_2,X_3)$.  We draw
\[
X_2
\sim
0.40\,N(-1.0,0.45^2)+0.60\,N(1.25,0.35^2),
\]
and
\[
X_3 \sim 0.75\,T_4+0.85,
\]
where $T_4$ is a Student-$t$ random variable with four degrees of freedom.
The first coordinate is
\[
X_1
=
-0.25
+0.55\tanh(1.1\widetilde X_2)
-0.70\widetilde X_3
+0.22(\widetilde X_2^2-1)
+\varepsilon_3,
\]
where $\varepsilon_3\sim N(0,0.2^2)$.  Thus $X_2$ and $X_3$ are marginally
independent under $p_3$.

This construction makes the target 
marginal independences exact by
construction, while the non-target coordinate pairs contain nonlinear
dependence.  The marginal distributions are skewed, heavy-tailed, or
multimodal, so the components are not well represented by single Gaussian
clusters.

\section{Semi-synthetic benchmark designs}
\label{app:semisynthetic-benchmark-designs}

This appendix describes the semi-synthetic benchmark construction used in
Section~\ref{app:semisynthetic-results}.  The purpose of these
experiments is to use empirical component marginals from benchmark data while
retaining known component identities, a known mixing matrix, and a controlled
marginal-independence pattern.

The multi-mixture experiments in
Table~\ref{tab:semisynthetic-benchmarks} and the oracle-assisted matrix
estimation comparison in Table~\ref{tab:oracle-binary-mpe} use \(m=d=3\) and
the mixing matrix
\[
\Theta
=
\begin{pmatrix}
0.55 & 0.35 & 0.10\\
0.20 & 0.65 & 0.15\\
0.15 & 0.30 & 0.55
\end{pmatrix}.
\]
The target sets are
\[
  R_{12}=\{1\},\qquad
  R_{13}=\{2\},\qquad
  R_{23}=\{3\}.
\]
For each component $i$, let $(a_i,b_i)$ be its target independent coordinate
pair and let $c_i$ be the remaining coordinate.  Given the class-conditional
empirical data for component $i$, we generate $X_{a_i}$ and $X_{b_i}$ by
independent resampling from their empirical marginal distributions.  We also
draw an independent baseline variable $B_{c_i}$ from the empirical marginal of
coordinate $c_i$.  With empirical means $\mu_j$ and standard deviations $s_j$,
define
\[
  Z_{a_i}=\frac{X_{a_i}-\mu_{a_i}}{s_{a_i}},
  \qquad
  Z_{b_i}=\frac{X_{b_i}-\mu_{b_i}}{s_{b_i}}.
\]
The remaining coordinate is generated by
\[
  X_{c_i}
  =
  0.55 B_{c_i}
  +
  \alpha s_{c_i}\frac{Z_{a_i}+Z_{b_i}}{\sqrt{2}}
  +
  \beta s_{c_i} Z_{a_i}Z_{b_i}
  +
  \varepsilon_i,
\]
where $\varepsilon_i\sim N(0,\sigma^2 s_{c_i}^2)$.  In the reported
benchmarks, we use $\alpha=3.0$, $\beta=0.15$, and $\sigma=0.05$.  This
construction makes the target pair independent by construction, while the two
non-target coordinate pairs carry a controlled dependence signal.

The benchmark-specific choices used in
Table~\ref{tab:semisynthetic-benchmarks} are as follows.  Class labels and
feature indices are the encoded labels and zero-based feature indices used in
the implementation.
\[
\begin{array}{llll}
\text{Digits} & \text{classes } \{0,1,2\} & \text{features } \{59,61,3\} &
\text{none/cubic},\\
\text{Wine} & \text{classes } \{0,1,2\} & \text{features } \{6,12,11\} &
\text{none/cubic},\\
\text{Dry Bean-A} & \text{classes } \{6,4,2\} & \text{features } \{6,11,9\} &
\text{none/cubic},\\
\text{Dry Bean-B} & \text{classes } \{6,5,4\} & \text{features } \{1,6,0\} &
\text{none/cubic}.
\end{array}
\]
For the cubic-warped settings, after generating the semi-synthetic samples we
first apply robust coordinate-wise centering and scaling.  Each standardized
coordinate $z$ is then transformed as
\[
  z \mapsto z+0.3z^3 .
\]
Because the warp is applied coordinate-wise, it preserves marginal
independence of the target coordinate pairs while making the component shapes
less Gaussian. 
The clean-component binary comparison in
Table~\ref{tab:clean-binary-summary} uses the same component generators but
forms binary mixtures
\[
  U_1=p_{\mathrm{clean}},
  \qquad
  U_2=0.55\,p_{\mathrm{clean}}+0.45\,p_{\mathrm{res}}.
\]
We evaluate the three binary pairs
\(p_0\to p_1\), \(p_1\to p_2\), and \(p_2\to p_0\), so that each component is
used once as the residual component.  The residual component is evaluated on
its designated marginal-independence pair. 
We also screened the Shuttle data, but the recovery was
unstable under the present construction, mainly because of strong class
imbalance.  Therefore, Shuttle is not included in the main comparison.

\section{Implementation details}
\label{app:implementation-details}

This appendix gives the implementation details for the numerical experiments
in Section~\ref{sec:numerical-experiments}.  All experiments were run with Python
3.10--3.11 using NumPy, SciPy, pandas, scikit-learn, matplotlib, and openpyxl.
The binary MPE baselines use the external official implementations described
below.  Unless otherwise stated, reported averages use ten independent random
seeds, $0,1,\ldots,9$.  To avoid thread oversubscription in parallel runs, the
shell scripts set 
\begin{align*}
&\texttt{OMP\_NUM\_THREADS},
\texttt{OPENBLAS\_NUM\_THREADS}, 
\texttt{MKL\_NUM\_THREADS},\\
&\texttt{VECLIB\_MAXIMUM\_THREADS}, \text{and}\ 
\texttt{NUMEXPR\_NUM\_THREADS} 
\end{align*}
to one.
The degree of process-level parallelism affects only wall-clock time.

\paragraph{Sample splitting.}
For PM-MMD experiments, samples from each observed mixture are split into an
optimization sample and an independent validation sample.  The reported sample
size $n$ denotes the number of optimization samples per observed mixture;
unless otherwise stated, each observed mixture also contributes $n$ validation
samples.  The optimization split is used to generate candidate affine weights,
and the validation split is used only for representative selection.  Clustering
and factorization references are fitted to the union of the two splits.

\paragraph{Preprocessing and kernels.}
Unless otherwise stated, PM-MMD uses an RBF product kernel on the selected
coordinate pair.  Each coordinate is robustly standardized using the median and
$1.4826$ times the median absolute deviation fitted on the pooled optimization
samples.  For nearly constant coordinates, the scale falls back to the empirical
standard deviation and then to one.  After this scaling, the coordinate-wise RBF
bandwidth is chosen by the median-distance heuristic using 20,000 random pairs
from the pooled optimization sample.  The same scaling and bandwidths are used
on the validation split.

\paragraph{PM-MMD objective and constrained optimization.}
For a selected coordinate pair $(s,t)$ and $L$ observed mixtures, the
optimization variable is $\bm{r}=(r_1,\ldots,r_L)^T\in\mathbb R^L$.  The empirical
squared PM-MMD loss is evaluated from precomputed kernel contractions.  In code
notation it has the quartic form
\[
\widehat J_{s,t}(\bm{r})
=
\sum_{\ell,k}r_\ell r_k A_{\ell k}
-2\sum_{\ell,a,b}r_\ell r_a r_b C_{\ell ab}
+\sum_{a,b,c,d}r_a r_b r_c r_d D_{abcd}.
\]
The constrained optimization problem is
\[
\min_{r\in\mathbb R^L}\ \widehat J_{s,t}(\bm{r})+\lambda\|\bm{r}\|_2^2
\quad\text{subject to}\quad
\mathbf 1^T \bm{r}=1,
\qquad
\|\bm{r}\|_1\le \bar r.
\]
No nonnegativity constraint is imposed on $r$; numerical box bounds
$-\bar r\le r_\ell\le\bar r$ are also imposed.  The reported experiments use
$\lambda=0$, and the validation score uses the same form with selection ridge
parameter $\lambda_{\mathrm{sel}}=0$.

The constrained problem is solved by \texttt{scipy.optimize.minimize} with the
\texttt{SLSQP} method, \texttt{maxiter=1000}, and \texttt{ftol=1e-12}.  No
analytic gradient is supplied.  Since the objective is nonconvex and quartic in
$r$, SLSQP is used only as a local constrained optimizer; the procedure relies
on multi-start optimization.

\paragraph{Candidate generation and representative selection.}
For each coordinate pair, SLSQP is initialized from the standard basis vectors,
the uniform vector, and random feasible vectors satisfying
$\mathbf 1^T \bm{r}=1$ and $\|r\|_1\le\bar r$.  The number of starts is controlled
by \texttt{num-paths}.  Debug-only oracle starts are available in the code but
are not used in the reported runs.

After optimization, we retain candidates whose optimization loss is below the
training threshold and also retain the \texttt{keep-top-train} candidates with
the smallest training losses.  If no candidate passes the threshold, the
lowest-training-loss candidates are used as a fallback.  These candidates are
deduplicated in $r$-space by the specified Euclidean radius and then rescored on
the independent validation split by
\[
\widehat J^{\mathrm{val}}_{s,t}(\bm{r})+\lambda_{\mathrm{sel}}\|\bm{r}\|_2^2 .
\]

Most reported PM-MMD experiments use a greedy validation-diversity rule.  For
each coordinate pair, candidates are sorted by validation score and scanned in
increasing order.  A candidate is retained only if its Euclidean distance in
$r$-space from all previously retained pairwise representatives is at least the
pairwise separation threshold.  This step retains at most $q_{\max}$
representatives for each coordinate pair.  The pairwise representatives are then
pooled across coordinate pairs, and the same greedy rule is applied with the
global separation threshold to retain at most $K_{\max}=m$ global
representatives.  If fewer than $m$ global representatives are obtained, the
global separation threshold is relaxed deterministically, first to half of its
value and then to zero.

For the raw gated-pool three-component DLBCL experiment in
Section~\ref{subsec:raw-dlbcl-three-component}, we use a condition-aware
global selector after the pairwise greedy step.  The selector enumerates all
$m$-subsets of the retained pairwise representatives and minimizes the score in
\eqref{eq:stable-global-selection-score}.  The weights are
\[
\lambda_{\mathrm{cond}}=0.05,
\qquad
\lambda_{\mathrm{neg}}=10.0,
\qquad
\lambda_{\Delta}=0.
\]
Subsets with singular $R_{\mathcal C}$ are discarded.  This rule penalizes
ill-conditioned inverses and, more importantly, candidate sets whose inverse
has substantial negative mass.  It uses only the candidate affine weights and
held-out PM-MMD scores.  The known mixing matrix $\Theta$ is not used in
candidate generation or representative selection; it is used only for
evaluation.

\paragraph{Deterministic stability of condition-aware selection.}
\label{app:stable-selection-guarantee}
The following elementary statement explains what the condition-aware global
selector guarantees once the candidate set contains a separated correct subset.
Let $\mathfrak C_m$ denote the finite collection of all $m$-subsets of the
retained pairwise representatives.  For $\mathcal C\in\mathfrak C_m$, let
$S(\mathcal C)$ be the population analogue of
\eqref{eq:stable-global-selection-score}, obtained by replacing the held-out
validation losses by their population PM-MMD losses, and let
$\widehat S(\mathcal C)$ be the empirical score used by the algorithm.

\begin{proposition}[Selection stability and inverse perturbation]
Suppose that there exists a subset $\mathcal C_\star\in\mathfrak C_m$ such that
\[
  S(\mathcal C)\ge S(\mathcal C_\star)+\gamma
  \qquad
  \text{for all }\mathcal C\in\mathfrak C_m,\ \mathcal C\ne\mathcal C_\star,
\]
for some margin $\gamma>0$.  If
\[
  \sup_{\mathcal C\in\mathfrak C_m}
  |\widehat S(\mathcal C)-S(\mathcal C)| < \gamma/2,
\]
then the empirical condition-aware selector returns $\mathcal C_\star$.
Moreover, write $R_\star=\Theta^{-1}$ and let $\widehat R$ be the matrix formed
from the selected candidate rows.  If
\[
  \|\widehat R-R_\star\|_2 < 1/\|R_\star^{-1}\|_2,
\]
then
\[
  \|\widehat R^{-1}-R_\star^{-1}\|_F
  \le
  \frac{\|R_\star^{-1}\|_2^2\,\|\widehat R-R_\star\|_F}
       {1-\|R_\star^{-1}\|_2\,\|\widehat R-R_\star\|_2}.
\]
\end{proposition}

\begin{proof}
For any $\mathcal C\ne\mathcal C_\star$, the assumed uniform score accuracy gives
\[
  \widehat S(\mathcal C)
  > S(\mathcal C)-\gamma/2
  \ge S(\mathcal C_\star)+\gamma/2
  > \widehat S(\mathcal C_\star).
\]
Thus $\mathcal C_\star$ is the unique empirical minimizer.  The inverse bound is
obtained from the identity
\[
  \widehat R^{-1}-R_\star^{-1}
  = R_\star^{-1}(R_\star-\widehat R)\widehat R^{-1}
\]
and the Neumann-series bound
\[
  \|\widehat R^{-1}\|_2
  \le
  \frac{\|R_\star^{-1}\|_2}
       {1-\|R_\star^{-1}\|_2\,\|\widehat R-R_\star\|_2}.
\]
\end{proof}

This proposition is intentionally deterministic.  It separates two issues:
whether candidate generation places representatives near the true rows of
$\Theta^{-1}$, and whether the global selector chooses a well-conditioned
subset from those representatives.  The PM-MMD concentration bounds in
Theorem~\ref{thm:pmmd-uniform-bound} control the validation-loss component of
$\widehat S$, while the condition-number and inverse-nonnegativity terms are
computed directly from the candidate affine weights.

\paragraph{Mixing-matrix estimation and evaluation.}
When $m$ global representatives have been selected, they are stacked as rows of
$\widehat R$, and the raw mixing-matrix estimate is
$\widehat\Theta_{\mathrm{raw}}=\widehat R^{-1}$.  The rows of
$\widehat\Theta_{\mathrm{raw}}$ are projected onto the probability simplex for
reporting.  Column permutation ambiguity is removed by solving the assignment
problem that minimizes the Frobenius distance between the estimated and true
mixing matrices.  The reported relative error is
\[
\frac{\|\widehat\Theta_{\mathrm{aligned}}-\Theta\|_F}{\|\Theta\|_F}.
\]
For diagnostic recovery, each selected vector is mapped to its nearest simplex
vertex,
\[
\operatorname{nearest}(\bm{r})=\arg\min_{j\in[m]}\|\bm{r}^T \Theta-\bm{e}_j\|_2.
\]
The maximum vertex distance is the maximum of these nearest-vertex distances.
Vertex coverage is the number of distinct nearest vertices among the relevant diagnostic vectors.  We use ``Rec.'' or ``PM rec.'' in tables for vertex recovery. 
In the singleton experiments, vertex recovery means that each selected diagnostic vector is nearest to its intended target vertex. 
In the multi-component experiment, pairwise vertex overage means that each coordinate pair's selected representatives cover all vertices in the corresponding target set $R_{st}$, 
and global vertex recovery means that the final global representatives recover all $m$ vertices. 

\paragraph{Controlled Gaussian experiments.}
For the singleton GMM experiment, the sample sizes are
$n\in\{2000,5000,10000,20000\}$ per split and per observed mixture.  PM-MMD uses
\begin{align*}
&\bar r=4.0,
\quad
\texttt{num-paths}=300,
\quad
\texttt{keep-top-train}=20,\\
&
\texttt{train-threshold}=10^{-3},
\quad
\texttt{dedup-radius}=0.15.    
\end{align*}
For the multi-component GMM experiment, the same sample sizes are used and
PM-MMD uses
\begin{align*}
&\bar r=2.5,
\quad
\texttt{num-paths}=400,
\quad
\texttt{keep-top-train}=80,\\
&\texttt{train-threshold}=10^{-3},
\quad
\texttt{dedup-radius}=0.20.
\end{align*}
In the multi-component case, four representatives are retained from each
coordinate pair using pairwise separation $0.30$ in $r$-space, and four global
representatives are retained using global separation $0.75$.  Both Gaussian
experiments use RBF kernels, MAD scaling, automatic bandwidths, and zero ridge
parameters.

\paragraph{Controlled non-Gaussian experiment.}
The component-wise data-generating mechanism is given in
Appendix~\ref{app:nongauss-design}.  PM-MMD uses
\begin{align*}
&\bar r=4.0,
\quad
\texttt{num-paths}=400,
\quad
\texttt{keep-top-train}=40,\\
&\texttt{train-threshold}=10^{-3},
\quad
\texttt{dedup-radius}=0.15,
\end{align*}
with RBF kernels, MAD scaling, automatic bandwidths, zero ridge parameters, and
the same training/validation split convention.

\paragraph{Semi-synthetic benchmark experiments.}
Digits and Wine are loaded from scikit-learn, and Dry Bean is downloaded from
the UCI repository on first use and cached locally.  The class labels, feature
choices, resampling scheme, dependence construction, and coordinate-wise cubic
warp are specified in Appendix~\ref{app:semisynthetic-benchmark-designs}.
PM-MMD uses
\begin{align*}
&\bar r=4.0,
\quad
\texttt{num-paths}=300,
\quad
\texttt{keep-top-train}=30,\\
&\texttt{train-threshold}=10^{-3},
\quad
\texttt{dedup-radius}=0.20,
\end{align*}
with RBF kernels, MAD scaling, automatic bandwidths, and zero ridge parameters.

\paragraph{Clustering and factorization references.}
The Gaussian mixture reference is fitted to the pooled optimization and
validation samples from all observed mixtures with $m$ components, full
covariance matrices, \texttt{reg\_covar}$=10^{-5}$, five random
initializations, and maximum iteration number 500.  Mixing weights are
estimated by averaging posterior responsibilities over the samples from each
observed mixture.  The $k$-means reference is fitted to the pooled samples with
$m$ clusters and 20 random initializations; mixing weights are empirical
cluster proportions for each observed mixture.  The NMF reference converts each
observed mixture into a normalized joint histogram on a common quantile grid
with four bins per coordinate.  The resulting nonnegative matrix is factorized
with rank $m$ using NMF with \texttt{init="nndsvda"}, multiplicative-update
solver, Kullback--Leibler loss, maximum iteration number 1000, and the seed of
the corresponding run.

The histogram archetypal reference uses the same normalized joint histograms
and estimates $m$ archetypes in the histogram simplex.  Mixing weights are then
obtained by nonnegative least-squares/simplex projection for each observed
mixture histogram.  This baseline is included as a geometric
simplex-factorization reference that, unlike PM-MMD, does not use marginal
independence.

The pairwise-residue MPE reference is a same-information decontamination
heuristic.  For each ordered pair of observed mixture histograms
$(\widehat U_a,\widehat U_b)$, a binary MPE subroutine estimates
$\widehat\kappa_{ab}\approx\kappa(U_a\mid U_b)$ and forms the pseudo-residue
\[
\widehat R_{a\setminus b}
=
\frac{\widehat U_a-\widehat\kappa_{ab}\widehat U_b}{1-\widehat\kappa_{ab}}.
\]
Negative histogram bins are clipped and the residue is renormalized.  A diverse
set of $m$ pseudo-residues is selected, and the observed mixture histograms are
then fitted by simplex least squares against these residue candidates.  In the
DLBCL three-component experiment we report a BBE-style implementation of this
binary MPE subroutine.  This baseline uses only the unlabeled observed mixtures;
it does not receive clean component samples.

\paragraph{Classifier-posterior transition baselines.}
For the raw gated-pool three-component DLBCL experiment, we also evaluate
classifier-posterior transition-matrix baselines~\cite{liu2016importance,patrini2017making,li2021volminnet}.  We pool the observed mixture
samples, denote the mixture index by $Y'\in\{1,2,3\}$, and treat it as an
observed noisy label.  A cross-fitted multinomial logistic regression is trained
to estimate posterior vectors
\[
  \widehat p(x)=\widehat P(Y'=\cdot\mid X=x)\in\Delta^{2}.
\]
For each seed, the transition baselines use 1000 samples per observed mixture,
matching the total sample size used by PM-MMD in the same experiment.

The forward transition matrix is
\[
  T_{ji}=P(Y'=j\mid Y=i),
\]
so its $i$th column is the distribution of the observed mixture index for a
point from latent component $i$.  We estimate these columns from the posterior
vectors $\widehat p(x)$ by four posterior-simplex constructions.  The
anchor-posterior variant averages posterior vectors with large
$\widehat P(Y'=i\mid X)$; the archetype variant selects three extreme posterior
vectors by a farthest-point heuristic~\cite{gonzalez1985clustering}; and the
$k$-means variant clusters the posterior vectors into three groups and uses the
cluster centers after normalization~\cite{macqueen1967some}.  The min-volume
variant fits a triangle to the posterior cloud by minimizing an outside-simplex
penalty for barycentric coordinates plus a small volume penalty and a condition
number penalty.  This is a lightweight post-hoc analogue of the
minimum-volume transition-matrix idea, not an official end-to-end VolMinNet
implementation.

Given the empirical mixture-index prior $\widehat q_j=\widehat P(Y'=j)$, we
compute the latent prior $\widehat\pi$ from
\[
  \widehat q=\widehat T\widehat\pi
\]
using a linear solve followed by simplex projection when needed.  We then apply
Bayes' rule to obtain the estimated mixing matrix
\[
  \widehat\Theta_{ji}
  =
  \widehat P(Y=i\mid Y'=j)
  =
  \frac{\widehat T_{ji}\widehat\pi_i}{\widehat q_j}.
\]
Thus, these baselines use the same observed-data information and sample budget
as PM-MMD, and estimate $\Theta$ through the geometry of the classifier
posterior vectors $\widehat P(Y'\mid X)$.

\paragraph{Binary MPE baselines.}
We use the official implementations of KM2, DEDPUL, and BBE.  KM2 is run
through the KMPE wrapper with \texttt{KM\_2=True}, \texttt{epsilon=0.04}, lower
bound 1.0, and upper bound 8.0.  DEDPUL uses cross-fitted random-forest
predictions with 200 trees, maximum depth 8, minimum leaf size 5, and three
folds.  BBE uses three-fold cross-fitted clean-probability scores from a
histogram gradient boosting classifier with 180 iterations, learning rate
0.06, maximum 15 leaf nodes, and $\ell_2$ regularization 0.01.

For the oracle-assisted semi-synthetic comparison, clean samples from each
true component are supplied to the binary MPE method.  These baselines
therefore use stronger information than PM-MMD.  The estimated binary
proportions are assembled into a mixing matrix, rows are normalized, and the
permutation-aligned relative Frobenius error is reported.

For the binary DLBCL prior-shift calibration in
Table~\ref{tab:dlbcl-binary-prior-shift}, no clean component samples are
provided.  Each binary MPE method is applied in both directions to the two
unlabeled mixtures, using 2000 samples from each mixture, to estimate
\[
\alpha=\kappa(U_1\mid U_2),
\qquad
\beta=\kappa(U_2\mid U_1).
\]
The mixing matrix is recovered from
\[
a=\frac{1-\alpha}{1-\alpha\beta},
\qquad
b=\beta a,
\]
corresponding to $U_1=a p_1+(1-a)p_2$ and
$U_2=b p_1+(1-b)p_2$ after choosing the orientation $a>b$.  This mutual-MPE
comparison uses the same unlabeled observed mixtures as PM-MMD in the DLBCL
experiment, but it exploits identities specific to the two-component case.

\paragraph{Clean-component binary PM-MMD variant.}
For the clean-component binary setting described in
Section~\ref{app:semisynthetic-results} and
Appendix~\ref{app:semisynthetic-benchmark-designs}, the PM-MMD variant
restricts the affine search to the one-dimensional residual path
\[
r_\alpha=
\left(-\frac{\alpha}{1-\alpha},\frac{1}{1-\alpha}\right).
\]
Here $\alpha\in[0,\alpha_{\max}]$ and
\[
\alpha_{\max}\le \min\left\{0.99,\frac{\bar r-1}{\bar r+1}\right\}.
\]
The reported runs use $\bar r=4.0$, so the radius constraint gives
$\alpha_{\max}\le0.6$.  The training loss is first evaluated on a grid of 301
points; the ten best grid points are refined by bounded scalar minimization
with tolerance $10^{-5}$.  Refined candidates are deduplicated at threshold
$10^{-5}$, and the final value is selected by validation loss.  The table
averages over the binary pairs $0\to1$, $1\to2$, and $2\to0$ on the
semi-synthetic benchmark settings.

\paragraph{Raw gated-pool three-component DLBCL flow-cytometry experiment.}
\label{app:raw-dlbcl-details}
For the raw gated-pool three-component DLBCL experiment in
Section~\ref{subsec:raw-dlbcl-three-component}, we use the manual-gated
classes $0,1,2$ with counts $47$, $604$, and $4873$, and the marker coordinates
FL1, FL2, and FL4.  The labels are used only to construct the empirical
component pools and to evaluate recovery.  The observed mixtures are formed with
\[
\Theta=0.70 I_3+\frac{0.30}{3}\mathbf 1\mathbf 1^T,
\]
so the diagonal entries are $0.8$ and the off-diagonal entries are $0.1$.
No coordinate-wise resampling is used to enforce independence.

PM-MMD uses the two marker pairs FL1--FL4 and FL2--FL4 (zero-indexed pairs
$(0,2)$ and $(1,2)$), $\bar r=4.0$, 500 optimization samples and 500
independent validation samples per observed
mixture, 300 multistart paths, 
\begin{align*}
\texttt{keep-top-train}=100,\quad
\texttt{train-threshold}=0.01, \quad
\texttt{dedup-radius}=0.05, 
\end{align*}
pairwise separation $0.15$, and $q_{\max}=3$.  The global representatives are chosen by the
condition-aware score in \eqref{eq:stable-global-selection-score} with the
weights stated above.  The same-information baselines are fitted to the same
unlabeled mixtures.  The oracle BBE-style binary MPE baseline is given clean
component samples and is used only as an oracle calibration.

The selected marker pairs are treated as prescribed choices.  Any
component-level product-marginal checks for these pairs are post-hoc diagnostics
of whether the raw component pools contain a marginal-independence signal; they
are not supplied to PM-MMD candidate generation or representative selection.
The very large vertex-distance value for spectral clustering in
Table~\ref{tab:raw-dlbcl-three-component} reflects the ill-conditioning of the
induced vertex inversion in that fit and is not used as a PM-MMD stability
criterion.

\paragraph{PM-MMD settings for the binary DLBCL calibration.}
For Table~\ref{tab:dlbcl-binary-prior-shift}, PM-MMD uses
$\rho\in\{0.40,0.50,0.60,0.70\}$, $\bar r=10.0$, 300 multistart paths,
\texttt{keep-top-train}$=30$, \texttt{train-threshold}$=10^{-3}$, and
\texttt{dedup-radius}$=0.20$.  The pairwise separation, global separation, and
vertex threshold are $0.30$, $0.60$, and $0.35$, respectively.  We use 1000
optimization samples and 1000 independent validation samples from each observed
mixture.  The pairwise and global representative limits are $q_{\max}=m$ and
$K_{\max}=m$.  The RBF kernel, MAD scaling, automatic bandwidth selection, and
zero ridge parameters are used as in the semi-synthetic experiments.